%% file: paper_for_ronny.tex
\documentclass[11pt]{article}

\usepackage{fullpage}
\usepackage{authblk}

\usepackage{times}
\usepackage{graphicx} 
\usepackage{caption}
\usepackage{subcaption} 

\usepackage{algorithm}
\usepackage{algorithmic}

\usepackage{hyperref}
\hypersetup{
	colorlinks   = true, 
	urlcolor     = blue, 
	linkcolor    = blue, 
	citecolor   = red 
}

\usepackage{amsfonts}
\usepackage{multirow}

\usepackage{amsmath}

\usepackage{amssymb,url,color}

\usepackage{etaremune}

\usepackage{tikz}

\usepackage[utf8]{inputenc}
\usepackage{pgfplots}

\newcommand{\N}{\mathbb{N}}

\newcommand{\norm}[1]{\|#1\|}
\renewcommand{\H}{\mathcal{H}}

\newcommand{\reals}{\mathbb{R}}

\newcommand{\inner}[1]{\langle #1 \rangle}

\newcommand{\numsteps}{55}
\newcommand{\Var}{\text{Var}}

\newtheorem{theorem}{Theorem}

\newtheorem{lemma}{Lemma}

\newcommand{\BlackBox}{\rule{1.5ex}{1.5ex}}  
\newenvironment{proof}{\par\noindent{\bf Proof\ }}{\hfill\BlackBox\\[2mm]}

\newcommand{\lemref}[1]{Lemma~\ref{#1}}
\newcommand{\thmref}[1]{Theorem~\ref{#1}}
\newcommand{\figref}[1]{Figure~\ref{#1}}
\newcommand{\secref}[1]{Section~\ref{#1}}
\newcommand{\appref}[1]{Appendix~\ref{#1}}

\newcommand{\bx}{\mathbf{x}}
\newcommand{\bv}{\mathbf{v}}
\newcommand{\bu}{\mathbf{u}}
\newcommand{\bw}{\mathbf{w}}
\newcommand{\bz}{\mathbf{z}}
\newcommand{\bg}{\mathbf{g}}
\newcommand{\ba}{\mathbf{a}}

\newcommand{\f}{\mathbf{f}}
\newcommand{\set}[1]{\left\{#1\right\}}

\DeclareMathOperator*{\E}{\mathbb{E}}

\newcommand{\half}{\frac{1}{2}}

\newcommand{\figsdir}{.}  

\newlength\figureheight
\newlength\figurewidth
\title{Failures of Gradient-Based Deep Learning}
\date{}

\author[1]{Shai Shalev-Shwartz}
\author[2]{Ohad Shamir}
\author[1]{Shaked Shammah}
\affil[1]{School of Computer Science and Engineering, The Hebrew University}
\affil[2]{Weizmann Institute of Science}

\begin{document}

\maketitle

\begin{abstract}
  In recent years, Deep Learning has become the go-to solution for a
  broad range of applications, often outperforming
  state-of-the-art. However, it is important, for both theoreticians
  and practitioners, to gain a deeper understanding of the
  difficulties and limitations associated with common approaches and
  algorithms. We describe four types of simple problems, for which the
  gradient-based algorithms commonly used in deep learning either fail
  or suffer from significant difficulties. We illustrate the failures
  through practical experiments, and provide theoretical insights
  explaining their source, and how they might be
  remedied\footnote{This paper was done with the support of the Intel
    Collaborative Research institute for Computational Intelligence
    (ICRI-CI) and is part of the “Why $\textrm{\&}$ When Deep Learning works –
    looking inside Deep Learning” ICRI-CI paper bundle}.
\end{abstract}

\section{Introduction}
The success stories of deep learning form an ever lengthening list
of practical breakthroughs and state-of-the-art performances, ranging
the fields of computer vision
\cite{krizhevsky2012imagenet, He_2016_CVPR, Schroff_2015_CVPR, taigman2014deepface}, audio and natural language processing and
generation \cite{collobert2008unified, hinton2012deep,
  graves2013speech, van2016wavenet}, as well as robotics
\cite{mnih2015human, schulman2015trust}, to name just a
few. The list of success stories can be matched and surpassed by a
list of practical ``tips and tricks'', from different optimization
algorithms, parameter tuning methods
\cite{sutskever2013importance, kingma2014adam},
initialization schemes \cite{glorot2010understanding}, architecture
designs \cite{szegedy2016inception}, loss functions, data
augmentation \cite{krizhevsky2012imagenet} and so on.

The current theoretical understanding of deep learning is far
from being sufficient for a rigorous analysis of the difficulties
faced by practitioners.  Progress must be made from both parties: from
a practitioner's perspective, emphasizing the difficulties provides
practical insights to the theoretician, which in turn, supplies
theoretical insights and guarantees, further strengthening and
sharpening practical intuitions and wisdom. In particular,
understanding \emph{failures} of existing algorithms is as important as
understanding where they succeed.

Our goal in this paper is to present and discuss families of simple
problems for which commonly used methods do not show as exceptional a
performance as one might expect. We use empirical results and insights
as a ground on which to build a theoretical analysis, characterising
the sources of failure. Those understandings are aligned, and
sometimes lead to, different approaches, either for an architecture,
loss function, or an optimization scheme, and explain their
superiority when applied to members of those families. Interestingly, the sources
for failure in our experiment do not seem to relate to stationary point issues such as spurious local minima
or a plethora of saddle points, a topic of much recent interest (e.g. \cite{dauphin2014identifying,choromanska2015loss}), but rather to more subtle issues, having to do with informativeness of the gradients, signal-to-noise ratios, conditioning etc. The code for
running all our experiments is available online\footnote{
	\url{https://github.com/shakedshammah/failures_of_DL}. See command
	lines in \appref{cmdlns}.}.

We start off in \secref{sec.Parities and Linear-Periodic
  Functions} by discussing a class of simple learning problems for
which the gradient information, central to deep learning
algorithms, provably carries negligible information on the target function
which we attempt to learn. This result is a property of the learning
problems themselves, and holds for any specific network architecture
one may choose for tackling the learning problem, implying that no
gradient-based method is likely to succeed. Our analysis relies on tools and 
insights from the Statistical Queries literature, and underscores one of the
main deficiencies of Deep Learning: its reliance on local properties
of the loss function, with the objective being of a global nature.

Next, in \secref{sec.Decomposition vs. End-to-end}, we tackle the
ongoing dispute between two common approaches to learning. Most, if
not all, learning and optimization problems can be viewed as some
structured set of sub-problems. The first approach, which we refer to
as the ``end-to-end'' approach, will tend to solve all of the
sub-problems together in one shot, by optimizing a single primary
objective. The second approach, which we refer to as the
``decomposition'' one, will tend to handle these sub-problems
separately, solving each one by defining and optimizing additional
objectives, and not rely solely on the primary objective.  The
benefits of the end-to-end approach, both in terms of requiring a
smaller amount of labeling and prior knowledge, and perhaps enabling
more expressive architectures, cannot be ignored. On the other hand,
intuitively and empirically, the extra supervision injected through
decomposition is helpful in the optimization process. We experiment
with a simple problem in which application of the two approaches
is possible, and the distinction between them is clear and
intuitive. We observe that an end-to-end approach can be much slower
than a decomposition method, to the extent that, as the scale of the
problem grows, no progress is observed. We analyze this gap 
by showing, theoretically and empirically, that the gradients are much more
noisy and less informative with the end-to-end approach, as opposed to the decomposition approach, 
explaining the disparity in practical performance. 

In \secref{sec.Piece-wise Linear AutoEncoders}, we demonstrate the
importance of both the network's architecture and the optimization
algorithm on the training time. While the choice of architecture is
usually studied in the context of its expressive power, we show that
even when two architectures have the same expressive power for a given
task, there may be a tremendous difference in the ability to optimize
them. We analyze the required runtime of gradient descent for the two
architectures through the lens of the condition number of the problem.
We further show that conditioning techniques can yield additional
orders of magnitude speedups. The experimental setup in this section
is around a seemingly simple problem --- encoding a piece-wise linear
one-dimensional curve. Despite the simplicity of this problem, we show
that following the common rule of ``perhaps I should use a
deeper/wider network''\footnote{See
  \url{http://joelgrus.com/2016/05/23/fizz-buzz-in-tensorflow/} for
  the inspiration behind this quote.} does not
significantly help here.

Finally, in \secref{sec.Non Continuous non-linearities}, we
consider deep learning's reliance on ``vanilla'' gradient information
for the optimization process. We previously discussed the deficiency
of using a local property of the objective in directing global
optimization. Here, we focus on a simple case in which it is possible
to solve the optimization problem based on local information, but not
in the form of a gradient. We experiment with architectures that
contain activation functions with flat regions, which leads to the
well known vanishing gradient problem. Practitioners take great care
when working with such activation functions, and many heuristic tricks
are applied in order to initialize the network's weights in non-flat
areas of its activations.  Here, we show that by using a different
update rule, we manage to solve the learning problem
efficiently. Moreover, one can show convergence guarantees for a family of such
functions. This provides a clean example where non-gradient-based optimization
schemes can overcome the limitations of gradient-based learning.

\section{Parities and Linear-Periodic Functions}\label{sec.Parities
  and Linear-Periodic Functions}

Most existing deep learning algorithms are gradient-based
methods; namely, algorithms which
optimize an objective through access to its gradient w.r.t. some
weight vector $\bw$, or estimates of the gradient. We consider a
setting where the goal of this optimization process is to learn some
underlying hypothesis class $\H$, of which one member, $h\in\H$, is
responsible for labelling the data. This yields an optimization problem
of the form 
\[\min_{\bw} F_h(\bw).\]  
The underlying assumption is that
the gradient of the objective w.r.t. $\bw$, $\nabla F_h(\bw)$,
contains useful information regarding the target function $h$, and
will help us make progress.

Below, we discuss a family of problems for which with high
probability, at any fixed point, the gradient, $\nabla F_{h}(\bw)$,
will be essentially the same regardless of the underlying target
function $h$.  Furthermore, we prove that this holds independently of
the choice of architecture or parametrization, and using a deeper/wider
network will not help. The family we study is that of compositions of
linear and periodic functions, and we experiment with the classical
problem of learning parities. Our empirical and theoretical study
shows that indeed, if there's little information in the gradient,
using it for learning cannot succeed.

\subsection{Experiment}\label{experiment_parity}
We begin with the simple problem of learning random
parities: After choosing some $\bv^*\in \{0,1\}^d$ uniformly at random, our 
goal is to train a predictor mapping
$\bx\in \{0,1\}^d$ to $y=(-1)^{\inner{\bx,\bv^*}}$, where $\bx$ is uniformly 
distributed. In words, $y$ indicates whether
the number of $1$'s in a certain subset of coordinates of $\bx$
(indicated by $\bv^*$) is odd or even. 

For our experiments, we use the hinge loss, and a simple network architecture 
of one fully
connected layer of width $10d>\frac{3d}{2}$ with ReLU
activations, and a fully connected output layer with linear
activation and a single unit. Note that this class realizes the parity
function corresponding to any $\bv^*$ (see \lemref{lem:boringparity} in the 
appendix). 

Empirically, as the dimension $d$ increases, so does the difficulty of
learning, which can be measured in the accuracy we arrive at after a
fixed number of training iterations, to the point where around $d=30$,
no advance beyond random performance is observed after reasonable
time. \figref{fig:parity} illustrates the results.

\setlength\figureheight{4.5cm}
\setlength\figurewidth{8cm}
\begin{figure}
\begin{center}
\input{\figsdir/Parity_mytikz.tex}
\end{center}
\caption{Parity Experiment: Accuracy as a function of the number of
  training iterations, for various input dimensions.} \label{fig:parity}
\end{figure}
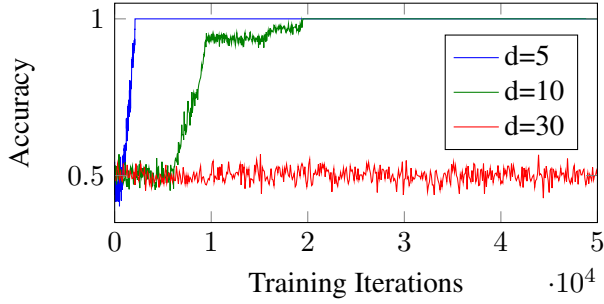

\subsection{Analysis}
To formally explain the failure from a geometric perspective, consider
the stochastic optimization problem associated with learning a target
function $h$,
\begin{equation}\label{eq:obj}
\min_{\bw} F_{h}(\bw) := \E_{\bx}\left[\ell(p_{\bw}(\bx),h(\bx))\right],
\end{equation}
where $\ell$ is a loss function, $\bx$ are the stochastic inputs (assumed to be 
vectors in Euclidean space), and $p_{\bw}$ 
is some predictor parametrized by a parameter vector $\bw$ (e.g. a neural 
network of a certain architecture). We will assume that $F$ is differentiable. 
A key quantity we will be interested in studying is the 
\emph{variance} of the gradient of $F$ with respect to $h$, 
when $h$ is drawn uniformly at random from a collection of candidate target functions 
$\H$:
\begin{equation}\label{vardef}
\Var(\H,F,\bw)=\E_h
\left\|\nabla 
F_{h}(\bw)- \E_{h'}\nabla F_{h'}(\bw)\right\|^2
\end{equation}
Intuitively, this measures the expected amount of ``signal'' about the
underlying target function contained in the gradient. As we will see
later, this variance correlates with the difficulty of solving
\eqref{eq:obj} using gradient-based methods\footnote{This should not
  be confused with the variance of gradient estimates used by SGD,
  which we discuss in \secref{sec.Decomposition
    vs. End-to-end}.}.

The following theorem bounds this variance term.
\begin{theorem}\label{thm:sqdim}
Suppose that 
\begin{itemize}
\item $\H$ consists of real-valued functions $h$ satisfying 
$\E_{\bx}[h^2(\bx)]\leq 1$, such     
that for any two distinct $h,h'\in \H$, $\E_{\bx}[h(\bx)h'(\bx)]=0$.
\item $p_{\bw}(\bx)$ is differentiable w.r.t. $\bw$, and satisfies
$\E_{\bx}\left[\norm{\frac{\partial}{\partial \bw}p_\bw(\bx)}^2\right]\leq G(\bw)^2$ 
for some scalar $G(\bw)$. 
\item The loss function $\ell$ in \eqref{eq:obj} is either the square loss 
$\ell(\hat{y},y) = \frac{1}{2}(\hat{y}-y)^2$ or a classification loss of 
the 
form $\ell(\hat{y},y)=r(\hat{y}\cdot y)$ for some $1$-Lipschitz function $r$, 
and the 
target function $h$ takes values in $\set{\pm1}$.
\end{itemize}
Then 
\[
\Var(\H,F,\bw) ~\leq~ \frac{G(\bw)^2}{|\H|}.
\]
\end{theorem}
The proof is given in \appref{proof:thm:sqdim}. The theorem implies that if we try to learn an unknown target function, 
possibly coming from a large collection of uncorrelated functions, then the 
sensitivity of the gradient to the target function at any point
decreases linearly with $|\H|$.

Before we make a more general statement, let us return to the case of
parities, and study it through the lens of this
framework. Suppose that our target function is some parity function
chosen uniformly at random, i.e. a random element from the set of
$2^d$ functions
$\H = \{\bx\mapsto (-1)^{\inner{\bx,\bv^*}}:\bv^*\in
\{0,1\}^d\}$. These are binary functions, which are easily seen to be
mutually orthogonal: Indeed, for any $\bv,\bv'$,
\begin{align*}
&\E_{\bx}\left[(-1)^{\inner{\bx,\bv}}(-1)^{\inner{\bx,\bv'}}\right] 
= \E_{\bx}\left[(-1)^{\inner{\bx,\bv+\bv'}}\right] \\
=& \prod_{i=1}^{d}\E\left[(-1)^{x_i(v_i+v'_i)}\right] 
= \prod_{i=1}^{d}\frac{(-1)^{v_i+v'_i}+(-1)^{-(v_i+v'_i)}}{2}
\end{align*}
which is non-zero if and only if $\bv=\bv'$. Therefore, by
\thmref{thm:sqdim}, we get that $ \Var(\H,F,\bw) \leq G(\bw)^2/2^d$ -- that is, 
exponentially small in the dimension $d$. By Chebyshev's inequality, this 
implies that the gradient at any point $\bw$ will be extremely concentrated 
around a fixed point independent of $h$.

This phenomenon of exponentially-small variance can also be observed 
for other distributions, and learning problems other than parities.  Indeed, in
\cite{shamir2016distribution}, it was shown that this also holds in a
more general setup, when the output $y$ corresponds to a linear
function composed with a periodic one, and the input $\bx$ is sampled
from a smooth distribution:
\begin{theorem}[Shamir 2016]
Let $\psi$ be a fixed periodic function, and let 
$\H=\{\bx\mapsto\psi(\bv^{*\top}\bx):\norm{\bv^*}=r\}$ for some $r>0$. 
Suppose $\bx\in\reals^d$ is sampled from an arbitrary mixture of 
distributions with the following property: The square root of the density 
function $\varphi$ has a Fourier transform $\hat{\varphi}$ satisfying 
$\frac{\int_{\bx:\|\bx\|>r}\hat{\varphi}^2(\bx)d\bx}{\int_{\bx}\hat{\varphi}^2(\bx)d\bx}
 \leq 
\exp(-\Omega(r))$. Then if $F$ 
denotes the objective function with respect to the squared loss, 
\[
\Var(\H,F,\bw) \le
  O\left(\exp(-\Omega(d))+\exp(-\Omega(r))\right).\] 
\end{theorem}
The condition on the Fourier transform of the density is generally satisfied 
for smooth distributions (e.g. arbitrary Gaussians whose covariance matrices 
are positive definite, with all eigenvalues at least $\Omega(1/r)$). Thus, the 
bound is extremely small as long as the 
norm $r$ and the dimension $d$ are moderately large, and indicates that the 
gradients contains little signal on the underlying target function.

Based on these bounds, one can also formally prove that a gradient-based 
method, under a reasonable model, will fail in returning a reasonable 
predictor, unless the number of iterations 
is exponentially large in $r$ and $d$
\footnote{Formally, this requires an oracle-based model, where
  given a point $\bw$, the algorithm receives the gradient at $\bw$ up
  to some arbitrary error much smaller than machine precision. See 
  \cite[Theorem 4]{shamir2016distribution} for details.} .
This provides strong evidence that
gradient-based methods indeed cannot learn random parities and linear-periodic 
functions. We emphasize that these results hold \emph{regardless of which class 
of predictors we use} (e.g. they can be arbitrarily complex neural
networks) -- the problem lies in using a gradient-based method to
train them. Also, we note that the difficulty lies in the random choice of 
$\bv^*$, and the problem is not difficult if $\bv^*$ is known and fixed in 
advance (for example, for a full parity $\bv^*=(1,\ldots,1)$, this problem is 
known to be solvable with an appropriate LSTM network 
\cite{hochreiter1997long}).

Finally, we remark that the connection between parities, 
difficulty of learning and orthogonal functions is not new, and has already 
been made in the context of statistical query 
learning \cite{kearns1998efficient,blum1994weakly}. This refers to algorithms 
which are constrained to interact with data 
by receiving estimates 
of the expected value of some query over the underlying distribution (e.g. the 
expected value of the first coordinate),  and it is well-known that parities 
cannot be learned with such algorithms. Recently, 
\cite{feldman2015statistical} have formally shown that 
gradient-based methods with an approximate gradient oracle can be implemented 
as a statistical query algorithm, which implies that gradient-based methods are 
indeed unlikely to solve learning problems which are known to be hard in the 
statistical queries framework, in particular parities. In the discussion on 
random parities above, we have simply made the connection between 
gradient-based methods and parities more explicit, by direct examination of 
gradients' variance w.r.t. the target function.

\section{Decomposition vs. End-to-end}\label{sec.Decomposition
  vs. End-to-end}
Many practical learning problems, and more generally, algorithmic
problems, can be viewed as a structured composition of
sub-problems. Applicable approaches for a solution can either be
tackling the problem in an end-to-end manner, or by
decomposition. Whereas for a traditional algorithmic solution, the
``divide-and-conquer'' strategy is an obvious choice, the ability of
deep learning to utilize big data and expressive architectures has
made ``end-to-end training'' an attractive alternative. Prior results
of end-to-end \cite{mnih2015human, graves2013speech} and decomposition
and added feedback \cite{gulccehre2016knowledge, hinton2006reducing,
  szegedy2015going, caruana1998multitask} approaches show success in
both directions. Here, we try to address the following questions: What is the
price of the rather appealing end-to-end approach? Is letting a
network ``learn by itself'' such a bad idea?  When is it necessary, or
worth the effort, to ``help'' it?

There are various aspects which can be considered in this context. For
example, \cite{shalev2016sample} analyzed the difference between the
approaches from the sample complexity point of view. Here, we focus on
the optimization aspect, showing that an end-to-end approach might suffer
from non-informative or noisy gradients, which may significantly affect the training
time. Helping the SGD process by decomposing the problem leads to much
faster training. We present a simple experiment, motivated by questions
every practitioner must answer when facing a new, non trivial problem:
What exactly is the required training data, what network architecture
should be used, and what is the right distribution of development
efforts. These are all correlated questions with no clear answer. Our
experiments and analysis show that making the wrong choice can be
expensive.

\subsection{Experiment}\label{sec:krect}
Our experiment compares the two approaches in a computer vision
setting, where convolutional neural networks (CNN) have become the
most widely used and successful algorithmic architectures.  We define
a family of problems, parameterized by $k\in\N$, and show a gap
(rapidly growing with $k$) between the performances of the end-to-end
and decomposition approaches. 

Let $X$ denote the space of $28\times28$ binary images, with a
distribution $D$ defined by the following sampling procedure:

\begin{itemize}
\item Sample $\theta\sim U([0,\pi])$, $l\sim U([5,28-5])$, $(x,y)\sim
  U([0,27])^2$.
\item The image $\bx_{\theta,l,(x,y)}$ associated with the above sample is set to $0$
  everywhere, except for a straight line of length $l$, centered at
  $(x,y)$, and rotated at an angle $\theta$. Note that as the images
  space is discrete, we round the values corresponding to the points
  on the lines to the closest integer coordinate.
\end{itemize}

Let us define an ``intermediate'' labeling function
$y:X\rightarrow \{\pm1\}$, denoting whether the line in a given image
slopes upwards or downwards, formally:
\[
y(\bx_{\theta,l,(x,y)})=
\begin{cases}
1 & ~\textrm{if}~\theta<\pi/2 \\
-1 & \textrm{otherwise}
\end{cases}.
\]

\begin{figure}[t]
\begin{center}
\includegraphics[width=0.8\textwidth]{\figsdir/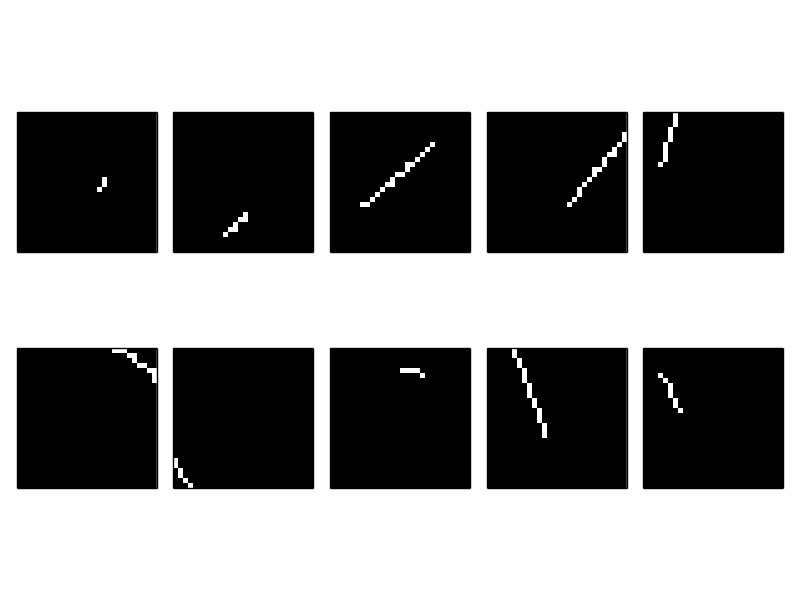}
\end{center}
\vskip -1.5cm
\caption{\secref{sec:krect}'s experiment - examples of samples from
  $X$. The $y$ values of the top and bottom rows are $1$ and $-1$, respectively.} \label{fig:krect_examples}
\end{figure}

\figref{fig:krect_examples} shows a few examples. We can now define
the problem for each $k$. Each input instance is a tuple $\bx_1^k:=(\bx_1,\ldots,\bx_k)$ 
of $k$ images sampled i.i.d. as above. The target output is the parity over the image
labels $y(\bx_1),\ldots,y(\bx_k)$, namely
$\tilde{y}(\bx_1^k)=\prod_{j=1...k}y(\bx_j)$.

Many architectures of DNN can be used for predicting $\tilde{y}(\bx_1^k)$
given $\bx_1^k$. A natural ``high-level'' choice can be:
\begin{itemize}
\item Feed each of the images, separately, to a single CNN (of some
  standard specific architecture, for example, LeNet-like), denoted
  $N^{(1)}_{\bw_1}$ and parameterized by its weights vector $\bw_1$,
  outputting a single scalar, which can be regarded as a ``score''.
\item Concatenate the ``scores'' of a tuple's entries, transform them
  to the range $[0,1]$ using a sigmoid function, and feed the
  resulting vector into another network, $N^{(2)}_{\bw_2}$, of a
  similar architecture to the one defined in \secref{sec.Parities
  and Linear-Periodic Functions}, outputting
  a single ``tuple-score'', which can then be thresholded for
  obtaining the binary prediction. 
\end{itemize}
Let the whole architecture be denoted $N_{\bw}$.  Assuming that
$N^{(1)}$ is expressive enough to provide, at least, a weak learner
for $y$ (a reasonable assumption), and that $N^{(2)}$ can express the
relevant parity function (see \lemref{lem:boringparity} in the
appendix), we obtain that this architecture has the potential for good
performance.

The final piece of the experimental setting is the choice of a loss
function. Clearly, the primary loss which we'd like to minimize is
the expected zero-one loss over the prediction, $N_\bw(\bx_1^k)$, and the
label, $\tilde{y}(\bx_1^k)$, namely:
\[
\tilde{L}_{0-1}(\bw) := \E_{\bx_1^k} \left[ N_\bw(\bx_1^k)\neq \tilde{y}(\bx_1^k) \right]
\]

A ``secondary'' loss which can be used in the decomposition approach
is the zero-one loss over the prediction of $N^{(1)}_{\bw_1}(\bx_1^k)$ and
the respective $y(\bx_1^k)$ value:
\[
L_{0-1}(\bw_1) := \E_{\bx_1^k} \left[ N^{(1)}_{\bw_1}(\bx_1^k)\neq y(\bx_1^k) \right]
\]

Let $\tilde{L},L$ be some differentiable surrogates for
$\tilde{L}_{0-1}, L_{0-1}$.  A classical end-to-end approach will
be to minimize $\tilde{L}$, and only it; this is our ``primary''
objective. We have no explicit desire for $N^{(1)}$ to output any
specific value, and hence $L$ is, a priori, irrelevant. A
decomposition approach would be to minimize both losses, under the
assumption that $L$ can ``direct'' $\bw_1$ towards an ``area'' in
which we know that the resulting outputs of $N^{(1)}$ can be separated
by $N^{(2)}$. Note that using $L$ is only possible when the $y$ values
are known to us.

Empirically, when comparing performances based on the ``primary''
objective, we see that the end-to-end approach is significantly inferior
to the decomposition approach (see
\figref{fig:rect_performance}). Using decomposition, we quickly arrive
at a good solution, regardless of the tuple's length, $k$ (as long as
$k$ is in the range where perfect input to $N^{(2)}$ is solvable by
SGD, as described in \secref{sec.Parities and Linear-Periodic
  Functions}). However, using the end-to-end approach works only for
$k=1,2$, and completely fails already when $k=3$ (or larger). This may
be somewhat surprising, as the end-to-end approach optimizes exactly
the primary objective, composed of two sub-problems each of which is
easily solved on its own, and with no additional irrelevant
objectives.

\setlength\figureheight{3.5cm}
\setlength\figurewidth{3.5cm}
\begin{figure}
\begin{center}
\input{\figsdir/RECT_performance_mytikz1.tex}
\input{\figsdir/RECT_performance_mytikz2.tex}
\input{\figsdir/RECT_performance_mytikz3.tex}
\input{\figsdir/RECT_performance_mytikz4.tex}
\end{center}
\caption{Performance comparison, \secref{sec:krect}'s experiment. The
  red and blue curves correspond to the end-to-end and
  decomposition approaches, respectively. The plots show the
  zero-one accuracy with respect to the primary objective, over a
  held out test set, as a function of training iterations. We have
  trained the end-to-end network for $20000$ SGD iterations, and the decomposition networks for only $2500$ iterations.} \label{fig:rect_performance}
\end{figure}
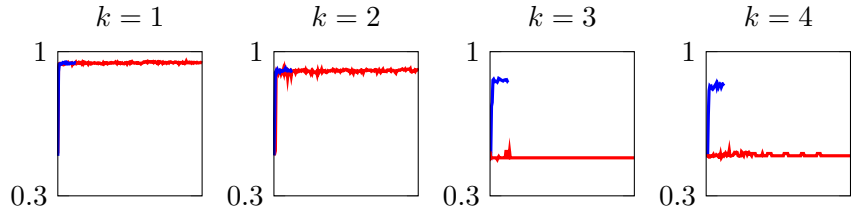

\subsection{Analysis}

We study the experiment from two directions: Theoretically, by analyzing the 
gradient variance (as in Section \ref{sec.Parities
  and Linear-Periodic Functions}), for a somewhat idealized version of the 
  experiment, and empirically, by estimating a signal-to-noise ratio (SNR) measure of 
  the stochastic gradients used by the algorithm. Both approaches point to a 
  similar issue: With the end-to-end approach, the gradients do not seem to 
  be sufficiently informative for the optimization process to succeed. 

Before continuing, we note that a conceptually similar experiment to ours has been reported in
\cite{gulccehre2016knowledge} (also involving a composition of an image recognition task and
a simple Boolean formula, and with qualitatively similar results). However, that experiment
came without a formal analysis, and the failure was attributed to local minima. In contrast,
our analysis indicates that the problem is not
due to local-minima (or saddle points), but from the gradients being non-informative and noisy. 

We begin with a theoretical result, which considers our experimental setup under two simplifying assumptions:
First, the input is assumed to be standard Gaussian, and second, we assume the labels are generated
by a target function of the form $h_{\bu}(\bx_1^k) = \prod_{l=1}^{k}\text{sign}(\bu^\top \bx_l)$. The first assumption is merely to simplify the analysis (similar results can be shown more generally, but the argument becomes more involved). The second assumption is equivalent to assuming that the labels $y(\bx)$ of individual images can be realized by a linear predictor, which is roughly the case for simple image labelling task such as ours. 

\begin{theorem}\label{thm:endtoend}
Let $\bx_1^k$ denote a k-tuple $(\bx_1,\ldots,\bx_k)$ of input instances, and assume
that each $\bx_l$ is i.i.d. standard Gaussian in $\reals^d$. 
Define
\[
h_{\bu}(\bx_1^k) = \prod_{l=1}^{k}\text{sign}(\bu^\top \bx_l),
\]
and the objective (w.r.t. some predictor $p_{\bw}$ parameterized by $\bw$)
\[
F(\bw) = \E_{\bx_1^k}\left[\ell(p_{\bw}(\bx_1^k),h_{\bu}(\bx_1^k)\right].
\]
Where the loss function $\ell$ is either the square loss 
$\ell(\hat{y},y) = \frac{1}{2}(\hat{y}-y)^2$ or a classification loss of 
the 
form $\ell(\hat{y},y)=r(\hat{y}\cdot y)$ for some $1$-Lipschitz function $r$.

Fix some $\bw$, and suppose that $p_{\bw}(\bx)$ is differentiable
w.r.t. $\bw$ and satisfies
$\E_{\bx_1^k}\left[\norm{\frac{\partial}{\partial
      \bw}p_\bw(\bx_1^k}^2\right]\leq G(\bw)^2$.  Then if
$\H=\{h_{\bu}:\bu\in\reals^d,\norm{\bu}=1\}$, then
\[
\Var(\H,F,\bw) ~\leq~ G(\bw)^2\cdot O\left(\sqrt{\frac{k\log(d)}{d}}\right)^{k}.
\]
\end{theorem}

The proof is given in
\secref{proof:thm:endtoend}. The theorem shows that the
``signal'' regarding $h_\bu$ (or, if applying to our experiment,
the signal for learning $N^{(1)}$, had $y$ been drawn uniformly at
random from some set of functions over $X$) decreases exponentially
with $k$. This is similar to the parity result in Section \ref{sec.Parities
	and Linear-Periodic Functions}, but with an important difference: Whereas the base of the exponent there was $1/2$,
here it is the much smaller quantity $k\log(d)/\sqrt{d}$ (e.g. in our experiment, we have $k\leq 4$ and $d=28^2$). 
This indicates that already for very small values of $k$, the information contained in the gradients about $\bu$ can become extremely small, and prevent gradient-based methods from succeeding, fully according with our experiment. 

To complement this analysis (which applies to an idealized version of our experiment), we consider a related ``signal-to-noise'' (SNR) quantity, which can be empirically estimated in our actual experiment. To motivate it, note that a key quantity used in the proof of \thmref{thm:endtoend}, for
estimating the amount of signal carried by the gradient, is the squared
norm of the correlation between the gradient of the predictor $p_\bw$,
$g(\bx_1^k):=\frac{\partial}{\partial \bw}p_\bw(\bx_1^k)$ and the target
function $h_\bu$, which we denote by $\textrm{Sig}_\bu$:
\[
\textrm{Sig}_\bu:=\left\|\E_{\bx_1^k}\left[h_\bu(\bx_1^k)g(\bx_1^k)\right]\right\|^2.
\]
We will consider the ratio between this quantity and a ``noise'' term
$\textrm{Noi}_\bu$, i.e. the variance of this correlation over the samples:
\[
\textrm{Noi}_\bu:=  \E_{\bx_1^k}\left\|h_\bu(\bx_1^k)g(\bx_1^k) -
  \E_{\bx_1^k}\left[h_\bu(\bx_1^k)g(\bx_1^k)\right]\right\|^2 .
\]
Since here the randomness is with respect to the data rather than the target function (as in \thmref{thm:endtoend}), we can estimate this SNR ratio in our experiment. It is well-known (e.g. \cite{ghadimi2013stochastic}) that the amount of noise in the stochastic gradient estimates used by stochastic gradient descent crucially affects its convergence rate. Hence, smaller SNR should be correlated with worse performance. 

We empirically estimated this SNR measure,
$\textrm{Sig}_y/\textrm{Noi}_y$, for the gradients w.r.t. the weights
of the last layer of $N^{(1)}$ (which potentially learns our
intermediate labeling function $y$) at the initialization point in
parameter space. The SNR estimate for various values of $k$ are
plotted in \figref{fig:rect_SNR}. We indeed see that when $k\ge 3$,
the SNR appears to approach extremely small values, where the
estimator's noise, and the additional noise introduced by a finite
floating point representation, can completely mask the signal, which
can explain the failure in this case.

\setlength\figureheight{5cm}
\setlength\figurewidth{8cm}
\begin{figure}
\begin{center}
\input{\figsdir/RECT_SNR_mytikz.tex}
\end{center}
\caption{\secref{sec:krect}'s experiment: comparing the SNR for the
  end-to-end approach (red) and the decomposition approach (blue), as
  a function of $k$, in $\log_e$
  scale. } \label{fig:rect_SNR}
\end{figure}
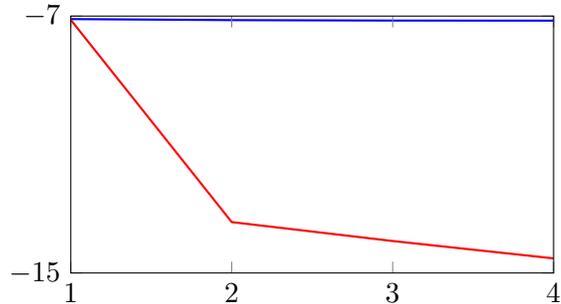

In \secref{exp:dec_vs_e2e} in the Appendix, we also present a second, more
synthetic, experiment, which demonstrates a case where the decomposition approach directly
decreases the stochastic noise in the SGD optimization process, hence benefiting 
the convergence rate. 


\section{Architecture and Conditioning}\label{sec.Piece-wise Linear
  AutoEncoders}

Network architecture choice is a crucial element in the success of
deep learning. New variants and development of novel architectures are
one of the main tools for achieving practical breakthroughs
\cite{He_2016_CVPR, szegedy2015going}. When choosing an architecture,
one consideration is how to inject prior knowledge on the problem at
hand, improving the network's expressiveness for that problem, while not
dramatically increasing sample complexity.  Another aspect involves
improving the computational complexity of training. In this section we
formally show how the choice of architecture affects the training time
through the lens of the condition number of the problem.

The study and practice of conditioning techniques, for convex and
non-convex problems, gained much attention recently (e.g.,
\cite{ioffe2015batch, kingma2014adam, duchi2011adaptive, shalev2011large}).
Here we show how architectural choice may have a dramatic effect on
the applicability of better conditioning techniques. 

The learning problem we consider in this section is that of encoding
one-dimensional, piecewise linear curves. We show how different
architectures, all of them of sufficient expressive power for solving
the problem, have orders-of-magnitude difference in their condition
numbers. In particular, this becomes apparent when considering
convolutional vs. fully connected layers. This sheds a new light over
the success of convolutional neural networks, which is generally
attributed to their sample complexity benefits. Moreover, we show how
conditioning, applied in conjunction with a better architecture
choice, can further decrease the condition number by orders of
magnitude.  The direct effect on the convergence rate is analyzed, and
is aligned with the significant performance gaps observed
empirically. We also demonstrate how performance may not significantly
improve by employing deeper and more powerful architectures, as well
as the price that comes with choosing a sub-optimal architecture.

\subsection{Experiments and Analysis}
We experiment with various deep learning solutions for encoding the
structure of one-dimensional, continuous, piecewise linear (PWL)
curves.  Any PWL curve with $k$ pieces can be written as:
$
f(x) = b + \sum_{i=1}^k a_i [x - \theta_i]_+ 
$,
where $a_i$ is the difference between the slope at the $i$'th segment
and the $(i-1)$'th segment. For example, the curve below can
be parametrized by $b = 1, a = (1, -2, 3), \theta = (0,2,6)$.

\begin{center}
\begin{tikzpicture}[scale=0.5, transform shape]
\draw[step=1cm,gray,very thin] (-2,-2) grid (9,4);
\draw[black,->] (0,-1) -- (0,4);
\draw[black,->] (-1,0) -- (9,0);
\draw[blue,very thick] (-1,1) -- (0,1) -- (2,3) -- (6,-1) -- (8,3);
\end{tikzpicture}
\end{center}

The problem we consider is that of receiving a vector of the values of
$f$ at $x \in \{0,1,\ldots,n-1\}$, namely
$\mathbf{f} := (f(0), f(1), \ldots, f(n-1))$, and outputting the
values of $b,\{a_i,\theta_i\}_{i=1}^k$. We can think of this problem
as an encoding problem, since we would like to be able to rebuild $f$
from the values of $b,\{a_i,\theta_i\}_{i=1}^k$.  Observe that
$b = f(0)$, so from now on, let us assume without loss of generality
that $b = 0$. 

Throughout our experiments, we use $n=100$, $k=3$. We sample
$\set{\theta_i}_{i\in[k]}$ uniformly without replacement from 
$\{0,1,\ldots,n-1\}$, and
sample each $a_i$ i.i.d. uniformly from $[-1,1]$.

\subsubsection{Convex Problem, Large Condition Number}\label{sec:FtoK}
As we assume that each $\theta_i$ is an integer in $\{0,1,\ldots,n-1\}$, we can 
represent $\{a_i,\theta_i\}_{i=1}^k$ as a vector
$\mathbf{p} \in \reals^n$ such that $p_j = 0$ unless there is some $i$
such that $\theta_i = j$, and in this case we set $p_j = a_i$. That
is, $p_j = \sum_{i=1}^k a_i \, \boldsymbol{1}_{[\theta_i = j-1]}$. 

This allows us to formalize the problem as a convex
optimization problem. Define a matrix $W \in \reals^{n,n}$ such that
$W_{i,j} = [i-j + 1]_+$. It is not difficult to show that
$\mathbf{f} = W \mathbf{p}$. Moreover, $W$ can be shown to be invertible, so we 
can extract
$\mathbf{p}$ from $\mathbf{f}$ by $\mathbf{p} = W^{-1} \mathbf{f}$.

We hence start by attempting to learn this linear transformation directly, using
a connected architecture of one layer, with $n$ output channels. Let the
weights of this layer be denoted $\hat{U}$. We therefore minimize the
objective:
\begin{equation}\label{obj:pwl1}
\min_{\hat{U}}\E_\f\left[\half (W^{-1} \mathbf{f} - \hat{U} \mathbf{f})^2\right]
\end{equation}
where $\f$ is sampled according to some distribution.  As a convex,
realizable (by $\hat{U}=W^{-1}$) problem, convergence is guaranteed, and we
can explicitly analyze its rate. However, perhaps unexpectedly, we observe a 
very slow rate of convergence to a satisfactory solution, where significant 
inaccuracies are present at the non-smoothness points. \figref{fig:FtoK} 
illustrates the results.

To analyze the convergence rate of this approach, and to benchmark the 
performance of the next set of experiments, we start
off by giving an explicit expression for $W^{-1}$:
\begin{lemma} \label{lem:invW}
The inverse of $W$ is the matrix $U$ s.t. $U_{i,i} = U_{i+2,i} = 1,
U_{i+1,i} = -2$,  and the rest of the coordinates of $U$ are zero.
\end{lemma}
The proof is given in \appref{proof:lem:invW}. Next, we analyze the
iteration complexity of SGD for learning the matrix $U$. To that end,
we give an explicit expression for the expected value of the learned
weight matrix at each iteration $t$, denoted as $\hat{U}^t$:
\begin{lemma} \label{lem:hatUt} Assume $\hat{U}^0=0$, and that
  $\E_\f[ U\f \f^\top U^\top] = \lambda I$ for some $\lambda$. Then, running 
  SGD with learning
  rate $\eta$ over objective \ref{obj:pwl1} for $t$ iterations yields:
\[
\E \hat{U}_{t} = \eta \lambda W^\top  \sum_{i=0}^{t-1} (I-\eta \lambda W W^\top)^i
\]
\end{lemma}
The proof is given in \appref{proof:lem:hatUt}. Note that the assumption
that $\E_\f[ U\f \f^\top U^\top] = \lambda I$ holds under the
distributional assumption over the curves, as changes of direction in the curve are independent, and are
sampled each time from the same distribution.  The following 
theorem
establishes a lower bound on $\norm{\E \hat{U}_{t+1}-U}$, which by Jensen's 
inequality, implies a lower bound on $\E\norm{\hat{U}_{t+1}-U}$, the expected 
distance of $\hat{U}_{t+1}$ from $U$. Note that the lower bound holds 
even if we use all the data for
updating (that is, gradient descent and not stochastic gradient descent).
\begin{theorem}\label{thm:cond_converge}
Let $W = Q S V^\top$ be the singular value decomposition
  of $W$. If $\eta\, \lambda \,S_{1,1}^2 \ge 1$ then $\E
  \hat{U}_{t+1}$ diverges. Otherwise, we have 
\[
t+1 \le \frac{S_{1,1}^2}{2\,S_{n,n}^2} ~~~\Rightarrow~~~ \| \E
\hat{U}_{t+1} - U\| \ge 0.5 ~,
\]
where the norm is the spectral norm.  Furthermore, the condition number
$\frac{S_{1,1}^2}{S_{n,n}^2}$ (where $S_{1,1},S_{n,n}$ are the top and 
bottom singular values of $W$) is $\Omega(n^{3.5})$.
\end{theorem}
The proof is given in \appref{proof:thm:cond_converge}. The theorem
implies that the condition number of W, and hence, the number of GD
iterations required for convergence, scales quite poorly with $n$. In
the next subsection, we will try to decrease the condition number of
the problem.
\subsubsection{Improved Condition Number through Convolutional
  Architecture}\label{sec:FtoKConv}
Examining the explicit expression for $U$ given in \lemref{lem:invW},
we see that $U \mathbf{f}$ can be written as a one-dimensional
convolution of $\mathbf{f}$ with the kernel $[1,-2,1]$. Therefore, the
mapping from $\mathbf{f}$ to $\mathbf{p}$ is realizable using a
convolutional layer.

Empirically, convergence to an accurate solution is faster using this
architecture. \figref{fig:FtoKConv} illustrates a few
examples. To theoretically understand the benefit of using a
convolution, from the perspective of the required number of iterations
for training, we will consider the new problem's condition number,
providing understanding of the gap in training time.  In the previous
section we saw that GD requires $\Omega(n^{3.5})$ iterations to learn
the full matrix $U$. In the appendix (sections \ref{sec:GDanalysis}
and \ref{sec:proofCondFtoKConv}) we show that under some mild assumptions, the 
condition number is only 
$\Theta(n^3)$, and GD requires only that order of iterations to learn the 
optimal 
filter $[1,-2,1]$.

\subsubsection{Additional Improvement through Explicit Conditioning}\label{sec:FtoKConvCond}

In \secref{sec:FtoKConv}, despite observing an improvement from the fully
connected architecture, we saw that GD still requires $\Omega(n^3)$
iterations even for the simple problem of learning the filter
$[1,-2,1]$. This motivates an application of additional conditioning
techniques, in the hope for extra performance gains.

First, let us explicitly represent the convolutional architecture as
a linear regression problem. We perform Vec2Row operation on
$\mathbf{f}$ as follows: given a sample $\f$, construct a matrix, $F$,
of size $n \times 3$, such that the $t$'th row of $F$ is
$[\f_{t-1}, \f_t, \f_{t+1}]$. Then, we obtain a vanilla linear
regression problem in $\reals^3$, with the filter $[1,-2,1]$ as its
solution. Given a sample $\f$, we can now approximate the correlation
matrix of $F$, denoted $C \in \reals^{3,3}$, by setting
$C_{i,j} = \E_{\f,t} [\f_{t-2+i} \f_{t-2+j}]$. We then
calculate the matrix $C^{-1/2}$ and replace every instance (namely, a
row of $F$)
$[\f_{t-1},\f_t,\f_{t+1}]$ by the instance
$[\f_{t-1},\f_t,\f_{t+1}]C^{-1/2}$. By construction, the correlation
matrix of the resulting instances is approximately the identity
matrix, hence the condition number is approximately $1$. It follows
(see again \appref{sec:GDanalysis}) that SGD converges using order of
$\log(1/\epsilon)$ iterations, independently of $n$. Empirically, we
quickly converge to extremely accurate results, illustrated in
\figref{fig:FtoKConvCond}.

We note that the use of a convolution architecture is crucial for the
efficiency of the conditioning; had the dimension of the problem not
been reduced so dramatically, the difficulty of estimating a large $n\times n$
correlation matrix scales strongly with $n$, and furthermore, its inversion
becomes a costly operation. The combined use of a better architecture and
of conditioning is what allows us to gain this dramatic improvement.

\subsubsection{Perhaps I should use a deeper network?}\label{sec:FAutoEncoder}

The solution arrived at in \secref{sec:FtoKConvCond} indicates that a suitable 
architecture choice and conditioning scheme
can provide training time speedups of multiple orders of
magnitude. Moreover, the benefit of reducing the number of parameters,
in the transition from a fully connected architecture to a
convolutional one, is shown to be helpful in terms of convergence
time. However, we should not rule out the possibility that a deeper,
wider network will not suffer from the deficiencies analyzed above for
the convex case.

Motivated by the success of deep auto-encoders, we experiment with a
deeper architecture for encoding $\f$. Namely, we minimize
$
\min_{\bv_1,\bv_2}\E_\f[(\f-M_{\bv_2}(N_{\bv_1}(\f)))^2]
$,
Where $N_{\bv_1},M_{\bv_2}$ are deep networks parametrized by their
weight vectors $\bv_1,\bv_2$, with the output of $N$ being of
dimension $2k$, enough for realization of the encoding problem. Each
of the two networks has three layers with ReLU activations, except for the output
layer of $M$ having a linear activation. The dimensions of the layers
are, $500,100,2k$ for $N$, and $100,100,n$ for $M$.

Aligned with the intuition gained through the previous experiments, we
observe that additional expressive power, when unnecessary, does not
solve inherent optimization problems, as this stronger Auto-Encoder
fails to capture the fine details of $\f$ at its non-smooth
points. See \figref{fig:FAutoEncoder} for examples.

\setlength\figureheight{3.3cm}
\setlength\figurewidth{3.3cm}
\begin{figure}
\centering
\begin{subfigure}[b]{1\textwidth}
\begin{center}
\input{\figsdir/FtoKTrainer_500_mytikz1.tex}
\input{\figsdir/FtoKTrainer_10000_mytikz1.tex}
\input{\figsdir/FtoKTrainer_50000_mytikz1.tex}
\input{\figsdir/FtoKTrainer_500_mytikz4.tex}
\input{\figsdir/FtoKTrainer_10000_mytikz4.tex}
\input{\figsdir/FtoKTrainer_50000_mytikz4.tex}
\end{center}
\caption{\secref{sec:FtoK}'s experiment - linear architecture.} \label{fig:FtoK}
\end{subfigure}

\begin{subfigure}[b]{1\textwidth}
\begin{center}
\input{\figsdir/FtoKConvTrainer_500_mytikz1.tex}
\input{\figsdir/FtoKConvTrainer_10000_mytikz1.tex}
\input{\figsdir/FtoKConvTrainer_50000_mytikz1.tex}
\input{\figsdir/FtoKConvTrainer_500_mytikz4.tex}
\input{\figsdir/FtoKConvTrainer_10000_mytikz4.tex}
\input{\figsdir/FtoKConvTrainer_50000_mytikz4.tex}
\end{center}
\caption{\secref{sec:FtoKConv}'s experiment - convolutional
  architecture. } \label{fig:FtoKConv}
\end{subfigure}

\begin{subfigure}[b]{1\textwidth}
\begin{center}
\input{\figsdir/FtoKConvCondTrainer_500_mytikz1.tex}
\input{\figsdir/FtoKConvCondTrainer_10000_mytikz1.tex}
\input{\figsdir/FtoKConvCondTrainer_50000_mytikz1.tex}
\input{\figsdir/FtoKConvCondTrainer_500_mytikz4.tex}
\input{\figsdir/FtoKConvCondTrainer_10000_mytikz4.tex}
\input{\figsdir/FtoKConvCondTrainer_50000_mytikz4.tex}
\end{center}
\caption{\secref{sec:FtoKConvCond}'s experiment - convolutional
  architecture with conditioning.} \label{fig:FtoKConvCond}
\end{subfigure}

\begin{subfigure}[b]{1\textwidth}
\begin{center}
\input{\figsdir/FAutoEncoderTrainer_500_mytikz1.tex}
\input{\figsdir/FAutoEncoderTrainer_10000_mytikz1.tex}
\input{\figsdir/FAutoEncoderTrainer_50000_mytikz1.tex}
\input{\figsdir/FAutoEncoderTrainer_500_mytikz4.tex}
\input{\figsdir/FAutoEncoderTrainer_10000_mytikz4.tex}
\input{\figsdir/FAutoEncoderTrainer_50000_mytikz4.tex}
\end{center}
\caption{\ref{sec:FAutoEncoder}'s experiment - vanilla auto encoder.} \label{fig:FAutoEncoder}
\end{subfigure}
\caption{Examples for decoded outputs of \secref{sec.Piece-wise Linear
  AutoEncoders}'s experiments, learning
  to encode PWL curves. In blue are the
  original curves. In red are the decoded curves. The plot shows the
  outputs for two curves, after 500, 10000, and 50000 iterations, from
  left to right.} 
\label{fig:sec.Piece-wise Linear AutoEncoders}
\end{figure}
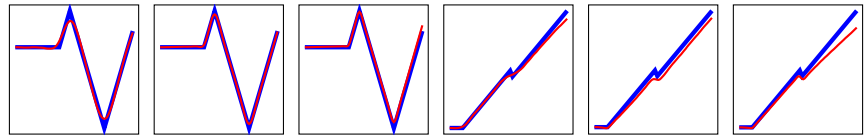
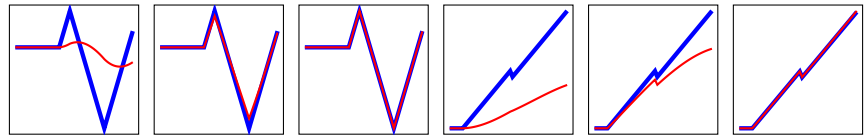
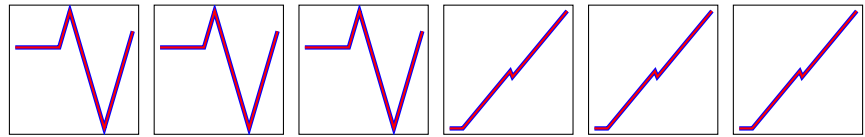
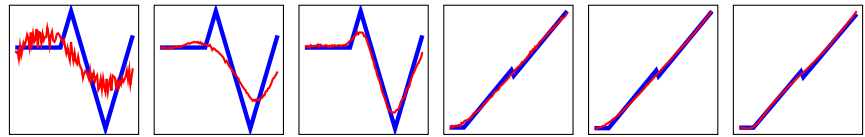

\section{Flat Activations}\label{sec.Non Continuous
  non-linearities}

We now examine a different aspect of gradient-based learning which
poses difficulties for optimization: namely, flatness of the loss
surface due to saturation of the activation functions, leading to
vanishing gradients and a slow-down of the training process. This
problem is amplified in deeper architectures, since it is likely that
the backpropagated message to lower layers in the architecture would
vanish due to a saturated activation somewhere along the way. This is
a major problem when using sigmoids as a gating mechanisms in
Recurrent Neural Networks such as LSTMS and GRUs
\cite{graves2005framewise, chung2014empirical}.

While non-local search-based optimization for large scale problems
seems to be beyond reach, variants on the gradient
update, whether by adding momentum, higher order methods, or
normalized gradients, are quite successful, leading to consideration of
update schemes deviating from ``vanilla'' gradient updates. 

In this section, we consider a family of activation functions which
amplify the ``vanishing gradient due to saturated activation''
problem; they are piecewise flat. Using such activations in a neural
network architecture will result in a gradient equal to 0, which will
be completely useless. We consider different ways to implement,
approximate or learn such activations, such that the error will
effectively propagate through them. Using a different variant of a
local search-based update, based on \cite{kalai2009isotron,kakade2011efficient} 
, we
arrive at an efficient solution. Convergence guarantees exist for a one-layer
architecture. We leave further study of deeper networks to future
work. 

\subsection{Experimental Setup}
Consider the following optimization setup. 
The sample space $X\subset \reals^d$ is symmetrically
 distributed. The 
 target function $y:\reals^d\rightarrow\reals$ is of the
  form $y(\bx) = u(\bv^{*\top} \bx+b^*)$, 
where $\bv^*\in\reals^d$, $b^*\in\reals$, and $u:\reals\rightarrow\reals$ is a
monotonically non-decreasing function. 
The objective of the optimization problem is given by:
\[
\min_{\bw} \E_x \left[\ell\left( u(N_{\bw}(\bx)), y(\bx)\right)\right]
\]
where $N_{\bw}$ is some neural network parametrized by $\bw$, and
$\ell$ is some loss function (for example, the squared or absolute
difference).

For the experiments, we use $u$ of the form:
\[
  u(r)=z_0 + \sum_{i\in[\numsteps]}\mathbf{1}_{[r>z_i]}\cdot \left(
    z_i-z_{i-1}\right),
\]
where $z_0 < z_1 < \ldots < z_{\numsteps}$ are known. In words, 
given $r$, the function rounds down to the nearest $z_i$. We
also experiment with normally distributed $X$. Our theoretical
analysis is not restricted to $u$ of this specific form, nor to normal
$X$. All figures are found in \figref{fig:steps}.

Of course, applying gradient-based methods to solve this problem
directly, is doomed to fail as the derivative of $u$ is identically
0. Is there anything which can be done instead?

\subsection{Non-Flat Approximation Experiment}\label{sec:non_flat}
We start off by trying to approximate $u$ using a non flat function
$\tilde{u}$ defined by
\[ 
\tilde{u}(r) = z_0 + \sum_{i\in[\numsteps]}
  (z_i-z_{i-1})\cdot\sigma(c\cdot(r-z_i)),
\] 
where $c$ is some constant, and $\sigma$ is the sigmoid
function $\sigma(z)=(1+\exp(-z))^{-1}$.
Intuitively, we approximate the ``steps'' in $u$ using a sum
of sigmoids, each of amplitude corresponding to the step's height, and
centered at the step's position. This is similar to the motivation for
using sigmoids as activation functions and as gates in LSTM cells ---
a non-flat approximation of the step function. Below is an example for $u$, and 
its
approximation $\tilde{u}$.

\setlength\figureheight{4cm}
\setlength\figurewidth{7cm}
\begin{center}
\begin{tikzpicture}
\input{\figsdir/smooth_approx_mytikz.tex}
\end{tikzpicture}
\end{center}

 The objective is the
expected squared loss, propagated through $\tilde{u}$, namely
\[ 
  \min_{\bv,b} \E_\bx \left[\left( \tilde{u}(\bv^\top \bx+b) -
      y(\bx)\right)^2\right].
\] 
Although the objective is not completely flat,
and is continuous, it suffers from the flatness and non continuity deficiencies of the
original $u$, and training using this objective is much slower, and
sometimes completely failing. In particular, sensitivity to the initialization of
bias term is observed, where the wrong initialization can cause the starting
point to be in a very wide flat region of $u$, and hence a very flat
region of $\tilde{u}$. 

\subsection{End-to-End Experiment}\label{sec:e2e}
Next, we attempt to solve the problem using improper learning, with
the objective now being:
\[
\min_{\bw} \E_\bx \left[\left( N_\bw(x) - y(\bx)\right)^2\right]
\]
where $N_\bw$ is a network parametrized by its weight vector $\bw$. We use
a simple architecture of four fully connected layers, the first three
with ReLU activations and 100 output channels, and the last, with only
one output channel and no activation function. 

As covered in \secref{sec.Piece-wise Linear
  AutoEncoders}, difficulty arises when regressing to non smooth
functions. In this case, with $u$ not even being continuous, the
inaccuracies in capturing the non continuity points are brought to the
forefront.  Moreover, this solution has its extra price in terms of
sample complexity, training time, and test time, due to the use of a
much larger than necessary network. An advantage is of course the
minimal prior knowledge about $u$ which is required. While this
approach manages to find a reasonable solution, it is far from being
perfect. 

\subsection{Multi-Class Experiment}\label{sec:mc}
In this experiment, we approach the problem as a general multi-class
classification problem, with each value of the image of $u$ is treated
as a separate class. We use a similar architecture to that of the end-to-end
experiment, with one less hidden layer, and with the final layer
outputting $\numsteps$ outputs, each corresponding to one of the steps defined
by the $z_i$s.
A problem here is the inaccuracies at the boundaries between
classes, due to the lack of structure imposed over the predictor. 
The fact that the linear connection between $x$ and the input
to $u$ is not imposed through the architecture results in
``blurry'' boundaries. In addition, the fact that we rely on an
``improper'' approach, in the sense that we ignore the ordering
imposed by $u$, results in higher sample complexity. 

\subsection{The ``Forward-Only'' Update Rule}\label{sec:fo}
Let us go back to a direct formulation of the problem, in the form of the 
objective function 
\[
  \min_{\bw} F(\bw) = \E_x \left[(u(\bw^\top \bx)-y(\bx))^2\right]
\] 
where $y(\bx) = u(\bv^{*\top}\bx)$. The gradient update rule in this case is
$\bw^{(t+1)} = \bw^{(t)} - \eta \nabla F(\bw^{(t)})$, where for our
objective we have
\[
\nabla F(\bw) = \E_\bx \left[ (u(\bw^\top
  \bx) - y(\bx) )\cdot u'(\bw^\top \bx) \cdot \bx\right]
\]
Since $u'$ is zero a.e., the gradient update is meaningless. 
\cite{kalai2009isotron,kakade2011efficient} 
proposed to replace the gradient with the following:
\begin{equation}\label{glmtron_update}
\tilde{\nabla} F(\bw) = \E_\bx \left[ (u(\bw^\top
  \bx) - y(\bx) )\cdot \bx\right]
\end{equation}
In terms of the backpropagation algorithm, this kind of update can be
interpreted as replacing the backpropagation message for the
activation function $u$ with an identity message. For notation
simplicity, we omitted the bias terms $b,b^*$, but the same
Forward-only concept is applied to them too.

This method empirically achieves the best results, both in terms of
final accuracy, training time, and test time cost.  As mentioned
before, the method is due to \cite{kalai2009isotron,
  kakade2011efficient}, where it is proven to converge to an $\epsilon$-optimal 
  solution in $O(L^2/\epsilon^2)$, under the additional assumptions that the 
  function $u$ is $L$-Lipschitz, and that $\bw$ is constrained to have bounded 
  norm. For completeness, we provide a short proof in  
  \appref{proof:glmtron}.

\setlength\figureheight{4cm}
\setlength\figurewidth{6cm}
\begin{figure}
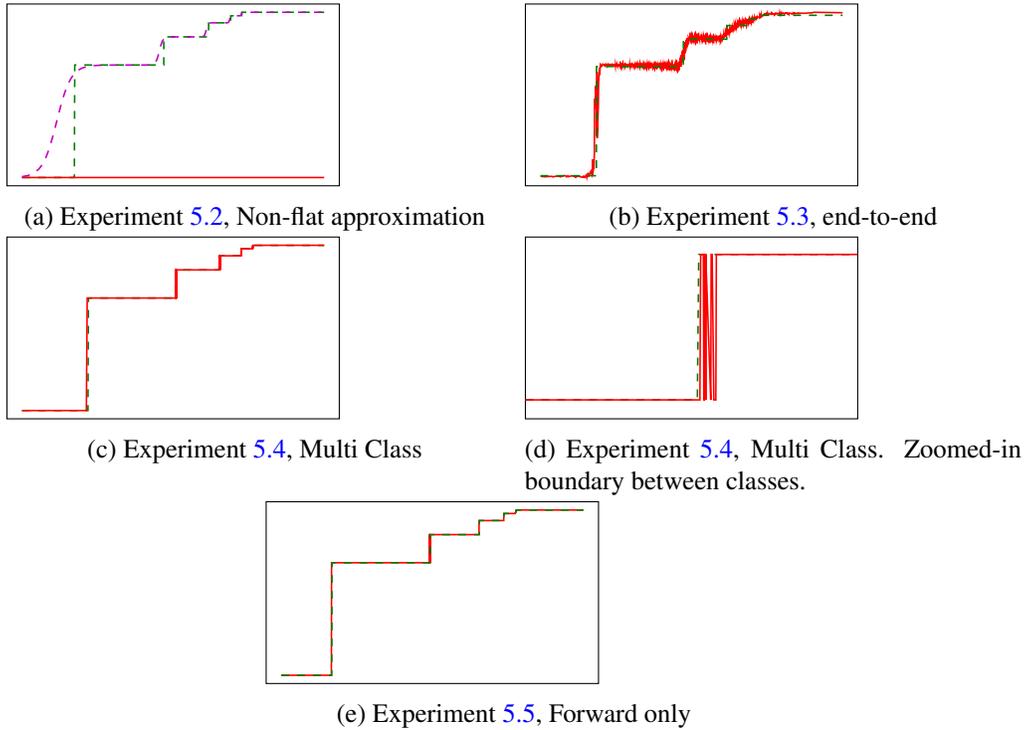

\begin{center}
\begin{subfigure}[t]{0.4\textwidth}
\input{\figsdir/DifferentiableApproximationFloorLearn_mytikz.tex}
\caption{Experiment \ref{sec:non_flat}, Non-flat approximation}
\end{subfigure}
~
\begin{subfigure}[t]{0.4\textwidth}
\input{\figsdir/EndToEndFloorLearn_mytikz.tex}
\caption{Experiment \ref{sec:e2e}, end-to-end}
\end{subfigure}

\begin{subfigure}[t]{0.4\textwidth}
\input{\figsdir/MCFloorLearn_mytikz.tex}
\caption{Experiment \ref{sec:mc}, Multi Class}
\end{subfigure}
~
\begin{subfigure}[t]{0.4\textwidth}
\input{\figsdir/MCFloorLearn_mytikz_zoom.tex}
\caption{Experiment \ref{sec:mc}, Multi Class. Zoomed-in boundary between
classes.}
\end{subfigure}

\begin{subfigure}[t]{0.4\textwidth}
\input{\figsdir/IsotronFloorLearn_mytikz.tex}
\caption{Experiment \ref{sec:fo}, Forward only}
\end{subfigure}
\end{center}
\caption{\secref{sec.Non Continuous non-linearities}'s Experiment:
  backpropagating through the steps function. The horizontal axis is
  the value of $r=v^{*\top} x+b^*$. The dashed green curves show the
  label, $u(r)$. The red curves show the outputs of the learnt
  hypothesis. The plots are zoomed around the mean of $r$. In the
  non-flat approximation plot, the dashed magenta curve shows the
  non-flat approximation, namely $\tilde{u}(r)$. Note the inaccuracies
  around the boundaries between classes in the Multi Class
  experiment. All plots show the results after 10000 training
  iterations, except for the forward only plot, showing results after
  5000 iterations.} \label{fig:steps}
\end{figure}

\section{Summary}\label{sec.discussion}
In this paper, we considered different families of
problems, where standard gradient-based deep learning approaches appear to suffer from significant difficulties. Our analysis indicates that these difficulties are not necessarily related to stationary point issues such as spurious local minima or a plethora of saddle points, but rather more subtle issues: Insufficient information in the gradients about the underlying target function; low SNR; bad conditioning; or flatness in the activations (see \figref{fig:summary} for a graphical illustration). We consider it as a first step towards a better  understanding of where standard deep learning methods might fail, as well as what approaches might overcome these failures. 

\setlength\figureheight{6cm}
\setlength\figurewidth{6cm}
\begin{figure}
\begin{center}
\begin{subfigure}[t]{0.4\textwidth}
\input{\figsdir/extremely_flat_mytikz.tex}
\caption{Extremely small variance in the loss surface's gradient,
  w.r.t. different target functions, each with a very different optimum.}
\end{subfigure}
~
\begin{subfigure}[t]{0.4\textwidth}
\input{\figsdir/low_snr_mytikz.tex}
\caption{Low SNR of gradient estimates. The dashed lines represent
  losses w.r.t. different samples, each implying a very different
  estimate than the average gradient.}
\end{subfigure}

\begin{subfigure}[t]{0.4\textwidth}
\input{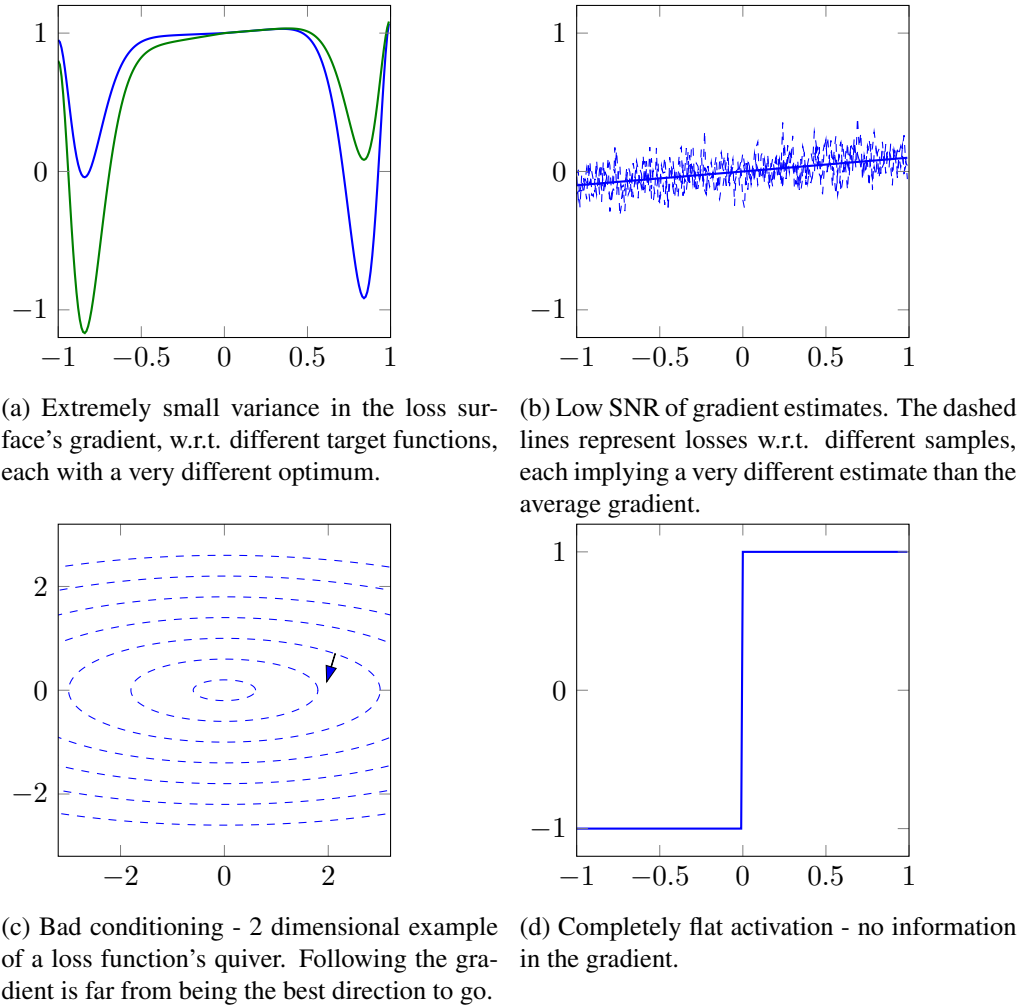}
\caption{Bad conditioning - 2 dimensional example of a loss function's
  quiver. Following the gradient is far from being the best direction
  to go.}
\end{subfigure}
~
\begin{subfigure}[t]{0.4\textwidth}
\input{\figsdir/completely_flat_mytikz.tex}
\caption{Completely flat activation - no information in the gradient.}
\end{subfigure}

\end{center}
\caption{A graphical summary of limitations of gradient-based
  learning.} \label{fig:summary}
\end{figure}

\paragraph{Acknowledgements:} This research is supported in part by
the Intel Collaborative Research Institute for Computational
Intelligence (ICRI-CI), and by the European Research Council (TheoryDL
project). OS was also supported in part by an FP7 Marie Curie CIG
grant and an Israel Science Foundation grant 425/13.

\bibliography{mybib}
\bibliographystyle{plain}

\appendix
\section{Reduced Noise through Decomposition -
  Experiment}\label{exp:dec_vs_e2e}
\subsection{Experiment}
For this experiment, consider the problem of training a predictor,
which given a ``positive media reference'' $\bx$ to a certain stock
option, will distribute our assets between the $k=500$ stocks in the
S\&P500 index in some manner.  One can, again, come up with two rather
different strategies for solving the problem.
\begin{itemize}
\item An end-to-end approach: train a deep network $N_\bw$ that
  given $\bx$ outputs a distribution over the $k$ stocks. The
  objective for training is maximizing the gain obtained by allocating
  our money according to this distribution.
\item A decomposition approach: train a deep network $N_\bw$ that
  given $\bx$ outputs a single stock, $y \in [k]$, whose future gains
  are the most positively correlated to $\bx$. Of course, we may need
  to gather extra labeling for training $N_\bw$ based on this
  criterion. 
\end{itemize}

We make the (non-realistic)
assumption that every instance of media reference is strongly and
positively correlated to a single stock $y\in[k]$, and it has no
correlation with future performance of other stocks.  This obviously
makes our problem rather toyish; the stock exchange and media worlds
have highly complicated correlations. However, it indeed arises from,
and is motivated by, practical problems.

To examine the problem in a simple and theoretically clean manner, we
design a synthetic experiment defined by the following optimization
problem: Let $X\times Z\subset\reals^d\times\set{\pm 1}^k$ be the
sample space, and let $y:X\rightarrow[k]$ be some labelling
function. We would like to learn a mapping
$N_{\bw}:X\rightarrow S^{k-1}$, with the objective being:
\[
\min_{\bw} ~L(\bw) := \E_{\bx,\bz\sim X\times Z} \left[ -\bz^\top N_{\bw}(\bx)  \right].
\]

To connect this to our story, $N_{\bw}(\bx)$ is our asset
distribution, $\bz$ indicates the future performance of the stocks,
and thus, we are seeking minimization of our expected future negative
gains, or in other words, maximization of expected profit.
We further assume that given $\bx$, the coordinate $\bz_{y(\bx)}$
equals $1$, and the rest of the coordinates are sampled i.i.d from
the uniform distribution over $\{\pm 1\}$. 

Whereas in \secref{sec:krect}'s experiment, the difference between the
end-to-end and decomposition approaches could be summarized by a
different loss function choice, in this experiment, the difference
boils down to the different gradient estimators we would use,
where we are again taking as a given fact that exact gradient
computations are expensive for large-scale problems, implying the
method of choice to be SGD. For the purpose of the experimental
discussion, let us write the two estimators explicitly as two
unconnected update rules. We will later analyze their (equal)
expectation.

For an end-to-end approach, we sample a pair
$(\bx,\bz)$, and use $\nabla_\bw(-\bz^\top N_{\bw}(\bx))$ as a
gradient estimate. It is clear that this is an unbiased estimator of
the gradient.

For a decomposition approach, we sample a pair $(\bx,\bz)$,
\emph{completely ignore} $\bz$, and instead, pay the extra costs and
gather the required labelling to get $y(\bx)$. We will then use
$\nabla_\bw(-e_{y(\bx)}^\top N_{\bw}(\bx))$ as a gradient
estimate. It will be shown later that this too is an unbiased
estimator of the gradient.

\figref{fig:decVSe2e} clearly shows that optimizing using the
end-to-end estimator is inferior to working with the
decomposition one, in terms of training time and final accuracy,
to the extent that for large $k$, the end-to-end estimator cannot
close the gap in performance in reasonable time. 

\setlength\figureheight{4cm}
\setlength\figurewidth{4cm}
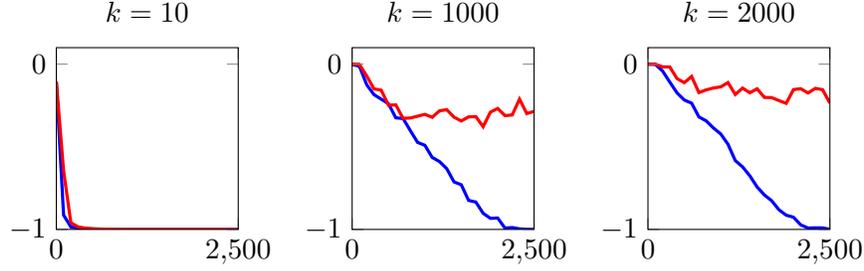
\begin{figure}
\begin{center}
\input{\figsdir/Decompositionvs.EndtoEnd_mytikz10.tex}
\input{\figsdir/Decompositionvs.EndtoEnd_mytikz1000.tex}
\input{\figsdir/Decompositionvs.EndtoEnd_mytikz2000.tex}
\end{center}
\caption{Decomposition vs. end-to-end Experiment: Loss as a function
  of the number of training iterations, for input dimension $d=1000$
  and for various $k$ values. The red and blue curves correspond to the losses of the
  end-to-end and decomposition estimators, respectively.} \label{fig:decVSe2e}
\end{figure}
\subsection{Analysis}
We examine the experiment from a SNR perspective. First, let us show
that indeed, both estimators are unbiased estimators of the true
gradient. As stated above, it is clear, by definition of $L$, that the
end-to-end estimator is an unbiased estimator of $\nabla_\bw
L(\bw)$. To observe this is also the case for the decomposition
estimator, we write:
\begin{align*}
&\nabla_\bw L(\bw)=\nabla_\bw\E_{\bx,\bz}[-\bz^\top N_{\bw}(\bx)]\\
=&\E_{\bx}[\E_{\bz|\bx}[\nabla_\bw(-\bz^\top N_{\bw}(\bx))]]\\
\overset{(1)}{=}&\E_{\bx}[\E_{\bz|\bx}[-\bz^\top\nabla_\bw(N_{\bw}(\bx))]]
\overset{(2)}{=}\E_{\bx}[-e_{y(\bx)}^\top\nabla_\bw(N_{\bw}(\bx))]
\end{align*} 
where $(1)$ follows from the chain rule, and $(2)$ from the assumption
on the distribution of $\bz$ given $\bx$. It is now easy to see that
the decomposition estimator is indeed a (different) unbiased estimator of the
gradient, hence the ``signal'' is the same.

Intuition says that when a choice between two unbiased estimators is
presented, we should choose the one with the lower variance. In our
context, \cite{ghadimi2013stochastic} showed that when running
SGD (even on non-convex objectives), arriving at a point where
$\|\nabla_\bw L(\bw)\|^2 \le \epsilon$ requires order of
$\bar{\nu}^2/\epsilon^2$ iterations, where
\[
  \bar{\nu}^2 = \max_t \E_{\bx,q} \|\nabla_\bw^t(\bx,q)\|^2 - \|\nabla_\bw L(\bw^{(t)}) \|^2,
\]
$\bw^t$ is the weight vector at time $t$, $q$ is sampled along
with $\bx$ (where it can be replaced by $\bz$ or $y(\bx)$, in our
experiment), and $\nabla_\bw^t$ is the unbiased estimator for the
gradient. This serves as a motivation for analyzing the problem
through this lens.

Motivated by \cite{ghadimi2013stochastic}'s result, and by our results
regarding \secref{sec:krect}, we examine the quantity
$\E_{\bx,q} \|\nabla_\bw^t(\bx,q)\|^2$, or ``noise'', explicitly. For
the end-to-end estimator, this quantity equals
\[
\E_{\bx,\bz} \|-\bz^\top\nabla_\bw N_\bw(\bx)\|^2 
= \E_{\bx,\bz} \|-\sum_{i=1}^k\bz_i\nabla_\bw N_\bw(\bx)_i\|^2 
\]
Denoting by $G_i:=\nabla_\bw N_\bw(\bx)_i$, we get:
\begin{equation}\label{est_e2e}
=\E_{\bx}\E_{\bz|\bx} \|-\sum_{i=1}^k\bz_iG_i\|^2 =\E_{\bx} \sum_{i=1}^k\|G_i\|^2 
\end{equation}
where the last equality follows from expanding the squared sum, and
taking expectation over $\bz$, while noting that mixed terms cancel out
(from independence of $\bz$'s coordinates), and that 
$\bz_i^2=1$ for all $i$.

As for the decomposition estimator, it is easy to see that
\begin{equation}\label{est_dec}
\E_{\bx} \|-e_{y(\bx)}^\top\nabla_\bw N_\bw(\bx)\|^2 = \E_{\bx} \|G_{y(\bx)}\|^2.
\end{equation}

Observe that in \ref{est_e2e} we are summing up, per $\bx$, $k$
summands, compared to the single element in \ref{est_dec}. When 
randomly initializing a network it is likely that the
values of $\|G_i\|^2$ are similar, hence we obtain that at the beginning of
training, the variance of the end-to-end estimator is roughly $k$
times larger than that of the decomposition estimator.  

\section{Proofs}

\subsection{Proof of Theorem \ref{thm:sqdim}}\label{proof:thm:sqdim}
\begin{proof}
Given two square-integrable functions $f,g$ on an Euclidean space $\reals^n$, 
let 
$\inner{f,g}_{L_2}=\E_{\bx}[f(\bx)g(\bx)]$ and 
$\norm{f}_{L_2}=\sqrt{\E_{\bx}[f^2(\bx)]}$ denote inner product and norm in the 
$L_2$ space of square-integrable functions (with respect to the relevant 
distribution). Also, define the vector-valued function
\[
\bg(\bx) = \frac{\partial}{\partial \bw}p_{\bw}(\bx),
\]
and let $\bg(\bx) = (g_1(\bx),g_2(\bx),\ldots,g_n(\bx))$ for real-valued 
functions $g_1,\ldots,g_n$. Finally, let $\E_h$ denote an expectation with 
respect 
to $h$ chosen uniformly at random from $\H$. Let $|\H|=d$.

We begin by proving the result for the squared loss. To prove the bound, it is 
enough to show that $\E_h\norm{\nabla 
F_{h}(\bw)-\ba}^2\leq \frac{G^2}{|\H|}$ for any vector $\ba$ independent of 
$h$. In particular, let us choose 
$\ba=\E_{\bx}\left[p_{\bw}(\bx)\bg(\bx)\right]$. We thus bound the following:
\begin{align*}
\E_{h}\norm{\nabla F_{h}(\bw)-\E_{\bx}\left[p_{\bw}(\bx)\bg(\bx)\right]}^2&= 
\E_h\norm{\E_{\bx}\left[\left(p_{\bw}(\bx)-h(\bx)\right)\bg(\bx)\right]-\E_{\bx}\left[p_{\bw}(\bx)\bg(\bx)\right]}^2\\
&= \E_h\norm{\E_{\bx}\left[h(\bx)\bg(\bx)\right]}^2~=~ 
\E_h\sum_{j=1}^{n}\left(\E_{\bx}\left[h(\bx)g_j(\bx)\right]\right)^2\\
&= 
\E_h\sum_{j=1}^{n}\inner{h,g_j}_{L_2}^2~=~\sum_{j=1}^{n}\left(\frac{1}{|\H|}\sum_{i=1}^{d}\inner{h_i,g_j}_{L_2}^2\right)\\
&\stackrel{(*)}{\leq}
\sum_{j=1}^{n}\left(\frac{1}{|\H|}\norm{g_j}_{L_2}^2\right)~=~\frac{1}{|\H|}\sum_{j=1}^{n}\E_{\bx}[g_j^2(\bx)]\\
&= \frac{1}{|\H|}\E_{\bx}\left[\norm{\bg(\bx)}^2\right]~\leq~ 
\frac{G(\bw)^2}{|\H|},
\end{align*}
where $(*)$ follows from the functions in $\H$ being mutually orthogonal, 
and satisfying $\norm{h}_{L_2}\leq 1$ for all $h\in\H$. 

To handle a classification loss, note that by its definition and the fact that 
$h(\bx)\in \{-1,+1\}$,
\begin{align*}
\nabla F_h(\bw) &= 
\E_{\bx}\left[r'(h(\bx)p_{\bw}(\bx))\cdot \frac{\partial}{\partial 
\bw}p_{\bw}(\bx)\right]\\
&= 
\E_{\bx}\left[\left(\frac{r'(p_{\bw}(\bx))+r'(-p_{\bw}(\bx))}{2}+h(\bx)\cdot
\frac{r'(p_{\bw}(\bx))-r'(-p_{\bw}(\bx))}{2}\right)\cdot 
\frac{\partial}{\partial 
\bw}p_{\bw}(\bx)\right]\\
&=
\E_{\bx}\left[\frac{r'(p_{\bw}(\bx))+r'(-p_{\bw}(\bx))}{2}\cdot 
\frac{\partial}{\partial 
\bw}p_{\bw}(\bx)\right]+\E_{\bx}\left[h(\bx)\cdot
\left(\frac{r'(p_{\bw}(\bx))-r'(-p_{\bw}(\bx))}{2}\right)\cdot 
\frac{\partial}{\partial 
\bw}p_{\bw}(\bx)\right].
\end{align*}
Letting $\bg(\bx) = 
\left(\frac{r'(p_{\bw}(\bx))-r'(-p_{\bw}(\bx))}{2}\right)\cdot 
\frac{\partial}{\partial 
\bw}p_{\bw}(\bx)$ (which satisfies $\E_{\bx}[\norm{\bg(\bx)}^2]\leq G^2$ since 
$r$ is $1$-Lipschitz) and $\ba = 
\E_{\bx}\left[\frac{r'(p_{\bw}(\bx))+r'(-p_{\bw}(\bx))}{2}\cdot 
\frac{\partial}{\partial 
\bw}p_{\bw}(\bx)\right]$ (which does not depend on $h$), we get that
\[
\E_{h}\norm{\nabla F_h(\bw)-\ba}^2 = \E_{h}\norm{\E_{\bx}[h(\bx)\bg(\bx)]}^2.
\]
Proceeding now exactly in the same manner as the squared 
loss case, the result follows.
\end{proof}

\subsection{Proof of Theorem \ref{thm:endtoend}}\label{proof:thm:endtoend}
\begin{proof}
We first state and prove two auxiliary lemmas.

\begin{lemma}\label{lem:corrsmall1}
  Let $h_1,\ldots,h_n$ be real-valued functions on some Euclidean
  space, which belong to some weighted $L_2$ space. Suppose that
  $\norm{h_i}_{L_2}=1$ and
  $\max_{i\neq j}|\inner{h_i,h_j}_{L_2}|\leq c$. Then for any function
  $g$ on the same domain,
	\[
	\frac{1}{n}\sum_{i=1}^{n}\inner{h_i,g}_{L_2}^2 ~\leq~ \norm{g}_{L_2}^2\left(\frac{1}{n}+c\right).
	\]
\end{lemma}
\begin{proof}
  For simplicity, suppose first that the functions are defined over
  some finite domain equipped with a uniform distribution, so that
  $h_1,\ldots,h_n$ and $g$ can be thought of as finite-dimensional
  vectors, and the $L_2$ inner product and norm reduce to the standard
  inner product and norm in Euclidean space. Let $H=(h_1,\ldots,h_n)$
  denote the matrix whose $i$-th column is $h_i$. Then
	\[
	\sum_{i=1}^{n}\inner{h_i,g}^2 = g^\top \left(\sum_{i=1}^{n}h_ih_i^\top \right)g ~=~ g^\top HH^\top g \leq \norm{g}^2\norm{HH^\top} ~=~ \norm{g}^2\norm{H^\top H},
	\]
	where $\norm{\cdot}$ for a matrix denotes the spectral
        norm. Since $H^\top H$ is simply the $n\times n$ matrix with
        entry $\inner{h_i,h_j}$ in location $i,j$, we can write it as
        $I+M$, where $I$ is the $n\times n$ identity matrix, and $M$
        is a matrix with $0$ along the main diagonal, and entries of
        absolute value at most $c$ otherwise. Therefore, letting
        $\norm{\cdot}_{F}$ denote the Frobenius norm, we have that the
        above is at most
	\[
	\norm{g}^2\left(\norm{I}+\norm{M}\right)~\leq~\norm{g}^2\left(1+\norm{M}_F\right)~=~\norm{g}^2\left(1+cn\right),
	\]
	from which the result follows.
	
	Finally, it is easily verified that the same proof holds even
        when $h_1,\ldots,h_n,g$ are functions over some Euclidean
        space, belonging to some weighted $L_2$ space. In that case,
        $H$ is a bounded linear operator, and it holds that
        $\norm{H^* H}=\norm{H}^2=\norm{H^*}^2=\norm{HH^*}$ where $H^*$
        is the Hermitian adjoint of $H$ and the norm is the operator
        norm. The rest of the proof is essentially identical.
\end{proof}

\begin{lemma}\label{lem:corrsmall2}
  If $\bw,\bv$ are two unit vectors in $\reals^d$, and $\bx$ is a
  standard Gaussian random vector, then
	\[
	\left|\E_{\bx}\left[\text{sign}(\bw^\top \bx)\text{sign}(\bv^\top\bx)\right]\right|\leq |\inner{\bw,\bv}|
	\]
\end{lemma}
\begin{proof}
  Note that $\bw^\top\bx,\bv^\top\bx$ are jointly zero-mean Gaussian,
  each with variance $1$ and with covariance
  $\E[\bw^\top \bx \bx^\top \bv]=\bw^\top\bv$. Therefore,
	\begin{align*}
	\E_{\bx}\left[\text{sign}(\bw^\top \bx)\text{sign}(\bv^\top\bx)\right] &= \Pr(\bw^\top \bx\geq 0,\bv^\top \bx \geq 0)+\Pr(\bw^\top\bx \leq 0,\bv^\top \bx \leq 0)\\
	&~~~~-\Pr(\bw^\top\bx\geq 0,\bv^\top\bx\leq 0)-\Pr(\bw^\top\bx\leq 0,\bv^\top\bx\geq 0)\\
	&= 2\Pr(\bw^\top\bx\geq 0,\bv^\top\bx\geq 0)-2\Pr(\bw^\top\bx\geq 0,\bv^\top\bx\leq 0),
	\end{align*}
	which by a standard fact on the quadrant probability of
        bivariate normal distributions, equals
	\begin{align*}
	2&\left(\frac{1}{4}+\frac{\sin^{-1}(\bw^\top \bv)}{2\pi}\right)-2\left(\frac{\cos^{-1}(\bw^\top \bv)}{2\pi}\right)
	= \frac{1}{2}+\frac{1}{\pi}\left(\sin^{-1}(\bw^\top\bv)-\cos^{-1}(\bw^\top\bv)\right)\\
	&=\frac{1}{2}+\frac{1}{\pi}\left(2\sin^{-1}(\bw^\top\bv)-\frac{\pi}{2}\right)~=~ \frac{2\sin^{-1}(\bw^\top\bv)}{\pi}~.
	\end{align*}
	The absolute value of the above can be easily verified to be
        upper bounded by $|\bw^\top \bv|$, from which the result
        follows.
\end{proof}

With these lemmas at hand, we turn to prove our theorem. By a standard
measure concentration argument, we can find $d^k$ unit vectors
$\bu_1,\bu_2,\ldots,\bu_{d^k}$ in $\reals^d$ such that their inner
product is at most $O(\sqrt{k\log(d)/d})$ (where the $O(\cdot)$
notation is w.r.t. $d$). This induces $d^k$ functions
$h_{\bu_1},h_{\bu_2},\ldots,h_{\bu_{d^k}}$ where
$h_{\bu}(\bx_1,\ldots,\bx_k)=\prod_{l=1}^{k}\text{sign}(\bu^{\top}\bx_l)$. Their
$L_2$ norm (w.r.t. the distribution over $\bx_1^k=(\bx_1,\ldots,\bx_k)$) is $1$,
as they take values in $\{-1,+1\}$. Moreover, since
$\bx_1,\ldots,\bx_k$ are i.i.d. standard Gaussian, we have by
\lemref{lem:corrsmall2} that for any $i\neq j$,
\begin{align*}
\inner{h_{\bu_i},h_{\bu_j}}_{L_2}&=
\left|\E\left[\prod_{l=1}^{k}\text{sign}(\bu_i^\top \bx_l)\prod_{l=1}^{k}\text{sign}(\bu_j^\top \bx_l)\right]\right|\\
&=
\left|\prod_{l=1}^{k}\E\left[\text{sign}(\bu_i^\top \bx_l)\text{sign}(\bu_j^\top\bx_l)\right]\right| \\
&= \left|\E\left[\text{sign}(\bu_i^\top \bx_l)\text{sign}(\bu_j^\top\bx_l)\right]\right|^k\\
&\leq |\bu_i^\top \bu_j|^k~\leq~ O\left(\sqrt{\frac{k\log(d)}{d}}\right)^k. 
\end{align*}
Using this and \lemref{lem:corrsmall1}, we have that for any function $g$,
\[
  \frac{1}{d^k}\sum_{i=1}^{d^k}\inner{h_{\bu_i},g}_{L_2}^2 ~\leq~
  \norm{g}_{L_2}^2\cdot
  \left(\frac{1}{d^k}+O\left(\sqrt{\frac{k\log(d)}{d}}\right)^k\right)~\leq~
  \norm{g}_{L_2}^2\cdot O\left(\sqrt{\frac{k\log(d)}{d}}\right)^{k}~.
\]
Moreover, since this bound is derived based only on an inner product
condition between $\bu_1,\ldots,\bu_{d^k}$, the same result would hold
for $U\bu_1,\ldots,U\bu_{d^k}$ where $U$ is an arbitrary orthogonal
matrix, and in particular if we pick it uniformly at random:
\[
  \E_{U}\left[\frac{1}{d^k}\sum_{i=1}^{d^k}\inner{h_{U\bu_i},g}_{L_2}^2\right]
  ~\leq~ \norm{g}_{L_2}^2\cdot
  \left(\frac{1}{d^k}+O\left(\sqrt{\frac{k\log(d)}{d}}\right)\right).
\]
Now, note that for any fixed $i$, $U\bu_i$ is uniformly distributed on
the unit sphere, so the left hand side simply equals
$\E_{\bu}\left[\inner{h_{\bu},g}_{L_2}^2\right]$, and we get
\begin{equation}\label{eq:nearorth}
  \E_{\bu}\left[\inner{h_{\bu},g}_{L_2}^2\right]~\leq~\norm{g}^2\cdot  O\left(\sqrt{\frac{k\log(d)}{d}}\right)^{k}~.
\end{equation}

With this key inequality at hand, the proof is now very similar to the
one of \thmref{thm:sqdim}. Given the predictor $p_{\bw}(\bx_1^k)$, where
$\bw\in\reals^n$, define the vector-valued function
$ \bg(\bx_1^k) = \frac{\partial}{\partial \bw}p_{\bw}(\bx_1^k) $, and let
$\bg(\bx_1^k) = (g_1(\bx_1^k),g_2(\bx_1^k),\ldots,g_n(\bx_1^k))$ for real-valued
functions $g_1,\ldots,g_n$. To prove the bound, it is enough to upper
bound $\E_\bu\norm{\nabla F_{\bu}(\bw)-\ba}^2$ for any vector $\ba$
independent of $\bu$.  In particular, let us choose
$\ba=\E_{\bx_1^k}\left[p_{\bw}(\bx_1^k)\bg(\bx_1^k)\right]$. We thus bound the
following:
\begin{align*}
\E_{\bu}\norm{\nabla F_{\bu}(\bw)-\E_{\bx_1^k}\left[p_{\bw}(\bx_1^k)\bg(\bx_1^k)\right]}^2&= 
\E_\bu\norm{\E_{\bx_1^k}\left[\left(p_{\bw}(\bx_1^k)-h_{\bu}(\bx_1^k)\right)\bg(\bx_1^k)\right]-\E_{\bx_1^k}\left[p_{\bw}(\bx_1^k)\bg(\bx_1^k)\right]}^2\\
&= \E_\bu\norm{\E_{\bx_1^k}\left[h_{\bu}(\bx_1^k)\bg(\bx_1^k)\right]}^2~=~ 
\E_{\bu}\sum_{j=1}^{n}\left(\E_{\bx_1^k}\left[h_{\bu}(\bx_1^k)g_j(\bx_1^k)\right]\right)^2\\
&= 
\E_\bu\sum_{j=1}^{n}\inner{h_{\bu},g_j}_{L_2}^2~=~\sum_{j=1}^{n}\E_{\bu}\inner{h_{\bu},g_j}_{L_2}^2\\
&\stackrel{(*)}{\leq}
\sum_{j=1}^{n}\norm{g_j}^2\cdot  O\left(\sqrt{\frac{k\log(d)}{d}}\right)^{k}
~=~\sum_{j=1}^{n}\E_{\bx_1^k}[g_j^2(\bx_1^k)]\cdot  O\left(\sqrt{\frac{k\log(d)}{d}}\right)^{k}\\
&= \E_{\bx_1^k}\norm{\bg(\bx_1^k)}^2\cdot  O\left(\sqrt{\frac{k\log(d)}{d}}\right)^{k}~\leq~ 
G(\bw)^2\cdot  O\left(\sqrt{\frac{k\log(d)}{d}}\right)^{k},
\end{align*}
where $(*)$ follows from \eqref{eq:nearorth}. By definition of
$\Var(\H,F,\bw)$, the result follows. Generalization for the
classification loss is obtained in the exact same way to the one used in
the proof of \thmref{thm:sqdim}.

\end{proof}

\subsection{Proof of lemma \ref{lem:invW}}\label{proof:lem:invW}
\begin{proof}
\[
(U W)_{i,j} = \sum_t U_{i,t} W_{t,j} = W_{i,j} -2 W_{i-1,j} + W_{i-2,j} 
\]
If $i \ge j+1$ then $\frac{W_{i,j} + W_{i-2,j}}{2} = W_{i-1,j}$ and
therefore the above is clearly zero. If $i < j$ then all the values of
$W$ are zeros. Finally, if $i=j$ we obtain $1$. This concludes our
proof. 
\end{proof}

\subsection{Proof of lemma \ref{lem:hatUt}}\label{proof:lem:hatUt}
\begin{proof}
  Given a sample $\f$, and that our current weight matrix is
  $\hat{U}$, let $\mathbf{p}=W^{-1}\f$. The loss function on $\f$ is
  given by
\[
  \frac{1}{2}\|\hat{U} \mathbf{f} - \mathbf{p}\|^2
\] 
The gradient w.r.t. $\hat{U}$ is 
\[
\nabla = (\hat{U} \mathbf{f}
- \mathbf{p}) \mathbf{f}^\top = \hat{U} \mathbf{f} \mathbf{f}^\top -
\mathbf{p} \mathbf{f}^\top
\]
We obtain that the update rule is 
\[
\hat{U}_{t+1} = \hat{U}_t - \eta \left(\hat{U}_t \mathbf{f} \mathbf{f}^\top -
\mathbf{p} \mathbf{f}^\top\right) = \hat{U}_t (I-\eta \mathbf{f} \mathbf{f}^\top) + \eta \mathbf{p} \mathbf{f}^\top
\]
Taking expectation with respect to the random choice of the pair, using again $\mathbf{f} = W \mathbf{p}$, and
assuming $\E
\mathbf{p} \mathbf{p}^\top = \lambda I$, we
obtain that the stochastic gradient update rule satisfies
\[
\E \hat{U}_{t+1} = \hat{U}_t (I-\eta \lambda W W^\top) + \eta \lambda W^\top
\]
Continuing recursively, we obtain
\begin{align*}
\E \hat{U}_{t+1} &= \E \hat{U}_t (I-\eta \lambda W W^\top) + \eta \lambda W^\top \\
&= \left[\E \hat{U}_{t-1} (I-\eta \lambda W W^\top) + \eta \lambda W^\top\right]
  (I-\eta \lambda W W^\top) + \eta \lambda W^\top \\
&= \E \hat{U}_{t-1} (I-\eta \lambda W W^\top)^2 + \eta \lambda W^\top (I-\eta
  \lambda W W^\top) + \eta \lambda W^\top \\
&= \hat{U}_0  (I-\eta \lambda W W^\top)^t + \eta \lambda W^\top 
  \sum_{i=0}^t (I-\eta \lambda W W^\top)^i
\end{align*}
We assume that $\hat{U}_0 = \boldsymbol{0}$, and thus
\[
\E \hat{U}_{t+1} = \eta \lambda W^\top  \sum_{i=0}^t (I-\eta \lambda W W^\top)^i
\]
\end{proof}

\subsection{Proof of Theorem \ref{thm:cond_converge}}\label{proof:thm:cond_converge}
\begin{proof}
Fix some $i$, we have that
\[
(I - \eta \,\lambda W W^\top)^i = 
 (Q I Q^\top -
 \eta \,\lambda Q S V^\top V S Q^\top)^i = 
Q (I -
 \eta \,\lambda S S)^i Q^\top = Q \Lambda^i Q^\top
\]
where $\Lambda^i$ is diagonal with $\Lambda_{j,j}^i = (1-\eta \lambda
S_j^2)^i$. Therefore, by the properties of geometric series, $\E
\hat{U}_{t+1}$ converges if and only if $\eta\, \lambda \,S_{1,1}^2 < 1$.
When this condition holds we have that 
\begin{align*}
\hat{U}_\infty &= \eta\,\lambda \,W^\top \sum_{i=0}^\infty (I - \eta
                 \,\lambda W W^\top)^i \\
&= \eta\,\lambda \,W^\top (\eta\,\lambda\,W W^\top)^{-1} = W^\top (W
  W^\top)^{-1}  \\
&=V S Q^\top (Q S V^\top V S
Q^\top)^{-1} = V S Q^\top Q S^{-2} Q^\top = V S^{-1} Q^\top = U ~.
\end{align*}
Therefore,
\begin{align*}
\E \hat{U}_{t+1} - U &= \eta \,\lambda W^\top  \sum_{i=t+1}^\infty (I -
                       \eta \,\lambda W W^\top)^i \\
&= \eta \,\lambda V S Q^\top  \sum_{i=t+1}^\infty Q \Lambda^i Q^\top
  \\
&= V \left[\sum_{i=t+1}^\infty (\eta\,\lambda S) \Lambda^i \right]
  Q^\top ~.
\end{align*}
The matrix in the parentheses is diagonal, where the $j$'th diagonal
element is
\[
\eta \lambda S_{j,j} \cdot \frac{(1 -  \eta \lambda
  S_{j,j}^2)^{t+1}}{\eta \lambda S_{j,j}^{2}}
= S_{j,j}^{-1} (1 - \eta \lambda S_{j,j}^2)^{t+1}
\]
It follows that
\[
\| \E \hat{U}_{t+1} - U\| = \max_j S_{j,j}^{-1} (1 - \eta \lambda S_{j,j}^2)^{t+1}
\]
Using the inequality $(1-a)^{t+1} \ge 1 - (t+1)a$, that holds for
every $a \in (-1,1)$, we obtain that
\[
\| \E \hat{U}_{t+1} - U\| \ge S_{n,n}^{-1} (1 - (t+1) \eta
\lambda S_{n,n}^2) \ge  S_{n,n}^{-1} (1 - (t+1) \frac{S_{n,n}^2}{S_{1,1}^2}) ~. 
\]
It follows that whenever, 
\[
t+1 \le \frac{S_{1,1}^2}{2\,S_{n,n}^2} ~,
\]
we have that $\| \E \hat{U}_{t+1} - U\| \ge 0.5\,S_{n,n}^{-1}$. 
Finally, observe that 
\[
S_{n,n}^2 = \min_{x : \|x\|=1} x^\top W W^\top x \le e_1^\top W
W^\top  e_1 = 1 ~,
\]
hence $S_{n,n}^{-1} \ge 1$.

We now prove the second part of the theorem, regarding the condition
number of $W^\top W$, namely, $\frac{S_{1,1}^2}{2\,S_{n,n}^2} \ge
\Omega(n^{3.5})$. We note that the condition number of $W^\top W$ can
be calculated through its inverse matrix's, namely, $U^\top U$'s
condition number, as those are equal. 


It is easy to verify that $U e_n = e_n$. Therefore, the maximal
eigenvalue of $U^\top U$ is at least $1$.  To construct an upper bound
on the minimal eigenvalue of $U^\top U$, consider $v \in \reals^n$
s.t. for $i \le \sqrt{n}$ we have $v_i = -\frac{1}{2} (i/n)^2$ and for
$i > \sqrt{n}$ we have $v_i = \frac{1}{2n} - \frac{i/n}{\sqrt{n}}$. We
have that $|v_i| = O(1/n)$ for $i < \sqrt{n}+2$ and $v$ is linear for
$i \ge \sqrt{n}$. This implies that $(Uv)_i = 0$ for $i \ge \sqrt{n}+2$.
We also have
$(U v)_1 = v_1 = -0.5 / n^2$, $(U v)_2 = -2 v_1 + v_2 \approx -1/n^2$,
and for $i \in \{3,\ldots,\sqrt{n}\}$ we have
\begin{align*}
(U v)_i &= v_{i-2} -2 v_{i-1} + v_i = -3/n^2
\end{align*}
Finally, for $i = \sqrt{n}+1$ we have
\begin{align*}
(U v)_i &= v_{i-2} -2 v_{i-1} + v_i = v_i-v_{i-1} - (v_{i-1}-v_{i-2})
  \\
&= \frac{1}{n\sqrt{n}} - \frac{1}{2n^2}\left((i-1)^2-(i-2)^2\right) = \frac{1}{n\sqrt{n}} - \frac{1}{2n^2}\left(2i - 3\right) = \frac{1}{2n^2} ~.
\end{align*}
This yields
\[
\|Uv\|^2 \approx  \Theta(\sqrt{n}/n^4) = \Theta(n^{-3.5})
\]
In addition, 
\[
\|v\|^2 \ge \sum_{i=n/2}^n v_i^2 \ge \frac{n}{2}\, v_{n/2}^2 =
\frac{n}{2}\,  \left( \frac{1}{2n} - \frac{1}{2\sqrt{n}} \right)^2 =
\Omega(1) ~.
\]
Therefore,
\[
\frac{\|Uv\|^2}{\|v\|^2} = O(n^{-3.5}) ~,
\]
which implies that the minimal eigenvalue of $U^\top U$ is at most
$O(n^{-3.5})$.  All in all, we have shown that the condition number of
$U^\top U$ is $\Omega(n^{-3.5})$, implying the same over $W^\top W$.
\end{proof}

\subsection{Gradient Descent for Linear
  Regression} \label{sec:GDanalysis}

The loss function is
\[
\E_{x,y} \frac{1}{2} (x^\top w - y)^2
\]
The gradient at $w$ is
\[
\nabla = \E_{x,y} x (x^\top w - y) = \left(\E_x x x^\top \right) w - \E_{x,y} x
y := C w - z
\]
For the optimal solution we have $C w^* - z =  \boldsymbol{0}$, hence
$z = C w^*$. 
The update is therefore
\[
w_{t+1} = w_t - \eta (Cw_t - z) = (I-\eta C) w_t + \eta z = \ldots =
\sum_{i=0}^t (I- \eta C)^i \eta z =
\sum_{i=0}^t (I- \eta C)^i \eta C w^*
\]
Let $C = V D V^\top$ be the eigenvalue decomposition of $C$.
Observe that 
\[
w_{t+1} = V \sum_{i=0}^t \eta (I- \eta D)^i D V^\top w^* 
\]
Hence
\begin{align*}
\|w^* - w_{t+1} \| &= \| (V V^\top - V \sum_{i=0}^t \eta (I- \eta D)^i
D V^\top) w^* \| \\
&= \| (I - \sum_{i=0}^t \eta (I- \eta D)^i
D) V^\top w^* \| \\
&= \| (I - \eta D)^{t+1}\, V^\top w^* \| ~,
\end{align*}
where the last equality is because for every $j$ we have
\[
 (I - \sum_{i=0}^t \eta (I- \eta D)^i
D)_{j,j} = 1 - \sum_{i=0}^t \eta (1- \eta D_{j,j})^i
D_{j,j} = (1- \eta D_{j,j})^{t+1} ~.
\]
Denote $v^* = V^\top w^*$. We therefore obtain that
\[
\|w^* - w_{t+1} \|^2 = \sum_{j=1}^n \left((1- \eta D_{j,j})^{t+1} v^*_j \right)^2
\]
To obtain an upper bound, choose $\eta = 1/D_{1,1}$ and $t+1 \ge
\frac{D_{1,1}}{D_{n,n}} \log(\|w^*\|/\epsilon)$, and then, using $1-a \le
e^{-a}$, we get that
\[
\|w^* - w_{t+1} \|^2 \le \sum_{j=1}^n \left( \exp(- \eta D_{j,j}
  (t+1))  v^*_j \right)^2 
\le \frac{\epsilon^2}{\|w^*\|^2} \sum_{j=1}^n \left(  v^*_j \right)^2 = \epsilon^2 ~.
\] 

To obtain a lower bound, observe that if $v^*_1$ is non-negligible then $\eta$ must be at most
$1/D_{1,1}$ (otherwise the process will diverge). If in addition
$v^*_n$ is a constant (for simplicity, say $v^*_n = 1$), then 
\[
\|w^* - w_{t+1} \| \ge (1- \eta D_{n,n})^{t+1} \ge (1 -
D_{n,n}/D_{1,1})^{t+1} \ge 1- (t+1) D_{n,n}/D_{1,1} ~,
\]
where we used $(1-a)^{t+1} \ge 1- (t+1)a$ for $a \in [-1,1]$. It
follows that if $t+1 < 0.5 \, D_{n,n}/D_{1,1}$ then we have that  $\|w^* -
w_{t+1} \| \ge 0.5$. 

\subsection{The Covariance Matrix of Section \ref{sec:FtoKConv}} \label{sec:proofCondFtoKConv}

Denote by $C \in \reals^{3,3}$ the covariance matrix, and let
$\lambda_i(C)$ denote the $i$'th eigenvalue of $C$ (in a decreasing
order). The condition number of $C$ is
$\lambda_1(C)/\lambda_3(C)$. Below we derive lower and upper bounds on
the condition number under some assumptions.

\paragraph{Lower bound}
We assume that $\E_{f,t} f_t^2 = \Omega(n^2)$ (this would be the case
in many typical cases, for example when the allowed slopes are in
$\{\pm 1\}$), that $k$ (the number of pieces of the curve) is
constant, and that the changes of slope are in $[-1,1]$.
 
Now, take $v = [1,1,1]^\top$, then
\[
v^\top C v = \E_{f,t} (f_{t-1} + f_t + f_{t+1})^2 = \Omega(n^2)
\]
This yields $\lambda_1(C) \ge \Omega(n^2)$. 
Next, take $v = [1,-2,1]^\top$ we obtain
\[
v^\top C v = \E_{f,t} (f_{t-1} - f_t + f_{t+1})^2 = O(k/n)
\]
This yields $\lambda_3(C) \le O(k/n)$. All in all, we obtain that the
condition number of $C$ is $\Omega(n^3)$.  

\paragraph{Upper bound}

We consider distribution over $\boldsymbol{f}$ s.t. for every
$\boldsymbol{f}$, at exactly $k$ indices
$\boldsymbol{f}$ changes slope from $1$ to $-1$ or from $-1$ to
$1$ (with equal probability over the indices),
and at the rest of the indices we have that $\boldsymbol{f}$ is linear
with a slope of $1$ or $-1$. Denote $p = k/n$. 
Take any unit vector $v$, and denote $\bar{v} = v_1+v_2+v_3$. Then
\begin{align*}
v^\top C v &= \E_{f,t} (v_1 f_{t-1} + v_2 f_t + v_3 f_{t+1})^2 \\
&= \E_f \left[ 0.5 (1-p) \left((\bar{v}  f_t + v_3 - v_1)^2  + 
  (\bar{v}  f_t + v_1 - v_3)^2 \right)+ 0.5 p \left( (\bar{v} f_t + v_3 + v_1)^2  + 
  (\bar{v}  f_t - v_3 - v_1)^2  \right)\right] \\
&= \bar{v}^2 \E_f f_t^2 + (1-p) (v_1-v_3)^2 + p (v_1+v_3)^2 ~.
\end{align*}
Since $\E_{f,t} f_t^2 = \Theta(n^2)$, it is clear that $v^\top C v =
O(n^2)$. We next establish a lower bound of $\Omega(1/n)$. Observe that 
if $\bar{v}^2 \ge \Omega(1/n^3)$, we are done. If this is not the
case, then $-v_2 \approx v_1+v_3$. If $|v_1+v_3| > 0.1$, we
are done. Otherwise, $0.9 \le 1 - v_2^2 = v_1^2 + v_3^2$, so we must
have that $v_1$ and $v_3$ has opposite signs, and one of them is
large, hence $(v_1-v_3)^2$ is larger than a constant. This concludes
our proof. 

\subsection{Proof of Update Rule \ref{glmtron_update} Convergence in
  the Lipschitz Case}\label{proof:glmtron}

\begin{proof}
Let $B$ be an upper bound over $\norm{\bw^{(t)}}$ for all time step
$t$, and over $\norm{\bv^*}$. Moreover, assume $|u|\le c$ for some
constant $c$.
We denote the update rule by $\bw^{(t)}=\bw^{(t-1)}+\eta\tilde{\nabla}$, and bound from below:
\begin{align*}
\norm{\bw^{(t)}-\bv^*}^2-\norm{\bw^{(t+1)}-\bv^*}^2
&=2\inner{ \bw^{(t)}-\bv^*, \eta \tilde{\nabla}} - \eta^2\norm{\tilde{\nabla}}^2\\
&= 2\eta \E \left[ (u((\bw^{(t)}) ^\top x)-u((\bv^*) ^\top x)) (\bw^{(t)}-\bv^* )\bx \right] -
  \eta^2\norm{\tilde{\nabla}}^2\\
&\overset{(1)}{\geq}
  \frac{2\eta}{L} \E \left[ (u((\bw^{(t)})^\top x)-u((\bv^*)^\top x))^2 \right] - \eta^2\norm{\tilde{\nabla}}^2\\
&\overset{(2)}{\geq}
  \frac{2\eta}{L} \E \left[ (u((\bw^{(t)})^\top x)-u((\bv^*)^\top x))^2 \right] - \eta^2B^2c^2\\
\end{align*}
where $(1)$ follows from $L$-Lipschitzness and monotonicity of $u$,
$(2)$ follows from bounding $\norm{\bw}$, $\norm{\bx}$ and
$u$. Let the
expected error of the regressor parametrized by $\bw^t$ be denoted
$e^t$. We separate into cases:
\begin{itemize}
\item If $\norm{\bw^t-\bv^*}^2-\norm{\bw^{t+1}-\bv^*}^2 \geq \eta^2B^2c^2$, we can
  rewrite $\norm{\bw^t-\bv^*}^2-\eta^2B^2c^2 \geq \norm{\bw^{t+1}-\bv^*}^2$, and
  note that since $\norm{\bw^{t+1}-\bv^*}^2\geq 0$, and
  $\norm{\bw^{0}-\bv^*}^2\leq B^2$, there can be at most
  $\frac{B^2}{\eta^2 B^2c^2}=\frac{1}{\eta^2c^2}$ iterations where this condition
  will hold.
\item Otherwise, we get that $e^t\leq \eta B^2c^2L$.
\end{itemize}
Therefore, for a given $\epsilon$, by taking $T=\frac{c^4B^4L^2}{\epsilon^2}$, and setting
$\eta=\sqrt{\frac{1}{Tc^2}}$, we obtain that after $T$ iterations, the
first case is not holding anymore, and the second case implies
$e^T\leq\epsilon$.
\end{proof}

\section{Technical Lemmas}

\begin{lemma} \label{lem:boringparity} Any parity function over $d$
  variables is realizable by a network with one fully connected layer
  of width $\tilde{d}>\frac{3d}{2}$ with ReLU activations, and a fully
  connected output layer with linear activation and a single unit.
\end{lemma}
\begin{proof}
Let the weights entering each of the first $\frac{3d}{2}$ hidden units be set to
$\bv^*$, and the rest to 0. Further assume that for $i\in[d/2]$, the
biases of the first $3i+\set{1,2,3}$ units
are set to $-(2i-\half)$, $-2i$, $-(2i+\half)$ respectively, and that
their weights in the output layer are $1$, $-2$, and $1$. It is not
hard to see that the weighted sum of those triads of neurons is $\half$ if
$\inner{\bx,\bv^*}=2i$, and $0$ otherwise. Observe that there's such a
triad defined for each even number in the range $[d]$. Therefore, the
output of this net is $0$ if $y=-1$, and $\half$ otherwise. It is easy
to see that scaling of the output layer's weights by $4$, and
introduction of a $-1$
bias value to it, results in a perfect predictor.
\end{proof}

\section{Command Lines for Experiments}\label{cmdlns}

Our experiments are implemented in a simple manner in python. We use
the tensorflow package for optimization. The following command lines
can be used for viewing all optional arguments:

To run experiment \ref{experiment_parity}, use:
\begin{verbatim}
python ./parity.py --help
\end{verbatim}

To run experiment \ref{sec:krect}, use:
\begin{verbatim}
python ./tuple_rect.py --help
\end{verbatim}
For SNR estimations, use:
\begin{verbatim}
python ./tuple_rect_SNR.py --help
\end{verbatim}

To run experiment \ref{exp:dec_vs_e2e}, use:
\begin{verbatim}
python ./dec_vs_e2e_stocks.py --help
\end{verbatim}

For \secref{sec.Piece-wise Linear
  AutoEncoders}'s experiments, given below are the command lines used to
generate the plots. Additional arguments can be viewed by running:
\begin{verbatim}
python PWL_fail1.py --help
\end{verbatim}

To run experiment \ref{sec:FtoK}, use:
\begin{verbatim}
python PWL_fail1.py --FtoK
\end{verbatim}

To run experiment \ref{sec:FtoKConv}, use:
\begin{verbatim}
python PWL_fail1.py --FtoKConv
\end{verbatim}

To run experiment \ref{sec:FtoKConvCond}, use:
\begin{verbatim}
python PWL_fail1.py --FtoKConvCond --batch_size 10 
--number_of_iterations 500 --learning_rate 0.99
\end{verbatim}

To run experiment \ref{sec:FAutoEncoder}, use:
\begin{verbatim}
python PWL_fail1.py --FAutoEncoder
\end{verbatim}

To run \secref{sec.Non Continuous
  non-linearities}'s experiments, run:
\begin{verbatim}
python step_learn.py --help
\end{verbatim}

\end{document}

%% file: Parity_mytikz.tex
\begin{tikzpicture}

\begin{axis}[
xlabel={Training Iterations},
ylabel={Accuracy},
xmin=0, xmax=50000,
ymin=0.35, ymax=1.05,
ytick={0.5, 1},
axis on top,
width=\figurewidth,
height=\figureheight,
legend entries={{d=5},{d=10},{d=30}},
legend cell align={left},
legend style={at={(axis cs:48000,0.95)},anchor=north east}
]
\addplot [blue]
table {%
0 0.544000029563904
10 0.48199999332428
20 0.483999997377396
30 0.437999993562698
40 0.48199999332428
50 0.5
60 0.456000000238419
70 0.421999990940094
80 0.439999997615814
90 0.437999993562698
100 0.449999988079071
110 0.469999998807907
120 0.433999985456467
130 0.416000008583069
140 0.421999990940094
150 0.492000013589859
160 0.430000007152557
170 0.458000004291534
180 0.458000004291534
190 0.430000007152557
200 0.42399999499321
210 0.435999989509583
220 0.430000007152557
230 0.46000000834465
240 0.453999996185303
250 0.437999993562698
260 0.419999986886978
270 0.448000013828278
280 0.465999990701675
290 0.462000012397766
300 0.479999989271164
310 0.467999994754791
320 0.419999986886978
330 0.451999992132187
340 0.430000007152557
350 0.430000007152557
360 0.433999985456467
370 0.416000008583069
380 0.5
390 0.421999990940094
400 0.430000007152557
410 0.472000002861023
420 0.479999989271164
430 0.453999996185303
440 0.490000009536743
450 0.486000001430511
460 0.490000009536743
470 0.453999996185303
480 0.465999990701675
490 0.488000005483627
500 0.40200001001358
510 0.449999988079071
520 0.458000004291534
530 0.497999995946884
540 0.486000001430511
550 0.560000002384186
560 0.509999990463257
570 0.519999980926514
580 0.465999990701675
590 0.5
600 0.504000008106232
610 0.51800000667572
620 0.449999988079071
630 0.512000024318695
640 0.465999990701675
650 0.458000004291534
660 0.46000000834465
670 0.513999998569489
680 0.493999987840652
690 0.419999986886978
700 0.526000022888184
710 0.477999985218048
720 0.488000005483627
730 0.46000000834465
740 0.483999997377396
750 0.472000002861023
760 0.453999996185303
770 0.492000013589859
780 0.501999974250793
790 0.505999982357025
800 0.533999979496002
810 0.560000002384186
820 0.522000014781952
830 0.519999980926514
840 0.527999997138977
850 0.515999972820282
860 0.533999979496002
870 0.504000008106232
880 0.529999971389771
890 0.565999984741211
900 0.537999987602234
910 0.533999979496002
920 0.537999987602234
930 0.550000011920929
940 0.5
950 0.519999980926514
960 0.541999995708466
970 0.568000018596649
980 0.586000025272369
990 0.579999983310699
1000 0.59799998998642
1010 0.561999976634979
1020 0.582000017166138
1030 0.575999975204468
1040 0.586000025272369
1050 0.533999979496002
1060 0.565999984741211
1070 0.579999983310699
1080 0.540000021457672
1090 0.59799998998642
1100 0.583999991416931
1110 0.568000018596649
1120 0.551999986171722
1130 0.578000009059906
1140 0.589999973773956
1150 0.601999998092651
1160 0.603999972343445
1170 0.583999991416931
1180 0.589999973773956
1190 0.582000017166138
1200 0.579999983310699
1210 0.681999981403351
1220 0.649999976158142
1230 0.666000008583069
1240 0.667999982833862
1250 0.663999974727631
1260 0.674000024795532
1270 0.684000015258789
1280 0.653999984264374
1290 0.646000027656555
1300 0.666000008583069
1310 0.643999993801117
1320 0.610000014305115
1330 0.628000020980835
1340 0.649999976158142
1350 0.643999993801117
1360 0.722000002861023
1370 0.677999973297119
1380 0.685999989509583
1390 0.680000007152557
1400 0.642000019550323
1410 0.723999977111816
1420 0.727999985218048
1430 0.699999988079071
1440 0.712000012397766
1450 0.737999975681305
1460 0.709999978542328
1470 0.737999975681305
1480 0.723999977111816
1490 0.741999983787537
1500 0.730000019073486
1510 0.688000023365021
1520 0.717999994754791
1530 0.772000014781952
1540 0.75
1550 0.779999971389771
1560 0.804000020027161
1570 0.800000011920929
1580 0.763999998569489
1590 0.801999986171722
1600 0.782000005245209
1610 0.748000025749207
1620 0.717999994754791
1630 0.702000021934509
1640 0.745999991893768
1650 0.71399998664856
1660 0.759999990463257
1670 0.779999971389771
1680 0.777999997138977
1690 0.794000029563904
1700 0.776000022888184
1710 0.801999986171722
1720 0.787999987602234
1730 0.800000011920929
1740 0.776000022888184
1750 0.822000026702881
1760 0.842000007629395
1770 0.870000004768372
1780 0.84799998998642
1790 0.867999970912933
1800 0.856000006198883
1810 0.86599999666214
1820 0.867999970912933
1830 0.878000020980835
1840 0.843999981880188
1850 0.888000011444092
1860 0.875999987125397
1870 0.864000022411346
1880 0.921999990940094
1890 0.843999981880188
1900 0.870000004768372
1910 0.90200001001358
1920 0.913999974727631
1930 0.907999992370605
1940 0.913999974727631
1950 0.921999990940094
1960 0.90200001001358
1970 0.916000008583069
1980 0.910000026226044
1990 0.889999985694885
2000 0.921999990940094
2010 0.907999992370605
2020 0.920000016689301
2030 0.931999981403351
2040 0.90200001001358
2050 0.898000001907349
2060 0.912000000476837
2070 0.935999989509583
2080 0.916000008583069
2090 0.967999994754791
2100 0.966000020503998
2110 1
2120 1
2130 1
2140 1
2150 1
2160 1
2170 1
2180 1
2190 1
2200 1
2210 1
2220 1
2230 1
2240 1
2250 1
2260 1
2270 1
2280 1
2290 1
2300 1
2310 1
2320 1
2330 1
2340 1
2350 1
2360 1
2370 1
2380 1
2390 1
2400 1
2410 1
2420 1
2430 1
2440 1
2450 1
2460 1
2470 1
2480 1
2490 1
2500 1
2510 1
2520 1
2530 1
2540 1
2550 1
2560 1
2570 1
2580 1
2590 1
2600 1
2610 1
2620 1
2630 1
2640 1
2650 1
2660 1
2670 1
2680 1
2690 1
2700 1
2710 1
2720 1
2730 1
2740 1
2750 1
2760 1
2770 1
2780 1
2790 1
2800 1
2810 1
2820 1
2830 1
2840 1
2850 1
2860 1
2870 1
2880 1
2890 1
2900 1
2910 1
2920 1
2930 1
2940 1
2950 1
2960 1
2970 1
2980 1
2990 1
3000 1
3010 1
3020 1
3030 1
3040 1
3050 1
3060 1
3070 1
3080 1
3090 1
3100 1
3110 1
3120 1
3130 1
3140 1
3150 1
3160 1
3170 1
3180 1
3190 1
3200 1
3210 1
3220 1
3230 1
3240 1
3250 1
3260 1
3270 1
3280 1
3290 1
3300 1
3310 1
3320 1
3330 1
3340 1
3350 1
3360 1
3370 1
3380 1
3390 1
3400 1
3410 1
3420 1
3430 1
3440 1
3450 1
3460 1
3470 1
3480 1
3490 1
3500 1
3510 1
3520 1
3530 1
3540 1
3550 1
3560 1
3570 1
3580 1
3590 1
3600 1
3610 1
3620 1
3630 1
3640 1
3650 1
3660 1
3670 1
3680 1
3690 1
3700 1
3710 1
3720 1
3730 1
3740 1
3750 1
3760 1
3770 1
3780 1
3790 1
3800 1
3810 1
3820 1
3830 1
3840 1
3850 1
3860 1
3870 1
3880 1
3890 1
3900 1
3910 1
3920 1
3930 1
3940 1
3950 1
3960 1
3970 1
3980 1
3990 1
4000 1
4010 1
4020 1
4030 1
4040 1
4050 1
4060 1
4070 1
4080 1
4090 1
4100 1
4110 1
4120 1
4130 1
4140 1
4150 1
4160 1
4170 1
4180 1
4190 1
4200 1
4210 1
4220 1
4230 1
4240 1
4250 1
4260 1
4270 1
4280 1
4290 1
4300 1
4310 1
4320 1
4330 1
4340 1
4350 1
4360 1
4370 1
4380 1
4390 1
4400 1
4410 1
4420 1
4430 1
4440 1
4450 1
4460 1
4470 1
4480 1
4490 1
4500 1
4510 1
4520 1
4530 1
4540 1
4550 1
4560 1
4570 1
4580 1
4590 1
4600 1
4610 1
4620 1
4630 1
4640 1
4650 1
4660 1
4670 1
4680 1
4690 1
4700 1
4710 1
4720 1
4730 1
4740 1
4750 1
4760 1
4770 1
4780 1
4790 1
4800 1
4810 1
4820 1
4830 1
4840 1
4850 1
4860 1
4870 1
4880 1
4890 1
4900 1
4910 1
4920 1
4930 1
4940 1
4950 1
4960 1
4970 1
4980 1
4990 1
5000 1
50000 1
};
\addplot [green!50.0!black]
table {%
0 0.476000010967255
50 0.532000005245209
100 0.483999997377396
150 0.529999971389771
200 0.472000002861023
250 0.5
300 0.515999972820282
350 0.509999990463257
400 0.488000005483627
450 0.568000018596649
500 0.493999987840652
550 0.519999980926514
600 0.505999982357025
650 0.51800000667572
700 0.48199999332428
750 0.486000001430511
800 0.492000013589859
850 0.497999995946884
900 0.476000010967255
950 0.48199999332428
1000 0.488000005483627
1050 0.486000001430511
1100 0.541999995708466
1150 0.504000008106232
1200 0.476000010967255
1250 0.46399998664856
1300 0.488000005483627
1350 0.46399998664856
1400 0.515999972820282
1450 0.483999997377396
1500 0.456000000238419
1550 0.492000013589859
1600 0.515999972820282
1650 0.527999997138977
1700 0.522000014781952
1750 0.504000008106232
1800 0.522000014781952
1850 0.529999971389771
1900 0.5
1950 0.472000002861023
2000 0.493999987840652
2050 0.508000016212463
2100 0.469999998807907
2150 0.501999974250793
2200 0.529999971389771
2250 0.492000013589859
2300 0.532000005245209
2350 0.486000001430511
2400 0.513999998569489
2450 0.5
2500 0.51800000667572
2550 0.53600001335144
2600 0.46000000834465
2650 0.547999978065491
2700 0.505999982357025
2750 0.495999991893768
2800 0.523999989032745
2850 0.493999987840652
2900 0.522000014781952
2950 0.508000016212463
3000 0.522000014781952
3050 0.513999998569489
3100 0.472000002861023
3150 0.519999980926514
3200 0.513999998569489
3250 0.53600001335144
3300 0.488000005483627
3350 0.490000009536743
3400 0.523999989032745
3450 0.488000005483627
3500 0.512000024318695
3550 0.523999989032745
3600 0.509999990463257
3650 0.509999990463257
3700 0.479999989271164
3750 0.515999972820282
3800 0.488000005483627
3850 0.46399998664856
3900 0.469999998807907
3950 0.508000016212463
4000 0.508000016212463
4050 0.495999991893768
4100 0.467999994754791
4150 0.467999994754791
4200 0.453999996185303
4250 0.477999985218048
4300 0.462000012397766
4350 0.555999994277954
4400 0.501999974250793
4450 0.490000009536743
4500 0.526000022888184
4550 0.527999997138977
4600 0.479999989271164
4650 0.488000005483627
4700 0.48199999332428
4750 0.474000006914139
4800 0.486000001430511
4850 0.495999991893768
4900 0.544000029563904
4950 0.492000013589859
5000 0.490000009536743
5050 0.523999989032745
5100 0.492000013589859
5150 0.5
5200 0.501999974250793
5250 0.495999991893768
5300 0.490000009536743
5350 0.476000010967255
5400 0.51800000667572
5450 0.508000016212463
5500 0.532000005245209
5550 0.477999985218048
5600 0.509999990463257
5650 0.486000001430511
5700 0.504000008106232
5750 0.451999992132187
5800 0.533999979496002
5850 0.522000014781952
5900 0.515999972820282
5950 0.519999980926514
6000 0.515999972820282
6050 0.51800000667572
6100 0.458000004291534
6150 0.527999997138977
6200 0.527999997138977
6250 0.546000003814697
6300 0.554000020027161
6350 0.550000011920929
6400 0.546000003814697
6450 0.560000002384186
6500 0.569999992847443
6550 0.533999979496002
6600 0.606000006198883
6650 0.592000007629395
6700 0.572000026702881
6750 0.596000015735626
6800 0.587999999523163
6850 0.592000007629395
6900 0.61599999666214
6950 0.662000000476837
7000 0.629999995231628
7050 0.629999995231628
7100 0.629999995231628
7150 0.614000022411346
7200 0.703999996185303
7250 0.643999993801117
7300 0.663999974727631
7350 0.629999995231628
7400 0.653999984264374
7450 0.646000027656555
7500 0.666000008583069
7550 0.720000028610229
7600 0.737999975681305
7650 0.712000012397766
7700 0.702000021934509
7750 0.699999988079071
7800 0.660000026226044
7850 0.689999997615814
7900 0.716000020503998
7950 0.688000023365021
8000 0.765999972820282
8050 0.699999988079071
8100 0.720000028610229
8150 0.741999983787537
8200 0.763999998569489
8250 0.755999982357025
8300 0.722000002861023
8350 0.759999990463257
8400 0.745999991893768
8450 0.765999972820282
8500 0.748000025749207
8550 0.776000022888184
8600 0.741999983787537
8650 0.776000022888184
8700 0.787999987602234
8750 0.796000003814697
8800 0.801999986171722
8850 0.808000028133392
8900 0.811999976634979
8950 0.824000000953674
9000 0.851999998092651
9050 0.842000007629395
9100 0.853999972343445
9150 0.856000006198883
9200 0.875999987125397
9250 0.888000011444092
9300 0.88400000333786
9350 0.924000024795532
9400 0.927999973297119
9450 0.910000026226044
9500 0.953999996185303
9550 0.934000015258789
9600 0.939999997615814
9650 0.931999981403351
9700 0.927999973297119
9750 0.948000013828278
9800 0.941999971866608
9850 0.925999999046326
9900 0.931999981403351
9950 0.941999971866608
10000 0.934000015258789
10050 0.952000021934509
10100 0.931999981403351
10150 0.948000013828278
10200 0.934000015258789
10250 0.949999988079071
10300 0.938000023365021
10350 0.921999990940094
10400 0.941999971866608
10450 0.935999989509583
10500 0.927999973297119
10550 0.938000023365021
10600 0.952000021934509
10650 0.925999999046326
10700 0.944000005722046
10750 0.939999997615814
10800 0.948000013828278
10850 0.912000000476837
10900 0.948000013828278
10950 0.938000023365021
11000 0.944000005722046
11050 0.949999988079071
11100 0.934000015258789
11150 0.921999990940094
11200 0.949999988079071
11250 0.935999989509583
11300 0.953999996185303
11350 0.924000024795532
11400 0.948000013828278
11450 0.945999979972839
11500 0.938000023365021
11550 0.948000013828278
11600 0.938000023365021
11650 0.925999999046326
11700 0.938000023365021
11750 0.925999999046326
11800 0.938000023365021
11850 0.913999974727631
11900 0.934000015258789
11950 0.934000015258789
12000 0.941999971866608
12050 0.927999973297119
12100 0.920000016689301
12150 0.939999997615814
12200 0.935999989509583
12250 0.945999979972839
12300 0.944000005722046
12350 0.931999981403351
12400 0.945999979972839
12450 0.934000015258789
12500 0.916000008583069
12550 0.927999973297119
12600 0.941999971866608
12650 0.955999970436096
12700 0.944000005722046
12750 0.930000007152557
12800 0.921999990940094
12850 0.938000023365021
12900 0.939999997615814
12950 0.934000015258789
13000 0.924000024795532
13050 0.930000007152557
13100 0.945999979972839
13150 0.916000008583069
13200 0.939999997615814
13250 0.921999990940094
13300 0.934000015258789
13350 0.910000026226044
13400 0.924000024795532
13450 0.948000013828278
13500 0.955999970436096
13550 0.931999981403351
13600 0.931999981403351
13650 0.944000005722046
13700 0.941999971866608
13750 0.941999971866608
13800 0.939999997615814
13850 0.931999981403351
13900 0.935999989509583
13950 0.934000015258789
14000 0.955999970436096
14050 0.958000004291534
14100 0.938000023365021
14150 0.934000015258789
14200 0.934000015258789
14250 0.925999999046326
14300 0.944000005722046
14350 0.935999989509583
14400 0.944000005722046
14450 0.939999997615814
14500 0.931999981403351
14550 0.925999999046326
14600 0.938000023365021
14650 0.935999989509583
14700 0.938000023365021
14750 0.927999973297119
14800 0.927999973297119
14850 0.945999979972839
14900 0.907999992370605
14950 0.944000005722046
15000 0.938000023365021
15050 0.917999982833862
15100 0.934000015258789
15150 0.931999981403351
15200 0.939999997615814
15250 0.925999999046326
15300 0.945999979972839
15350 0.945999979972839
15400 0.931999981403351
15450 0.907999992370605
15500 0.944000005722046
15550 0.938000023365021
15600 0.949999988079071
15650 0.925999999046326
15700 0.958000004291534
15750 0.930000007152557
15800 0.948000013828278
15850 0.949999988079071
15900 0.962000012397766
15950 0.944000005722046
16000 0.958000004291534
16050 0.953999996185303
16100 0.970000028610229
16150 0.952000021934509
16200 0.96399998664856
16250 0.959999978542328
16300 0.962000012397766
16350 0.959999978542328
16400 0.958000004291534
16450 0.967999994754791
16500 0.980000019073486
16550 0.977999985218048
16600 0.966000020503998
16650 0.959999978542328
16700 0.972000002861023
16750 0.958000004291534
16800 0.96399998664856
16850 0.970000028610229
16900 0.970000028610229
16950 0.976000010967255
17000 0.966000020503998
17050 0.970000028610229
17100 0.977999985218048
17150 0.972000002861023
17200 0.977999985218048
17250 0.953999996185303
17300 0.972000002861023
17350 0.967999994754791
17400 0.970000028610229
17450 0.973999977111816
17500 0.976000010967255
17550 0.970000028610229
17600 0.976000010967255
17650 0.972000002861023
17700 0.959999978542328
17750 0.955999970436096
17800 0.962000012397766
17850 0.967999994754791
17900 0.96399998664856
17950 0.970000028610229
18000 0.984000027179718
18050 0.967999994754791
18100 0.962000012397766
18150 0.959999978542328
18200 0.970000028610229
18250 0.955999970436096
18300 0.96399998664856
18350 0.98199999332428
18400 0.973999977111816
18450 0.949999988079071
18500 0.972000002861023
18550 0.972000002861023
18600 0.973999977111816
18650 0.973999977111816
18700 0.973999977111816
18750 0.972000002861023
18800 0.984000027179718
18850 0.962000012397766
18900 0.967999994754791
18950 0.96399998664856
19000 0.966000020503998
19050 0.958000004291534
19100 0.980000019073486
19150 0.959999978542328
19200 0.980000019073486
19250 0.970000028610229
19300 0.984000027179718
19350 0.994000017642975
19400 0.972000002861023
19450 0.991999983787537
19500 1
19550 1
19600 1
19650 1
19700 1
19750 1
19800 1
19850 1
19900 1
19950 1
20000 1
20050 1
20100 1
20150 1
20200 1
20250 1
20300 1
20350 1
20400 1
20450 1
20500 1
20550 1
20600 1
20650 1
20700 1
20750 1
20800 1
20850 1
20900 1
20950 1
21000 1
21050 1
21100 1
21150 1
21200 1
21250 1
21300 1
21350 1
21400 1
21450 1
21500 1
21550 1
21600 1
21650 1
21700 1
21750 1
21800 1
21850 1
21900 1
21950 1
22000 1
22050 1
22100 1
22150 1
22200 1
22250 1
22300 1
22350 1
22400 1
22450 1
22500 1
22550 1
22600 1
22650 1
22700 1
22750 1
22800 1
22850 1
22900 1
22950 1
23000 1
23050 1
23100 1
23150 1
23200 1
23250 1
23300 1
23350 1
23400 1
23450 1
23500 1
23550 1
23600 1
23650 1
23700 1
23750 1
23800 1
23850 1
23900 1
23950 1
24000 1
24050 1
24100 1
24150 1
24200 1
24250 1
24300 1
24350 1
24400 1
24450 1
24500 1
24550 1
24600 1
24650 1
24700 1
24750 1
24800 1
24850 1
24900 1
24950 1
25000 1
50000 1
};
\addplot [red]
table {%
0 0.483999997377396
100 0.483999997377396
200 0.504000008106232
300 0.537999987602234
400 0.515999972820282
500 0.486000001430511
600 0.509999990463257
700 0.488000005483627
800 0.48199999332428
900 0.537999987602234
1000 0.501999974250793
1100 0.497999995946884
1200 0.501999974250793
1300 0.533999979496002
1400 0.493999987840652
1500 0.493999987840652
1600 0.513999998569489
1700 0.504000008106232
1800 0.490000009536743
1900 0.492000013589859
2000 0.512000024318695
2100 0.469999998807907
2200 0.537999987602234
2300 0.5
2400 0.497999995946884
2500 0.483999997377396
2600 0.509999990463257
2700 0.53600001335144
2800 0.513999998569489
2900 0.508000016212463
3000 0.501999974250793
3100 0.526000022888184
3200 0.451999992132187
3300 0.495999991893768
3400 0.512000024318695
3500 0.501999974250793
3600 0.472000002861023
3700 0.488000005483627
3800 0.46399998664856
3900 0.504000008106232
4000 0.492000013589859
4100 0.477999985218048
4200 0.488000005483627
4300 0.5
4400 0.490000009536743
4500 0.488000005483627
4600 0.462000012397766
4700 0.504000008106232
4800 0.497999995946884
4900 0.477999985218048
5000 0.492000013589859
5100 0.474000006914139
5200 0.505999982357025
5300 0.533999979496002
5400 0.476000010967255
5500 0.492000013589859
5600 0.508000016212463
5700 0.504000008106232
5800 0.483999997377396
5900 0.501999974250793
6000 0.479999989271164
6100 0.5
6200 0.483999997377396
6300 0.483999997377396
6400 0.529999971389771
6500 0.5
6600 0.474000006914139
6700 0.512000024318695
6800 0.5
6900 0.474000006914139
7000 0.479999989271164
7100 0.526000022888184
7200 0.5
7300 0.474000006914139
7400 0.5
7500 0.474000006914139
7600 0.5
7700 0.479999989271164
7800 0.493999987840652
7900 0.488000005483627
8000 0.479999989271164
8100 0.5
8200 0.509999990463257
8300 0.488000005483627
8400 0.5
8500 0.508000016212463
8600 0.477999985218048
8700 0.483999997377396
8800 0.513999998569489
8900 0.529999971389771
9000 0.48199999332428
9100 0.493999987840652
9200 0.495999991893768
9300 0.48199999332428
9400 0.472000002861023
9500 0.469999998807907
9600 0.479999989271164
9700 0.490000009536743
9800 0.486000001430511
9900 0.504000008106232
10000 0.501999974250793
10100 0.537999987602234
10200 0.546000003814697
10300 0.479999989271164
10400 0.497999995946884
10500 0.492000013589859
10600 0.488000005483627
10700 0.526000022888184
10800 0.486000001430511
10900 0.527999997138977
11000 0.5
11100 0.53600001335144
11200 0.48199999332428
11300 0.490000009536743
11400 0.476000010967255
11500 0.490000009536743
11600 0.509999990463257
11700 0.505999982357025
11800 0.544000029563904
11900 0.515999972820282
12000 0.504000008106232
12100 0.504000008106232
12200 0.490000009536743
12300 0.501999974250793
12400 0.492000013589859
12500 0.505999982357025
12600 0.509999990463257
12700 0.509999990463257
12800 0.501999974250793
12900 0.519999980926514
13000 0.53600001335144
13100 0.469999998807907
13200 0.53600001335144
13300 0.509999990463257
13400 0.527999997138977
13500 0.458000004291534
13600 0.497999995946884
13700 0.519999980926514
13800 0.501999974250793
13900 0.519999980926514
14000 0.462000012397766
14100 0.490000009536743
14200 0.488000005483627
14300 0.486000001430511
14400 0.477999985218048
14500 0.533999979496002
14600 0.512000024318695
14700 0.550000011920929
14800 0.529999971389771
14900 0.504000008106232
15000 0.48199999332428
15100 0.568000018596649
15200 0.44200000166893
15300 0.526000022888184
15400 0.515999972820282
15500 0.513999998569489
15600 0.483999997377396
15700 0.490000009536743
15800 0.495999991893768
15900 0.48199999332428
16000 0.472000002861023
16100 0.492000013589859
16200 0.522000014781952
16300 0.488000005483627
16400 0.486000001430511
16500 0.467999994754791
16600 0.509999990463257
16700 0.522000014781952
16800 0.527999997138977
16900 0.508000016212463
17000 0.527999997138977
17100 0.523999989032745
17200 0.497999995946884
17300 0.495999991893768
17400 0.519999980926514
17500 0.479999989271164
17600 0.523999989032745
17700 0.509999990463257
17800 0.53600001335144
17900 0.515999972820282
18000 0.5
18100 0.495999991893768
18200 0.533999979496002
18300 0.501999974250793
18400 0.488000005483627
18500 0.495999991893768
18600 0.508000016212463
18700 0.46399998664856
18800 0.469999998807907
18900 0.497999995946884
19000 0.529999971389771
19100 0.515999972820282
19200 0.497999995946884
19300 0.527999997138977
19400 0.537999987602234
19500 0.527999997138977
19600 0.504000008106232
19700 0.527999997138977
19800 0.483999997377396
19900 0.46399998664856
20000 0.486000001430511
20100 0.53600001335144
20200 0.551999986171722
20300 0.493999987840652
20400 0.467999994754791
20500 0.472000002861023
20600 0.505999982357025
20700 0.504000008106232
20800 0.515999972820282
20900 0.501999974250793
21000 0.504000008106232
21100 0.53600001335144
21200 0.497999995946884
21300 0.540000021457672
21400 0.509999990463257
21500 0.53600001335144
21600 0.523999989032745
21700 0.53600001335144
21800 0.495999991893768
21900 0.523999989032745
22000 0.513999998569489
22100 0.51800000667572
22200 0.508000016212463
22300 0.488000005483627
22400 0.509999990463257
22500 0.448000013828278
22600 0.5
22700 0.522000014781952
22800 0.474000006914139
22900 0.497999995946884
23000 0.476000010967255
23100 0.465999990701675
23200 0.493999987840652
23300 0.488000005483627
23400 0.509999990463257
23500 0.492000013589859
23600 0.501999974250793
23700 0.523999989032745
23800 0.490000009536743
23900 0.515999972820282
24000 0.523999989032745
24100 0.490000009536743
24200 0.474000006914139
24300 0.486000001430511
24400 0.529999971389771
24500 0.513999998569489
24600 0.51800000667572
24700 0.504000008106232
24800 0.483999997377396
24900 0.527999997138977
25000 0.5
25100 0.488000005483627
25200 0.522000014781952
25300 0.486000001430511
25400 0.501999974250793
25500 0.501999974250793
25600 0.501999974250793
25700 0.488000005483627
25800 0.509999990463257
25900 0.505999982357025
26000 0.501999974250793
26100 0.537999987602234
26200 0.519999980926514
26300 0.541999995708466
26400 0.519999980926514
26500 0.505999982357025
26600 0.477999985218048
26700 0.48199999332428
26800 0.515999972820282
26900 0.532000005245209
27000 0.508000016212463
27100 0.509999990463257
27200 0.446000009775162
27300 0.509999990463257
27400 0.490000009536743
27500 0.533999979496002
27600 0.486000001430511
27700 0.519999980926514
27800 0.519999980926514
27900 0.513999998569489
28000 0.490000009536743
28100 0.515999972820282
28200 0.51800000667572
28300 0.490000009536743
28400 0.527999997138977
28500 0.492000013589859
28600 0.479999989271164
28700 0.497999995946884
28800 0.504000008106232
28900 0.483999997377396
29000 0.490000009536743
29100 0.479999989271164
29200 0.448000013828278
29300 0.492000013589859
29400 0.51800000667572
29500 0.493999987840652
29600 0.504000008106232
29700 0.513999998569489
29800 0.544000029563904
29900 0.437999993562698
30000 0.488000005483627
30100 0.490000009536743
30200 0.547999978065491
30300 0.490000009536743
30400 0.495999991893768
30500 0.467999994754791
30600 0.477999985218048
30700 0.523999989032745
30800 0.48199999332428
30900 0.493999987840652
31000 0.488000005483627
31100 0.492000013589859
31200 0.515999972820282
31300 0.501999974250793
31400 0.5
31500 0.497999995946884
31600 0.497999995946884
31700 0.490000009536743
31800 0.508000016212463
31900 0.490000009536743
32000 0.508000016212463
32100 0.512000024318695
32200 0.488000005483627
32300 0.5
32400 0.448000013828278
32500 0.467999994754791
32600 0.479999989271164
32700 0.513999998569489
32800 0.5
32900 0.469999998807907
33000 0.469999998807907
33100 0.437999993562698
33200 0.512000024318695
33300 0.515999972820282
33400 0.497999995946884
33500 0.483999997377396
33600 0.474000006914139
33700 0.453999996185303
33800 0.483999997377396
33900 0.46399998664856
34000 0.490000009536743
34100 0.505999982357025
34200 0.508000016212463
34300 0.5
34400 0.479999989271164
34500 0.512000024318695
34600 0.51800000667572
34700 0.48199999332428
34800 0.469999998807907
34900 0.53600001335144
35000 0.515999972820282
35100 0.541999995708466
35200 0.509999990463257
35300 0.532000005245209
35400 0.5
35500 0.492000013589859
35600 0.467999994754791
35700 0.44200000166893
35800 0.504000008106232
35900 0.490000009536743
36000 0.512000024318695
36100 0.477999985218048
36200 0.504000008106232
36300 0.505999982357025
36400 0.523999989032745
36500 0.456000000238419
36600 0.529999971389771
36700 0.495999991893768
36800 0.490000009536743
36900 0.479999989271164
37000 0.515999972820282
37100 0.479999989271164
37200 0.537999987602234
37300 0.533999979496002
37400 0.479999989271164
37500 0.492000013589859
37600 0.522000014781952
37700 0.467999994754791
37800 0.469999998807907
37900 0.51800000667572
38000 0.519999980926514
38100 0.483999997377396
38200 0.537999987602234
38300 0.44200000166893
38400 0.513999998569489
38500 0.46399998664856
38600 0.515999972820282
38700 0.493999987840652
38800 0.501999974250793
38900 0.523999989032745
39000 0.486000001430511
39100 0.515999972820282
39200 0.477999985218048
39300 0.495999991893768
39400 0.505999982357025
39500 0.508000016212463
39600 0.5
39700 0.532000005245209
39800 0.492000013589859
39900 0.508000016212463
40000 0.474000006914139
40100 0.48199999332428
40200 0.476000010967255
40300 0.523999989032745
40400 0.504000008106232
40500 0.515999972820282
40600 0.497999995946884
40700 0.529999971389771
40800 0.529999971389771
40900 0.515999972820282
41000 0.483999997377396
41100 0.537999987602234
41200 0.495999991893768
41300 0.532000005245209
41400 0.515999972820282
41500 0.492000013589859
41600 0.483999997377396
41700 0.492000013589859
41800 0.5
41900 0.483999997377396
42000 0.467999994754791
42100 0.483999997377396
42200 0.527999997138977
42300 0.462000012397766
42400 0.523999989032745
42500 0.495999991893768
42600 0.51800000667572
42700 0.48199999332428
42800 0.51800000667572
42900 0.523999989032745
43000 0.492000013589859
43100 0.48199999332428
43200 0.490000009536743
43300 0.497999995946884
43400 0.501999974250793
43500 0.462000012397766
43600 0.483999997377396
43700 0.488000005483627
43800 0.505999982357025
43900 0.565999984741211
44000 0.501999974250793
44100 0.505999982357025
44200 0.501999974250793
44300 0.490000009536743
44400 0.428000003099442
44500 0.522000014781952
44600 0.492000013589859
44700 0.529999971389771
44800 0.504000008106232
44900 0.483999997377396
45000 0.446000009775162
45100 0.515999972820282
45200 0.522000014781952
45300 0.46000000834465
45400 0.519999980926514
45500 0.513999998569489
45600 0.493999987840652
45700 0.513999998569489
45800 0.493999987840652
45900 0.497999995946884
46000 0.483999997377396
46100 0.495999991893768
46200 0.523999989032745
46300 0.501999974250793
46400 0.508000016212463
46500 0.508000016212463
46600 0.483999997377396
46700 0.512000024318695
46800 0.479999989271164
46900 0.554000020027161
47000 0.495999991893768
47100 0.48199999332428
47200 0.474000006914139
47300 0.490000009536743
47400 0.477999985218048
47500 0.519999980926514
47600 0.532000005245209
47700 0.479999989271164
47800 0.509999990463257
47900 0.513999998569489
48000 0.472000002861023
48100 0.51800000667572
48200 0.51800000667572
48300 0.501999974250793
48400 0.501999974250793
48500 0.513999998569489
48600 0.519999980926514
48700 0.490000009536743
48800 0.469999998807907
48900 0.486000001430511
49000 0.508000016212463
49100 0.474000006914139
49200 0.504000008106232
49300 0.493999987840652
49400 0.483999997377396
49500 0.526000022888184
49600 0.486000001430511
49700 0.5
49800 0.5
49900 0.515999972820282
50000 0.501999974250793
};
\end{axis}

\end{tikzpicture}

%% file: RECT_performance_mytikz1.tex
\begin{tikzpicture}

\begin{axis}[
title={$k=1$},
xmin=0, xmax=20000,
ymin=0.3, ymax=1,
ytick={0.3,1}, 
xmajorticks=false,
axis on top,
width=\figurewidth,
height=\figureheight
]
\addplot [red, very thick]
table {%
0 0.50600004196167
100 0.923000037670135
200 0.942000031471252
300 0.943499982357025
400 0.942000031471252
500 0.940500020980835
600 0.942000031471252
700 0.937000036239624
800 0.936000108718872
900 0.938000023365021
1000 0.944500088691711
1100 0.942499995231628
1200 0.94650000333786
1300 0.945500016212463
1400 0.948500096797943
1500 0.942500054836273
1600 0.942000031471252
1700 0.944000124931335
1800 0.944500088691711
1900 0.946000039577484
2000 0.935999989509583
2100 0.942000091075897
2200 0.940000057220459
2300 0.943000078201294
2400 0.945500075817108
2500 0.949500143527985
2600 0.941499948501587
2700 0.94350004196167
2800 0.944500088691711
2900 0.948000073432922
3000 0.944500088691711
3100 0.947500050067902
3200 0.946000039577484
3300 0.940999984741211
3400 0.948000073432922
3500 0.941500067710876
3600 0.950500130653381
3700 0.94950008392334
3800 0.943000018596649
3900 0.945000052452087
4000 0.94350004196167
4100 0.945000052452087
4200 0.939500033855438
4300 0.947000086307526
4400 0.943999946117401
4500 0.9410001039505
4600 0.944500029087067
4700 0.941500067710876
4800 0.944000124931335
4900 0.944500088691711
5000 0.944000124931335
5100 0.946500062942505
5200 0.944999992847443
5300 0.943499982357025
5400 0.942499995231628
5500 0.944500088691711
5600 0.942000031471252
5700 0.945000052452087
5800 0.945000052452087
5900 0.943000018596649
6000 0.94050008058548
6100 0.941000044345856
6200 0.939500033855438
6300 0.94650012254715
6400 0.941000044345856
6500 0.945500016212463
6600 0.945999979972839
6700 0.943000018596649
6800 0.945000112056732
6900 0.9410001039505
7000 0.944000065326691
7100 0.94350004196167
7200 0.945500075817108
7300 0.944500029087067
7400 0.947000086307526
7500 0.945000052452087
7600 0.946000099182129
7700 0.944500088691711
7800 0.945500016212463
7900 0.945500075817108
8000 0.950000047683716
8100 0.94650012254715
8200 0.941999971866608
8300 0.949000120162964
8400 0.944500029087067
8500 0.945000052452087
8600 0.946000039577484
8700 0.942000031471252
8800 0.944000124931335
8900 0.945999979972839
9000 0.942000150680542
9100 0.941500067710876
9200 0.940999984741211
9300 0.945500075817108
9400 0.944000124931335
9500 0.946000099182129
9600 0.939999997615814
9700 0.946500062942505
9800 0.947500050067902
9900 0.949000060558319
10000 0.943000018596649
10100 0.948500037193298
10200 0.946000099182129
10300 0.951000094413757
10400 0.946500062942505
10500 0.945500075817108
10600 0.944000005722046
10700 0.945500075817108
10800 0.946000039577484
10900 0.944000065326691
11000 0.945999979972839
11100 0.943500101566315
11200 0.945500016212463
11300 0.944000065326691
11400 0.943000078201294
11500 0.940000057220459
11600 0.944500088691711
11700 0.944500088691711
11800 0.946000039577484
11900 0.950000047683716
12000 0.949000060558319
12100 0.951000034809113
12200 0.952000081539154
12300 0.948499977588654
12400 0.946000039577484
12500 0.953500032424927
12600 0.952000081539154
12700 0.950000047683716
12800 0.950999975204468
12900 0.949000120162964
13000 0.950500071048737
13100 0.949000120162964
13200 0.951500058174133
13300 0.946500062942505
13400 0.949500024318695
13500 0.951500058174133
13600 0.948500096797943
13700 0.945000052452087
13800 0.948500096797943
13900 0.943000078201294
14000 0.94950008392334
14100 0.94650012254715
14200 0.950500071048737
14300 0.948000073432922
14400 0.948499977588654
14500 0.943000018596649
14600 0.948000073432922
14700 0.943499982357025
14800 0.944500029087067
14900 0.948000013828278
15000 0.947499990463257
15100 0.94350004196167
15200 0.94650000333786
15300 0.948000073432922
15400 0.945000052452087
15500 0.950000047683716
15600 0.951000034809113
15700 0.950000047683716
15800 0.947500050067902
15900 0.946500062942505
16000 0.944500029087067
16100 0.94700014591217
16200 0.94950008392334
16300 0.948000073432922
16400 0.949500024318695
16500 0.948000073432922
16600 0.945000052452087
16700 0.94700014591217
16800 0.940000116825104
16900 0.945000052452087
17000 0.947000086307526
17100 0.946000039577484
17200 0.945000112056732
17300 0.949000000953674
17400 0.944500148296356
17500 0.946000099182129
17600 0.942500114440918
17700 0.946500062942505
17800 0.954000115394592
17900 0.946000039577484
18000 0.946000039577484
18100 0.948000073432922
18200 0.946000039577484
18300 0.949000060558319
18400 0.947000026702881
18500 0.944000065326691
18600 0.948000013828278
18700 0.946000039577484
18800 0.944000065326691
18900 0.948000073432922
19000 0.947000086307526
19100 0.94650000333786
19200 0.944999992847443
19300 0.948000073432922
19400 0.946000039577484
19500 0.945000052452087
19600 0.946000039577484
19700 0.946500062942505
19800 0.948000013828278
19900 0.946500062942505
20000 0.945000052452087
};
\addplot [blue, very thick]
table {%
0 0.494000017642975
100 0.902500033378601
200 0.939500033855438
300 0.941500067710876
400 0.938500046730042
500 0.939000010490417
600 0.944000005722046
700 0.947500050067902
800 0.94650000333786
900 0.944500029087067
1000 0.94950008392334
1100 0.949999988079071
1200 0.94050008058548
1300 0.94050008058548
1400 0.944500029087067
1500 0.9410001039505
1600 0.940999984741211
1700 0.944000005722046
1800 0.943000078201294
1900 0.945000052452087
2000 0.945000052452087
2100 0.943500101566315
2200 0.944000065326691
2300 0.944000124931335
2400 0.946500062942505
2500 0.945000052452087
};

\end{axis}

\end{tikzpicture}

%% file: RECT_performance_mytikz2.tex
\begin{tikzpicture}

\begin{axis}[
title={$k=2$},
xmin=0, xmax=20000,
ymin=0.3, ymax=1,
ytick={0.3,1}, 
xmajorticks=false,
axis on top,
width=\figurewidth,
height=\figureheight
]
\addplot [red, very thick]
table {%
0 0.503500044345856
100 0.503500044345856
200 0.518500030040741
300 0.880500078201294
400 0.907000064849854
500 0.901500046253204
600 0.893500089645386
700 0.907500028610229
800 0.910500049591064
900 0.908500075340271
1000 0.893000066280365
1100 0.902500092983246
1200 0.897500038146973
1300 0.912000060081482
1400 0.897000014781952
1500 0.918500006198883
1600 0.903999984264374
1700 0.915000081062317
1800 0.912500023841858
1900 0.876000046730042
2000 0.906500101089478
2100 0.90149998664856
2200 0.900500118732452
2300 0.904500007629395
2400 0.883000075817108
2500 0.906500101089478
2600 0.906500041484833
2700 0.907000064849854
2800 0.904500126838684
2900 0.907000064849854
3000 0.897000074386597
3100 0.906000077724457
3200 0.902500092983246
3300 0.907500028610229
3400 0.91100001335144
3500 0.904500067234039
3600 0.907000064849854
3700 0.909000039100647
3800 0.906000018119812
3900 0.910500109195709
4000 0.908500015735626
4100 0.91100001335144
4200 0.90500009059906
4300 0.90149998664856
4400 0.911500096321106
4500 0.901500105857849
4600 0.907999992370605
4700 0.903500020503998
4800 0.905500113964081
4900 0.905500054359436
5000 0.90200012922287
5100 0.893500089645386
5200 0.902999997138977
5300 0.894500017166138
5400 0.907000005245209
5500 0.90200001001358
5600 0.900000095367432
5700 0.898000061511993
5800 0.901000082492828
5900 0.904000043869019
6000 0.888999998569489
6100 0.907000005245209
6200 0.894500076770782
6300 0.900500059127808
6400 0.905500054359436
6500 0.904000103473663
6600 0.913500070571899
6700 0.89900004863739
6800 0.903499960899353
6900 0.907500028610229
7000 0.905499994754791
7100 0.907500028610229
7200 0.89900004863739
7300 0.898000001907349
7400 0.901000022888184
7500 0.907500028610229
7600 0.905500054359436
7700 0.906000077724457
7800 0.905499994754791
7900 0.909000039100647
8000 0.909000039100647
8100 0.907000064849854
8200 0.915000081062317
8300 0.913000047206879
8400 0.907500028610229
8500 0.90200001001358
8600 0.906000018119812
8700 0.913000047206879
8800 0.911000072956085
8900 0.906000077724457
9000 0.908500075340271
9100 0.908500075340271
9200 0.902500092983246
9300 0.908500015735626
9400 0.905500054359436
9500 0.907000064849854
9600 0.912500083446503
9700 0.908500134944916
9800 0.910500049591064
9900 0.911499977111816
10000 0.914500057697296
10100 0.913500070571899
10200 0.907999992370605
10300 0.906500041484833
10400 0.910000026226044
10500 0.90800005197525
10600 0.896499991416931
10700 0.900000095367432
10800 0.909000039100647
10900 0.901000082492828
11000 0.894500076770782
11100 0.909000039100647
11200 0.90800005197525
11300 0.902000069618225
11400 0.900499999523163
11500 0.90500009059906
11600 0.900000035762787
11700 0.906500101089478
11800 0.906000077724457
11900 0.906500101089478
12000 0.909000098705292
12100 0.902500033378601
12200 0.910000085830688
12300 0.907999992370605
12400 0.897500038146973
12500 0.907000064849854
12600 0.908500075340271
12700 0.908500075340271
12800 0.913000047206879
12900 0.908500015735626
13000 0.910500049591064
13100 0.909500062465668
13200 0.910500049591064
13300 0.909500062465668
13400 0.910500049591064
13500 0.905500054359436
13600 0.908500015735626
13700 0.907500028610229
13800 0.90800005197525
13900 0.906000018119812
14000 0.913000047206879
14100 0.913500070571899
14200 0.912500083446503
14300 0.898500025272369
14400 0.905499994754791
14500 0.904000043869019
14600 0.911500036716461
14700 0.905500054359436
14800 0.89900004863739
14900 0.904000043869019
15000 0.910000085830688
15100 0.914500057697296
15200 0.904000103473663
15300 0.904000043869019
15400 0.903500080108643
15500 0.905000030994415
15600 0.903500020503998
15700 0.910500109195709
15800 0.911500036716461
15900 0.910000026226044
16000 0.910000085830688
16100 0.908000111579895
16200 0.912999987602234
16300 0.905500054359436
16400 0.909000098705292
16500 0.909500062465668
16600 0.908500015735626
16700 0.905500054359436
16800 0.899500131607056
16900 0.895500123500824
17000 0.907000064849854
17100 0.906000077724457
17200 0.909000039100647
17300 0.911000072956085
17400 0.910000085830688
17500 0.909500122070312
17600 0.912500023841858
17700 0.914000034332275
17800 0.91650003194809
17900 0.914000034332275
18000 0.910500049591064
18100 0.911500096321106
18200 0.916000127792358
18300 0.909500002861023
18400 0.912999987602234
18500 0.914500057697296
18600 0.915500044822693
18700 0.914000034332275
18800 0.913000047206879
18900 0.917500019073486
19000 0.903000116348267
19100 0.901000022888184
19200 0.90500009059906
19300 0.907000005245209
19400 0.90500009059906
19500 0.910000085830688
19600 0.904500067234039
19700 0.911500096321106
19800 0.91100001335144
19900 0.912500083446503
20000 0.907000064849854
};
\addplot [blue, very thick]
table {%
0 0.496500015258789
100 0.685999989509583
200 0.898500084877014
300 0.906500101089478
400 0.900500059127808
500 0.911500036716461
600 0.904000103473663
700 0.901500105857849
800 0.898500084877014
900 0.908500075340271
1000 0.905499994754791
1100 0.904500067234039
1200 0.907000064849854
1300 0.906500041484833
1400 0.909000039100647
1500 0.914000034332275
1600 0.920000016689301
1700 0.910000085830688
1800 0.907500088214874
1900 0.911500036716461
2000 0.907000005245209
2100 0.906500041484833
2200 0.906499981880188
2300 0.910500049591064
2400 0.906500101089478
2500 0.910500109195709
};

\end{axis}

\end{tikzpicture}

%% file: RECT_performance_mytikz3.tex
\begin{tikzpicture}

\begin{axis}[
title={$k=3$},
xmin=0, xmax=20000,
ymin=0.3, ymax=1,
ytick={0.3,1}, 
xmajorticks=false,
axis on top,
width=\figurewidth,
height=\figureheight
]
\addplot [blue, very thick]
table {%
0 0.515500009059906
100 0.507000029087067
200 0.734500050544739
300 0.7535001039505
400 0.861500024795532
500 0.866000056266785
600 0.858500003814697
700 0.85450005531311
800 0.856000065803528
900 0.855500102043152
1000 0.852000057697296
1100 0.859000027179718
1200 0.870500028133392
1300 0.868499994277954
1400 0.86599999666214
1500 0.862500071525574
1600 0.861500024795532
1700 0.859000027179718
1800 0.857500076293945
1900 0.861000061035156
2000 0.862500071525574
2100 0.86350005865097
2200 0.862000107765198
2300 0.86350005865097
2400 0.857000112533569
2500 0.861500024795532
};
\addplot [red, very thick]
table {%
0 0.515500009059906
100 0.481000006198883
200 0.484500020742416
300 0.484500020742416
400 0.484499990940094
500 0.484499990940094
600 0.484500020742416
700 0.484500020742416
800 0.484500020742416
900 0.481500029563904
1000 0.476000010967255
1100 0.484499990940094
1200 0.490500032901764
1300 0.484500020742416
1400 0.484500020742416
1500 0.484500020742416
1600 0.484500020742416
1700 0.484500020742416
1800 0.484500020742416
1900 0.484500020742416
2000 0.484500050544739
2100 0.484500020742416
2200 0.515500009059906
2300 0.515500068664551
2400 0.515500009059906
2500 0.484500020742416
2600 0.484500020742416
2700 0.515500009059906
2800 0.484500020742416
2900 0.484500020742416
3000 0.484500020742416
3100 0.484500020742416
3200 0.484500020742416
3300 0.484500020742416
3400 0.484500020742416
3500 0.484500020742416
3600 0.484500020742416
3700 0.484500020742416
3800 0.484500020742416
3900 0.484500020742416
4000 0.484500020742416
4100 0.484499990940094
4200 0.484500020742416
4300 0.484500020742416
4400 0.484500020742416
4500 0.484500020742416
4600 0.484500020742416
4700 0.484499990940094
4800 0.484500020742416
4900 0.484500020742416
5000 0.484500020742416
5100 0.484499990940094
5200 0.484500020742416
5300 0.484500020742416
5400 0.484500020742416
5500 0.484500020742416
5600 0.484500020742416
5700 0.484500020742416
5800 0.484500020742416
5900 0.484500020742416
6000 0.484500020742416
6100 0.484499990940094
6200 0.484499990940094
6300 0.484500020742416
6400 0.484500020742416
6500 0.484500020742416
6600 0.484500020742416
6700 0.484500020742416
6800 0.484500020742416
6900 0.484500020742416
7000 0.484500020742416
7100 0.484500020742416
7200 0.484500020742416
7300 0.484499990940094
7400 0.484500020742416
7500 0.484500050544739
7600 0.484500020742416
7700 0.484500020742416
7800 0.484500020742416
7900 0.484500020742416
8000 0.484500020742416
8100 0.484500020742416
8200 0.484500020742416
8300 0.484499990940094
8400 0.484500020742416
8500 0.484499990940094
8600 0.484500020742416
8700 0.484500020742416
8800 0.484500020742416
8900 0.484500020742416
9000 0.484500020742416
9100 0.484499990940094
9200 0.484500020742416
9300 0.484500020742416
9400 0.484500020742416
9500 0.484500050544739
9600 0.484500020742416
9700 0.484500020742416
9800 0.484499990940094
9900 0.484500020742416
10000 0.484500020742416
10100 0.484500020742416
10200 0.484499990940094
10300 0.484500020742416
10400 0.484500020742416
10500 0.484500020742416
10600 0.484499990940094
10700 0.484500020742416
10800 0.484500020742416
10900 0.484500050544739
11000 0.484500020742416
11100 0.484500020742416
11200 0.484500020742416
11300 0.484500020742416
11400 0.484499990940094
11500 0.484500020742416
11600 0.484500020742416
11700 0.484500020742416
11800 0.484500020742416
11900 0.484500020742416
12000 0.484500020742416
12100 0.484500020742416
12200 0.484499990940094
12300 0.484500020742416
12400 0.484500050544739
12500 0.484500020742416
12600 0.484500020742416
12700 0.484500020742416
12800 0.484500020742416
12900 0.484500020742416
13000 0.484500020742416
13100 0.484499990940094
13200 0.484500020742416
13300 0.484500020742416
13400 0.484500020742416
13500 0.484500020742416
13600 0.484500020742416
13700 0.484500020742416
13800 0.484500020742416
13900 0.484500020742416
14000 0.484500020742416
14100 0.484500020742416
14200 0.484500020742416
14300 0.484500020742416
14400 0.484500020742416
14500 0.484500020742416
14600 0.484500020742416
14700 0.484500050544739
14800 0.484500020742416
14900 0.484500020742416
15000 0.484500020742416
15100 0.484499990940094
15200 0.484500020742416
15300 0.484500020742416
15400 0.484500020742416
15500 0.484500020742416
15600 0.484500020742416
15700 0.484500020742416
15800 0.484500020742416
15900 0.484500020742416
16000 0.484500020742416
16100 0.484499990940094
16200 0.484499990940094
16300 0.484500020742416
16400 0.484500020742416
16500 0.484500020742416
16600 0.484500020742416
16700 0.484500020742416
16800 0.484500020742416
16900 0.484500020742416
17000 0.484500020742416
17100 0.484499990940094
17200 0.484500050544739
17300 0.484500020742416
17400 0.484500020742416
17500 0.484500020742416
17600 0.484500020742416
17700 0.484500020742416
17800 0.484500050544739
17900 0.484500020742416
18000 0.484500020742416
18100 0.484500020742416
18200 0.484500020742416
18300 0.484500020742416
18400 0.484500020742416
18500 0.484500020742416
18600 0.484500020742416
18700 0.484499990940094
18800 0.484500020742416
18900 0.484500020742416
19000 0.484500020742416
19100 0.484499990940094
19200 0.484500020742416
19300 0.484500020742416
19400 0.484500020742416
19500 0.484500020742416
19600 0.484500020742416
19700 0.484500020742416
19800 0.484500020742416
19900 0.484499990940094
20000 0.484500020742416
};
\end{axis}

\end{tikzpicture}

%% file: RECT_performance_mytikz4.tex
\begin{tikzpicture}

\begin{axis}[
title={$k=4$},
xmin=0, xmax=20000,
ymin=0.3, ymax=1,
ytick={0.3,1}, 
xmajorticks=false,
axis on top,
width=\figurewidth,
height=\figureheight
]
\addplot [blue, very thick]
table {%
0 0.505499958992004
100 0.494500011205673
200 0.494500041007996
300 0.693499982357025
400 0.810500025749207
500 0.832500040531158
600 0.828000068664551
700 0.831000089645386
800 0.820000112056732
900 0.817000091075897
1000 0.82150000333786
1100 0.828499972820282
1200 0.835000157356262
1300 0.839000105857849
1400 0.846500039100647
1500 0.827500104904175
1600 0.839500069618225
1700 0.829000115394592
1800 0.840000033378601
1900 0.844000101089478
2000 0.840000033378601
2100 0.827500104904175
2200 0.836000084877014
2300 0.843500018119812
2400 0.84250009059906
2500 0.84550005197525
};
\addplot [red, very thick]
table {%
0 0.492000043392181
100 0.494500011205673
200 0.494500011205673
300 0.494500011205673
400 0.488000005483627
500 0.488000005483627
600 0.494500011205673
700 0.494500011205673
800 0.492000013589859
900 0.494500041007996
1000 0.500500023365021
1100 0.492500007152557
1200 0.494500011205673
1300 0.494500011205673
1400 0.499500006437302
1500 0.496500015258789
1600 0.494500011205673
1700 0.497499972581863
1800 0.486499965190887
1900 0.48649999499321
2000 0.503000020980835
2100 0.485000014305115
2200 0.49099999666214
2300 0.49850007891655
2400 0.501500010490417
2500 0.495000004768372
2600 0.5
2700 0.507999956607819
2800 0.503000020980835
2900 0.494500011205673
3000 0.497500061988831
3100 0.498999983072281
3200 0.520500004291534
3300 0.494000017642975
3400 0.49250003695488
3500 0.494500011205673
3600 0.49549999833107
3700 0.495999991893768
3800 0.503000020980835
3900 0.50900000333786
4000 0.505500018596649
4100 0.508000016212463
4200 0.512000024318695
4300 0.510499954223633
4400 0.49549999833107
4500 0.492999970912933
4600 0.492000043392181
4700 0.494000017642975
4800 0.494000017642975
4900 0.514499962329865
5000 0.510500013828278
5100 0.5
5200 0.494500011205673
5300 0.507499992847443
5400 0.510500073432922
5500 0.512000024318695
5600 0.510500013828278
5700 0.5
5800 0.507000029087067
5900 0.510500013828278
6000 0.503499984741211
6100 0.504999995231628
6200 0.499500006437302
6300 0.494500041007996
6400 0.505500018596649
6500 0.505500018596649
6600 0.505499958992004
6700 0.505499958992004
6800 0.505500018596649
6900 0.494500011205673
7000 0.494500011205673
7100 0.494500011205673
7200 0.494500011205673
7300 0.494500011205673
7400 0.494500011205673
7500 0.494500011205673
7600 0.494500041007996
7700 0.494500011205673
7800 0.505499958992004
7900 0.494500011205673
8000 0.494500011205673
8100 0.494500011205673
8200 0.494500041007996
8300 0.494500011205673
8400 0.494500011205673
8500 0.505500018596649
8600 0.505499958992004
8700 0.505500018596649
8800 0.505499958992004
8900 0.505499958992004
9000 0.505499958992004
9100 0.494500041007996
9200 0.494500011205673
9300 0.494500011205673
9400 0.494500011205673
9500 0.494500011205673
9600 0.494500011205673
9700 0.494500011205673
9800 0.494500011205673
9900 0.494500011205673
10000 0.494500011205673
10100 0.494500011205673
10200 0.494500011205673
10300 0.494500011205673
10400 0.494500011205673
10500 0.494500041007996
10600 0.494500011205673
10700 0.494500011205673
10800 0.505500018596649
10900 0.505500018596649
11000 0.505499958992004
11100 0.505500018596649
11200 0.505500018596649
11300 0.505499958992004
11400 0.494500011205673
11500 0.494500011205673
11600 0.494500011205673
11700 0.494500011205673
11800 0.494500011205673
11900 0.494500011205673
12000 0.494500041007996
12100 0.494500011205673
12200 0.494500011205673
12300 0.494500011205673
12400 0.494500011205673
12500 0.494500011205673
12600 0.494500011205673
12700 0.494500041007996
12800 0.494500041007996
12900 0.494500011205673
13000 0.494500011205673
13100 0.494500011205673
13200 0.505500018596649
13300 0.505499958992004
13400 0.505499958992004
13500 0.505499958992004
13600 0.494500011205673
13700 0.494500041007996
13800 0.494500011205673
13900 0.494500011205673
14000 0.494500011205673
14100 0.494500011205673
14200 0.494500041007996
14300 0.494500011205673
14400 0.494500011205673
14500 0.494500011205673
14600 0.494500041007996
14700 0.494500041007996
14800 0.494500011205673
14900 0.494500011205673
15000 0.494500041007996
15100 0.494500011205673
15200 0.494500011205673
15300 0.494500011205673
15400 0.505499958992004
15500 0.505499958992004
15600 0.505499958992004
15700 0.505500018596649
15800 0.505500018596649
15900 0.494500011205673
16000 0.494500011205673
16100 0.494500011205673
16200 0.494500011205673
16300 0.494500011205673
16400 0.494500011205673
16500 0.494500011205673
16600 0.494500011205673
16700 0.494500011205673
16800 0.494500011205673
16900 0.494500011205673
17000 0.494500011205673
17100 0.494500011205673
17200 0.494500041007996
17300 0.494500011205673
17400 0.494500011205673
17500 0.494500011205673
17600 0.494500011205673
17700 0.494500041007996
17800 0.494500011205673
17900 0.494500011205673
18000 0.494500011205673
18100 0.494500011205673
18200 0.494500041007996
18300 0.494500011205673
18400 0.494500041007996
18500 0.494500041007996
18600 0.494500011205673
18700 0.494500041007996
18800 0.494500011205673
18900 0.494500011205673
19000 0.494500011205673
19100 0.494500011205673
19200 0.494500011205673
19300 0.494500011205673
19400 0.494500011205673
19500 0.494500041007996
19600 0.494500011205673
19700 0.494500041007996
19800 0.494500011205673
19900 0.494500041007996
20000 0.494500011205673
};
\end{axis}

\end{tikzpicture}

%% file: RECT_SNR_mytikz.tex
\begin{tikzpicture}

\begin{axis}[
xmin=1, xmax=4,
xtick={1,2,3,4},
ytick={-7,-15},
ymin=-15, ymax=-7,
axis on top,
width=\figurewidth,
height=\figureheight
]
\addplot [thick, blue]
table {%
1 -7.0935594727961
2 -7.1261229165805915
3 -7.1394724581214852
4 -7.1426319443290813
};
\addplot [thick, red]
table {%
1 -7.1186027251371593
2 -13.415309454829117
3 -14.00475224985728
4 -14.545383954625903
};
\end{axis}

\end{tikzpicture}

%% file: FtoKTrainer_500_mytikz1.tex
\begin{tikzpicture}

\begin{axis}[
xmin=-4.95, xmax=103.95, ticks=none,
ymin=-21.45, ymax=10.45,
width=\figurewidth,
height=\figureheight,
tick align=outside,
x grid style={lightgray!92.026143790849673!black},
y grid style={lightgray!92.026143790849673!black}
]
\addplot [line width=1.64pt, blue]
table {%
0 0
1 0
2 0
3 0
4 0
5 0
6 0
7 0
8 0
9 0
10 0
11 0
12 0
13 0
14 0
15 0
16 0
17 0
18 0
19 0
20 0
21 0
22 0
23 0
24 0
25 0
26 0
27 0
28 0
29 0
30 0
31 0
32 0
33 0
34 0
35 0
36 0
37 0
38 1
39 2
40 3
41 4
42 5
43 6
44 7
45 8
46 9
47 8
48 7
49 6
50 5
51 4
52 3
53 2
54 1
55 0
56 -1
57 -2
58 -3
59 -4
60 -5
61 -6
62 -7
63 -8
64 -9
65 -10
66 -11
67 -12
68 -13
69 -14
70 -15
71 -16
72 -17
73 -18
74 -19
75 -20
76 -19
77 -18
78 -17
79 -16
80 -15
81 -14
82 -13
83 -12
84 -11
85 -10
86 -9
87 -8
88 -7
89 -6
90 -5
91 -4
92 -3
93 -2
94 -1
95 0
96 1
97 2
98 3
99 4
};
\addplot [thick, red]
table {%
0 0
1 0
2 0
3 -0.0069754864089191
4 -0.0219952799379826
5 -0.0401506870985031
6 -0.0522629581391811
7 -0.0748085081577301
8 -0.0863587409257889
9 -0.0861943662166595
10 -0.0864416360855103
11 -0.0774574130773544
12 -0.0725926011800766
13 -0.0438455753028393
14 -0.0143482964485884
15 -0.0109460353851318
16 -0.0128098092973232
17 -0.0226919911801815
18 -0.0365021526813507
19 -0.0601185858249664
20 -0.0830197036266327
21 -0.0882921516895294
22 -0.101158432662487
23 -0.122324973344803
24 -0.1628657579422
25 -0.231216788291931
26 -0.291593968868256
27 -0.32721346616745
28 -0.339224368333817
29 -0.340061604976654
30 -0.336611926555634
31 -0.309971988201141
32 -0.229629158973694
33 -0.0846951380372047
34 0.17443560063839
35 0.519306302070618
36 0.951756775379181
37 1.47076058387756
38 2.09253263473511
39 2.76302719116211
40 3.4736647605896
41 4.15829181671143
42 4.8021068572998
43 5.39760828018188
44 5.90670108795166
45 6.28059577941895
46 6.45575141906738
47 6.45186471939087
48 6.2802848815918
49 5.9281177520752
50 5.42011737823486
51 4.74338340759277
52 3.93264102935791
53 3.00787305831909
54 1.99264788627625
55 0.918838262557983
56 -0.202609300613403
57 -1.34975516796112
58 -2.51276159286499
59 -3.69437170028687
60 -4.88368797302246
61 -6.09204864501953
62 -7.27498245239258
63 -8.47039318084717
64 -9.63816452026367
65 -10.7905912399292
66 -11.8757572174072
67 -12.9187850952148
68 -13.9133968353271
69 -14.8343715667725
70 -15.6476764678955
71 -16.3420066833496
72 -16.9118499755859
73 -17.3281154632568
74 -17.5877780914307
75 -17.695837020874
76 -17.6044006347656
77 -17.3553619384766
78 -16.9007663726807
79 -16.2998352050781
80 -15.5691337585449
81 -14.7232418060303
82 -13.7809295654297
83 -12.7708177566528
84 -11.6999988555908
85 -10.5968732833862
86 -9.46065902709961
87 -8.30097389221191
88 -7.11483812332153
89 -5.95156002044678
90 -4.81209230422974
91 -3.70399188995361
92 -2.64303302764893
93 -1.61542129516602
94 -0.651542663574219
95 0.259548187255859
96 1.10955715179443
97 1.95956754684448
98 2.8095760345459
99 3.65958786010742
};
\end{axis}

\end{tikzpicture}

%% file: FtoKTrainer_10000_mytikz1.tex
\begin{tikzpicture}

\begin{axis}[
xmin=-4.95, xmax=103.95, ticks=none,
ymin=-21.45, ymax=10.45,
width=\figurewidth,
height=\figureheight,
tick align=outside,
x grid style={lightgray!92.026143790849673!black},
y grid style={lightgray!92.026143790849673!black}
]
\addplot [line width=1.64pt, blue]
table {%
0 0
1 0
2 0
3 0
4 0
5 0
6 0
7 0
8 0
9 0
10 0
11 0
12 0
13 0
14 0
15 0
16 0
17 0
18 0
19 0
20 0
21 0
22 0
23 0
24 0
25 0
26 0
27 0
28 0
29 0
30 0
31 0
32 0
33 0
34 0
35 0
36 0
37 0
38 1
39 2
40 3
41 4
42 5
43 6
44 7
45 8
46 9
47 8
48 7
49 6
50 5
51 4
52 3
53 2
54 1
55 0
56 -1
57 -2
58 -3
59 -4
60 -5
61 -6
62 -7
63 -8
64 -9
65 -10
66 -11
67 -12
68 -13
69 -14
70 -15
71 -16
72 -17
73 -18
74 -19
75 -20
76 -19
77 -18
78 -17
79 -16
80 -15
81 -14
82 -13
83 -12
84 -11
85 -10
86 -9
87 -8
88 -7
89 -6
90 -5
91 -4
92 -3
93 -2
94 -1
95 0
96 1
97 2
98 3
99 4
};
\addplot [thick, red]
table {%
0 0
1 0
2 0
3 0.00520124426111579
4 0.00550634087994695
5 0.0192065015435219
6 0.0232600625604391
7 0.0315447449684143
8 0.0309752374887466
9 0.0447717681527138
10 0.0579165816307068
11 0.0564558207988739
12 0.0767451524734497
13 0.0814575403928757
14 0.0796049386262894
15 0.0890195071697235
16 0.106299936771393
17 0.108772486448288
18 0.110724538564682
19 0.12733481824398
20 0.139832571148872
21 0.143483772873878
22 0.136486500501633
23 0.122313439846039
24 0.134967610239983
25 0.141626983880997
26 0.146947115659714
27 0.165399134159088
28 0.196008861064911
29 0.198405355215073
30 0.203785717487335
31 0.190550744533539
32 0.151898041367531
33 0.0894738435745239
34 0.0699301660060883
35 0.109896421432495
36 0.229292213916779
37 0.54422527551651
38 1.14978194236755
39 2.0017786026001
40 3.02620577812195
41 4.13820934295654
42 5.31281757354736
43 6.41174697875977
44 7.31764221191406
45 7.91841411590576
46 8.13288879394531
47 7.90876197814941
48 7.3004322052002
49 6.43122005462646
50 5.38422727584839
51 4.28268146514893
52 3.14477038383484
53 2.0464940071106
54 0.974892139434814
55 -0.0325908660888672
56 -1.04815757274628
57 -2.04015684127808
58 -3.01262497901917
59 -3.97721862792969
60 -4.96242094039917
61 -5.93002223968506
62 -6.89974498748779
63 -7.86384963989258
64 -8.8622989654541
65 -9.89096736907959
66 -10.9209308624268
67 -11.9629783630371
68 -13.0557994842529
69 -14.1713218688965
70 -15.2792873382568
71 -16.3642101287842
72 -17.3570404052734
73 -18.1755886077881
74 -18.7312068939209
75 -18.9611663818359
76 -18.7498016357422
77 -18.2044486999512
78 -17.4007949829102
79 -16.4213085174561
80 -15.3350439071655
81 -14.2225332260132
82 -13.121337890625
83 -12.0406455993652
84 -10.9991016387939
85 -9.99376678466797
86 -9.01261711120605
87 -8.04351997375488
88 -7.09054231643677
89 -6.10859966278076
90 -5.09753799438477
91 -4.12708568572998
92 -3.12814140319824
93 -2.12369322776794
94 -1.13551926612854
95 -0.113242387771606
96 0.889230012893677
97 1.89170432090759
98 2.89417600631714
99 3.89665126800537
};
\end{axis}

\end{tikzpicture}

%% file: FtoKTrainer_50000_mytikz1.tex
\begin{tikzpicture}

\begin{axis}[
xmin=-4.95, xmax=103.95, ticks=none,
ymin=-21.45, ymax=10.45,
width=\figurewidth,
height=\figureheight,
tick align=outside,
x grid style={lightgray!92.026143790849673!black},
y grid style={lightgray!92.026143790849673!black}
]
\addplot [line width=1.64pt, blue]
table {%
0 0
1 0
2 0
3 0
4 0
5 0
6 0
7 0
8 0
9 0
10 0
11 0
12 0
13 0
14 0
15 0
16 0
17 0
18 0
19 0
20 0
21 0
22 0
23 0
24 0
25 0
26 0
27 0
28 0
29 0
30 0
31 0
32 0
33 0
34 0
35 0
36 0
37 0
38 1
39 2
40 3
41 4
42 5
43 6
44 7
45 8
46 9
47 8
48 7
49 6
50 5
51 4
52 3
53 2
54 1
55 0
56 -1
57 -2
58 -3
59 -4
60 -5
61 -6
62 -7
63 -8
64 -9
65 -10
66 -11
67 -12
68 -13
69 -14
70 -15
71 -16
72 -17
73 -18
74 -19
75 -20
76 -19
77 -18
78 -17
79 -16
80 -15
81 -14
82 -13
83 -12
84 -11
85 -10
86 -9
87 -8
88 -7
89 -6
90 -5
91 -4
92 -3
93 -2
94 -1
95 0
96 1
97 2
98 3
99 4
};
\addplot [thick, red]
table {%
0 0
1 0
2 0
3 0.00430600577965379
4 0.00976956356316805
5 0.0172063056379557
6 0.0268700607120991
7 0.0339580029249191
8 0.0471889451146126
9 0.0497073158621788
10 0.066035307943821
11 0.0717687085270882
12 0.0788590237498283
13 0.0917624309659004
14 0.0952340364456177
15 0.107478715479374
16 0.108700633049011
17 0.11788772046566
18 0.125913947820663
19 0.134193241596222
20 0.136922404170036
21 0.151234537363052
22 0.15989638864994
23 0.158478021621704
24 0.171548053622246
25 0.17242431640625
26 0.179476737976074
27 0.193019673228264
28 0.199987515807152
29 0.197422489523888
30 0.221175163984299
31 0.236544206738472
32 0.244769126176834
33 0.240172788500786
34 0.202943712472916
35 0.198594808578491
36 0.269093096256256
37 0.570316672325134
38 1.25479817390442
39 2.21243715286255
40 3.22447276115417
41 4.28554391860962
42 5.36711025238037
43 6.48031139373779
44 7.54996967315674
45 8.41318321228027
46 8.75176429748535
47 8.41291522979736
48 7.58028221130371
49 6.52234697341919
50 5.41618347167969
51 4.36698913574219
52 3.37953400611877
53 2.41142988204956
54 1.4309309720993
55 0.471616744995117
56 -0.490329444408417
57 -1.47625684738159
58 -2.46502757072449
59 -3.46573448181152
60 -4.46440362930298
61 -5.45684623718262
62 -6.43478870391846
63 -7.43172121047974
64 -8.40059757232666
65 -9.37980079650879
66 -10.3536167144775
67 -11.3226156234741
68 -12.2821750640869
69 -13.2542629241943
70 -14.2553634643555
71 -15.3423976898193
72 -16.4264678955078
73 -17.4426040649414
74 -18.20166015625
75 -18.5165176391602
76 -18.1785888671875
77 -17.3687438964844
78 -16.3087120056152
79 -15.1670894622803
80 -14.0441417694092
81 -12.9793586730957
82 -11.9487075805664
83 -10.9373779296875
84 -9.95260047912598
85 -8.96228408813477
86 -7.93113136291504
87 -6.90121459960938
88 -5.85157346725464
89 -4.82369613647461
90 -3.80328512191772
91 -2.7867169380188
92 -1.75597023963928
93 -0.71326732635498
94 0.317503213882446
95 1.33998799324036
96 2.35794496536255
97 3.37590193748474
98 4.39385986328125
99 5.41181659698486
};
\end{axis}

\end{tikzpicture}

%% file: FtoKTrainer_500_mytikz4.tex
\begin{tikzpicture}

\begin{axis}[
xmin=-4.95, xmax=103.95, ticks=none, ticks=none,
ymin=-4.25771758258343, ymax=86.1075103610754,
width=\figurewidth,
height=\figureheight,
tick align=outside,
x grid style={lightgray!92.026143790849673!black},
y grid style={lightgray!92.026143790849673!black}
]
\addplot [line width=1.64pt, blue]
table {%
0 0
1 0
2 0
3 0
4 0
5 0
6 0
7 0
8 0
9 0
10 0
11 0
12 1
13 2
14 3
15 4
16 5
17 6
18 7
19 8
20 9
21 10
22 11
23 12
24 13
25 14
26 15
27 16
28 17
29 18
30 19
31 20
32 21
33 22
34 23
35 24
36 25
37 26
38 27
39 28
40 29
41 30
42 31
43 32
44 33
45 34
46 35
47 36
48 37
49 38
50 39
51 40
52 38
53 36
54 37
55 38
56 39
57 40
58 41
59 42
60 43
61 44
62 45
63 46
64 47
65 48
66 49
67 50
68 51
69 52
70 53
71 54
72 55
73 56
74 57
75 58
76 59
77 60
78 61
79 62
80 63
81 64
82 65
83 66
84 67
85 68
86 69
87 70
88 71
89 72
90 73
91 74
92 75
93 76
94 77
95 78
96 79
97 80
98 81
99 82
};
\addplot [thick, red]
table {%
0 0
1 0
2 0
3 -0.0815853178501129
4 -0.101195633411407
5 -0.150207221508026
6 -0.111796021461487
7 -0.0290400981903076
8 0.154956638813019
9 0.345345616340637
10 0.677585124969482
11 1.0058718919754
12 1.49811148643494
13 2.10073781013489
14 2.85144424438477
15 3.84757471084595
16 4.93437480926514
17 5.99332666397095
18 7.07460594177246
19 8.20680904388428
20 9.3115930557251
21 10.2597255706787
22 11.1447706222534
23 12.0367918014526
24 12.9658946990967
25 14.0248498916626
26 15.0161046981812
27 15.9698238372803
28 16.9130554199219
29 17.9530258178711
30 18.9313545227051
31 19.9533519744873
32 20.8889083862305
33 21.8603668212891
34 22.8620109558105
35 23.8618202209473
36 24.8352279663086
37 25.8042373657227
38 26.7587966918945
39 27.8032970428467
40 28.9158134460449
41 30.0143127441406
42 31.0972557067871
43 31.9618797302246
44 32.8354949951172
45 33.5057106018066
46 34.1768836975098
47 34.887996673584
48 35.509033203125
49 35.9617919921875
50 36.3289031982422
51 36.6375732421875
52 36.945987701416
53 37.3377838134766
54 37.7293815612793
55 38.2470626831055
56 38.7413673400879
57 39.2631607055664
58 39.8073310852051
59 40.3791656494141
60 41.038387298584
61 41.810489654541
62 42.6431007385254
63 43.6559371948242
64 44.6661567687988
65 45.7005386352539
66 46.5768737792969
67 47.5006484985352
68 48.422477722168
69 49.4595603942871
70 50.5114974975586
71 51.4654350280762
72 52.2953186035156
73 53.1215515136719
74 54.1341018676758
75 55.1742706298828
76 56.089225769043
77 57.0574417114258
78 57.9523735046387
79 58.8960876464844
80 59.8224449157715
81 60.7595062255859
82 61.6241836547852
83 62.4704895019531
84 63.4053421020508
85 64.3891448974609
86 65.3195953369141
87 66.2478256225586
88 67.1543350219727
89 68.0852279663086
90 68.9380111694336
91 69.8029022216797
92 70.6551513671875
93 71.3615036010742
94 72.1387634277344
95 72.9993286132812
96 73.8566360473633
97 74.7139587402344
98 75.5712738037109
99 76.4286041259766
};
\end{axis}

\end{tikzpicture}

%% file: FtoKTrainer_10000_mytikz4.tex
\begin{tikzpicture}

\begin{axis}[
xmin=-4.95, xmax=103.95, ticks=none,
ymin=-4.89082613587379, ymax=86.1376583874226,
width=\figurewidth,
height=\figureheight,
tick align=outside,
x grid style={lightgray!92.026143790849673!black},
y grid style={lightgray!92.026143790849673!black}
]
\addplot [line width=1.64pt, blue]
table {%
0 0
1 0
2 0
3 0
4 0
5 0
6 0
7 0
8 0
9 0
10 0
11 0
12 1
13 2
14 3
15 4
16 5
17 6
18 7
19 8
20 9
21 10
22 11
23 12
24 13
25 14
26 15
27 16
28 17
29 18
30 19
31 20
32 21
33 22
34 23
35 24
36 25
37 26
38 27
39 28
40 29
41 30
42 31
43 32
44 33
45 34
46 35
47 36
48 37
49 38
50 39
51 40
52 38
53 36
54 37
55 38
56 39
57 40
58 41
59 42
60 43
61 44
62 45
63 46
64 47
65 48
66 49
67 50
68 51
69 52
70 53
71 54
72 55
73 56
74 57
75 58
76 59
77 60
78 61
79 62
80 63
81 64
82 65
83 66
84 67
85 68
86 69
87 70
88 71
89 72
90 73
91 74
92 75
93 76
94 77
95 78
96 79
97 80
98 81
99 82
};
\addplot [thick, red]
table {%
0 0
1 0
2 0
3 -0.0191307291388512
4 -0.0822013169527054
5 -0.20269051194191
6 -0.248358264565468
7 -0.38944336771965
8 -0.512822866439819
9 -0.685454487800598
10 -0.753167748451233
11 -0.576001405715942
12 -0.104711174964905
13 0.71304190158844
14 1.61602294445038
15 2.52861785888672
16 3.42311525344849
17 4.33338022232056
18 5.26632833480835
19 6.06789350509644
20 6.85193824768066
21 7.62861204147339
22 8.55390644073486
23 9.49983692169189
24 10.3343477249146
25 11.2506675720215
26 12.1328687667847
27 12.9943857192993
28 13.9071407318115
29 14.8386335372925
30 15.7445030212402
31 16.6001586914062
32 17.5020771026611
33 18.4746265411377
34 19.3676624298096
35 20.2478218078613
36 21.2162380218506
37 22.2863845825195
38 23.19407081604
39 24.1570205688477
40 25.13014793396
41 26.13059425354
42 27.0794105529785
43 28.0471038818359
44 29.1483497619629
45 30.2816982269287
46 31.2546234130859
47 32.1805000305176
48 33.0764465332031
49 33.6021347045898
50 33.8714866638184
51 33.7884826660156
52 33.6140632629395
53 33.355525970459
54 33.2745170593262
55 33.5972557067871
56 34.3243827819824
57 35.2187423706055
58 36.1256256103516
59 37.0751724243164
60 38.1571884155273
61 39.2271270751953
62 40.3733062744141
63 41.4587516784668
64 42.5435180664062
65 43.609489440918
66 44.6279754638672
67 45.5382843017578
68 46.5032119750977
69 47.4876899719238
70 48.4935989379883
71 49.4886245727539
72 50.4043960571289
73 51.2829627990723
74 52.3207206726074
75 53.3559379577637
76 54.3449325561523
77 55.3118362426758
78 56.4788398742676
79 57.5311851501465
80 58.5301246643066
81 59.4877090454102
82 60.4907913208008
83 61.4837532043457
84 62.4576263427734
85 63.4129180908203
86 64.3620986938477
87 65.2597351074219
88 66.3444213867188
89 67.3293304443359
90 68.3332290649414
91 69.424072265625
92 70.4248504638672
93 71.3705673217773
94 72.2979125976562
95 73.2194366455078
96 74.1723327636719
97 75.125244140625
98 76.0781402587891
99 77.0310516357422
};
\end{axis}

\end{tikzpicture}

%% file: FtoKTrainer_50000_mytikz4.tex
\begin{tikzpicture}

\begin{axis}[
xmin=-4.95, xmax=103.95, ticks=none,
ymin=-4.72865976989269, ymax=86.1299361795187,
width=\figurewidth,
height=\figureheight,
tick align=outside,
x grid style={lightgray!92.026143790849673!black},
y grid style={lightgray!92.026143790849673!black}
]
\addplot [line width=1.64pt, blue]
table {%
0 0
1 0
2 0
3 0
4 0
5 0
6 0
7 0
8 0
9 0
10 0
11 0
12 1
13 2
14 3
15 4
16 5
17 6
18 7
19 8
20 9
21 10
22 11
23 12
24 13
25 14
26 15
27 16
28 17
29 18
30 19
31 20
32 21
33 22
34 23
35 24
36 25
37 26
38 27
39 28
40 29
41 30
42 31
43 32
44 33
45 34
46 35
47 36
48 37
49 38
50 39
51 40
52 38
53 36
54 37
55 38
56 39
57 40
58 41
59 42
60 43
61 44
62 45
63 46
64 47
65 48
66 49
67 50
68 51
69 52
70 53
71 54
72 55
73 56
74 57
75 58
76 59
77 60
78 61
79 62
80 63
81 64
82 65
83 66
84 67
85 68
86 69
87 70
88 71
89 72
90 73
91 74
92 75
93 76
94 77
95 78
96 79
97 80
98 81
99 82
};
\addplot [thick, red]
table {%
0 0
1 0
2 0
3 -0.0511850193142891
4 -0.106824479997158
5 -0.168257832527161
6 -0.252829313278198
7 -0.32282555103302
8 -0.3877894282341
9 -0.481128305196762
10 -0.598723590373993
11 -0.475399196147919
12 0.282951414585114
13 1.19805383682251
14 2.2030131816864
15 3.14505696296692
16 4.12963247299194
17 4.99853420257568
18 5.96522283554077
19 6.88568782806396
20 7.82530403137207
21 8.80324935913086
22 9.70383644104004
23 10.6972198486328
24 11.5800886154175
25 12.5417375564575
26 13.541672706604
27 14.4706583023071
28 15.3935146331787
29 16.3706665039062
30 17.3227195739746
31 18.2931499481201
32 19.234748840332
33 20.1779594421387
34 21.1173229217529
35 22.0931816101074
36 23.023983001709
37 23.9896087646484
38 24.954833984375
39 25.8439464569092
40 26.8438720703125
41 27.8363151550293
42 28.7280406951904
43 29.6434593200684
44 30.5288887023926
45 31.4583587646484
46 32.414665222168
47 33.3690948486328
48 34.427619934082
49 35.4089965820312
50 35.9663429260254
51 35.7945556640625
52 34.8926429748535
53 34.0403823852539
54 33.8197631835938
55 34.2997055053711
56 35.1955871582031
57 36.1902885437012
58 37.217414855957
59 38.1081733703613
60 39.0213394165039
61 39.8639907836914
62 40.7085647583008
63 41.5482330322266
64 42.4169425964355
65 43.1684341430664
66 43.9963455200195
67 44.8974418640137
68 45.7567253112793
69 46.6430740356445
70 47.3545112609863
71 48.1698532104492
72 49.0182456970215
73 49.7463417053223
74 50.4819869995117
75 51.3078193664551
76 52.1427536010742
77 52.9640274047852
78 53.7516593933105
79 54.4541778564453
80 55.1950302124023
81 56.0123710632324
82 56.8109550476074
83 57.6630401611328
84 58.4425964355469
85 59.217170715332
86 59.9669647216797
87 60.7727355957031
88 61.5378456115723
89 62.3914413452148
90 63.1636352539062
91 63.8679656982422
92 64.6380157470703
93 65.3859024047852
94 66.1170806884766
95 66.921501159668
96 67.7388458251953
97 68.5561904907227
98 69.3735427856445
99 70.1908950805664
};
\end{axis}

\end{tikzpicture}

%% file: FtoKConvTrainer_500_mytikz1.tex
\begin{tikzpicture}

\begin{axis}[
xmin=-4.95, xmax=103.95, ticks=none,
ymin=-21.45, ymax=10.45,
width=\figurewidth,
height=\figureheight,
tick align=outside,
x grid style={lightgray!92.026143790849673!black},
y grid style={lightgray!92.026143790849673!black}
]
\addplot [line width=1.64pt, blue]
table {%
0 0
1 0
2 0
3 0
4 0
5 0
6 0
7 0
8 0
9 0
10 0
11 0
12 0
13 0
14 0
15 0
16 0
17 0
18 0
19 0
20 0
21 0
22 0
23 0
24 0
25 0
26 0
27 0
28 0
29 0
30 0
31 0
32 0
33 0
34 0
35 0
36 0
37 0
38 1
39 2
40 3
41 4
42 5
43 6
44 7
45 8
46 9
47 8
48 7
49 6
50 5
51 4
52 3
53 2
54 1
55 0
56 -1
57 -2
58 -3
59 -4
60 -5
61 -6
62 -7
63 -8
64 -9
65 -10
66 -11
67 -12
68 -13
69 -14
70 -15
71 -16
72 -17
73 -18
74 -19
75 -20
76 -19
77 -18
78 -17
79 -16
80 -15
81 -14
82 -13
83 -12
84 -11
85 -10
86 -9
87 -8
88 -7
89 -6
90 -5
91 -4
92 -3
93 -2
94 -1
95 0
96 1
97 2
98 3
99 4
};
\addplot [thick, red]
table {%
0 0
1 0
2 0
3 0
4 0
5 0
6 0
7 0
8 0
9 0
10 0
11 0
12 0
13 0
14 0
15 0
16 0
17 0
18 0
19 0
20 0
21 0
22 0
23 0
24 0
25 0
26 0
27 0
28 0
29 0
30 0
31 0
32 0
33 0
34 0
35 0
36 0
37 0
38 0.0323295891284943
39 0.0820146352052689
40 0.148726463317871
41 0.232136443257332
42 0.331915855407715
43 0.447736144065857
44 0.579268574714661
45 0.726184487342834
46 0.888155281543732
47 1.00019299983978
48 1.09191739559174
49 1.16365706920624
50 1.21574068069458
51 1.24849689006805
52 1.26225423812866
53 1.25734162330627
54 1.2340874671936
55 1.19282054901123
56 1.13386964797974
57 1.05756306648254
58 0.964229941368103
59 0.854198694229126
60 0.72779780626297
61 0.585355937480927
62 0.427202224731445
63 0.253664910793304
64 0.065072774887085
65 -0.138245582580566
66 -0.355961441993713
67 -0.58774596452713
68 -0.833270847797394
69 -1.09220743179321
70 -1.36422669887543
71 -1.64900028705597
72 -1.94619917869568
73 -2.25549530982971
74 -2.5765597820282
75 -2.9090633392334
76 -3.18801879882812
77 -3.44304633140564
78 -3.67447304725647
79 -3.88262844085693
80 -4.06784152984619
81 -4.23044109344482
82 -4.37075471878052
83 -4.48911142349243
84 -4.58584070205688
85 -4.6612696647644
86 -4.71572923660278
87 -4.74954557418823
88 -4.76304912567139
89 -4.75656795501709
90 -4.73043060302734
91 -4.68496608734131
92 -4.62050342559814
93 -4.53736925125122
94 -4.43589401245117
95 -4.31640625
96 -4.17923450469971
97 -4.02470731735229
98 -3.85315322875977
99 -3.66490077972412
};
\end{axis}

\end{tikzpicture}

%% file: FtoKConvTrainer_10000_mytikz1.tex
\begin{tikzpicture}

\begin{axis}[
xmin=-4.95, xmax=103.95, ticks=none,
ymin=-21.45, ymax=10.45,
width=\figurewidth,
height=\figureheight,
tick align=outside,
x grid style={lightgray!92.026143790849673!black},
y grid style={lightgray!92.026143790849673!black}
]
\addplot [line width=1.64pt, blue]
table {%
0 0
1 0
2 0
3 0
4 0
5 0
6 0
7 0
8 0
9 0
10 0
11 0
12 0
13 0
14 0
15 0
16 0
17 0
18 0
19 0
20 0
21 0
22 0
23 0
24 0
25 0
26 0
27 0
28 0
29 0
30 0
31 0
32 0
33 0
34 0
35 0
36 0
37 0
38 1
39 2
40 3
41 4
42 5
43 6
44 7
45 8
46 9
47 8
48 7
49 6
50 5
51 4
52 3
53 2
54 1
55 0
56 -1
57 -2
58 -3
59 -4
60 -5
61 -6
62 -7
63 -8
64 -9
65 -10
66 -11
67 -12
68 -13
69 -14
70 -15
71 -16
72 -17
73 -18
74 -19
75 -20
76 -19
77 -18
78 -17
79 -16
80 -15
81 -14
82 -13
83 -12
84 -11
85 -10
86 -9
87 -8
88 -7
89 -6
90 -5
91 -4
92 -3
93 -2
94 -1
95 0
96 1
97 2
98 3
99 4
};
\addplot [thick, red]
table {%
0 0
1 0
2 0
3 0
4 0
5 0
6 0
7 0
8 0
9 0
10 0
11 0
12 0
13 0
14 0
15 0
16 0
17 0
18 0
19 0
20 0
21 0
22 0
23 0
24 0
25 0
26 0
27 0
28 0
29 0
30 0
31 0
32 0
33 0
34 0
35 0
36 0
37 0
38 0.848757266998291
39 1.70052218437195
40 2.55505585670471
41 3.41211915016174
42 4.27147388458252
43 5.13288021087646
44 5.99609899520874
45 6.8608922958374
46 7.72702121734619
47 6.89673137664795
48 6.06128549575806
49 5.22091770172119
50 4.37587070465088
51 3.52638626098633
52 2.67269611358643
53 1.81504344940186
54 0.953664779663086
55 0.0888023376464844
56 -0.779304504394531
57 -1.65042304992676
58 -2.52431106567383
59 -3.40072250366211
60 -4.27942657470703
61 -5.16019058227539
62 -6.04276084899902
63 -6.92690086364746
64 -7.81238555908203
65 -8.69896507263184
66 -9.58639717102051
67 -10.4744510650635
68 -11.3628845214844
69 -12.2514495849609
70 -13.1399307250977
71 -14.028076171875
72 -14.9156360626221
73 -15.8023834228516
74 -16.6880798339844
75 -17.5724868774414
76 -16.7578392028809
77 -15.935417175293
78 -15.1054344177246
79 -14.2681617736816
80 -13.4238090515137
81 -12.572624206543
82 -11.71484375
83 -10.850715637207
84 -9.98049926757812
85 -9.1043701171875
86 -8.22261810302734
87 -7.33548736572266
88 -6.44316864013672
89 -5.54595947265625
90 -4.64405822753906
91 -3.73772430419922
92 -2.82715606689453
93 -1.91265106201172
94 -0.994422912597656
95 -0.0727081298828125
96 0.852264404296875
97 1.78023529052734
98 2.71097564697266
99 3.64425659179688
};
\end{axis}

\end{tikzpicture}

%% file: FtoKConvTrainer_50000_mytikz1.tex
\begin{tikzpicture}

\begin{axis}[
xmin=-4.95, xmax=103.95, ticks=none,
ymin=-21.4518639564514, ymax=10.4531399726868,
width=\figurewidth,
height=\figureheight,
tick align=outside,
x grid style={lightgray!92.026143790849673!black},
y grid style={lightgray!92.026143790849673!black}
]
\addplot [line width=1.64pt, blue]
table {%
0 0
1 0
2 0
3 0
4 0
5 0
6 0
7 0
8 0
9 0
10 0
11 0
12 0
13 0
14 0
15 0
16 0
17 0
18 0
19 0
20 0
21 0
22 0
23 0
24 0
25 0
26 0
27 0
28 0
29 0
30 0
31 0
32 0
33 0
34 0
35 0
36 0
37 0
38 1
39 2
40 3
41 4
42 5
43 6
44 7
45 8
46 9
47 8
48 7
49 6
50 5
51 4
52 3
53 2
54 1
55 0
56 -1
57 -2
58 -3
59 -4
60 -5
61 -6
62 -7
63 -8
64 -9
65 -10
66 -11
67 -12
68 -13
69 -14
70 -15
71 -16
72 -17
73 -18
74 -19
75 -20
76 -19
77 -18
78 -17
79 -16
80 -15
81 -14
82 -13
83 -12
84 -11
85 -10
86 -9
87 -8
88 -7
89 -6
90 -5
91 -4
92 -3
93 -2
94 -1
95 0
96 1
97 2
98 3
99 4
};
\addplot [thick, red]
table {%
0 0
1 0
2 0
3 0
4 0
5 0
6 0
7 0
8 0
9 0
10 0
11 0
12 0
13 0
14 0
15 0
16 0
17 0
18 0
19 0
20 0
21 0
22 0
23 0
24 0
25 0
26 0
27 0
28 0
29 0
30 0
31 0
32 0
33 0
34 0
35 0
36 0
37 0
38 0.999992251396179
39 2.00005507469177
40 3.00019359588623
41 4.000412940979
42 5.00071954727173
43 6.00111675262451
44 7.00161123275757
45 8.00220775604248
46 9.0029125213623
47 8.00374317169189
48 7.00455093383789
49 6.00533008575439
50 5.00607442855835
51 4.00678157806396
52 3.00744438171387
53 2.008056640625
54 1.00861358642578
55 0.00911140441894531
56 -0.990453720092773
57 -1.9900951385498
58 -2.98980903625488
59 -3.98960113525391
60 -4.98948287963867
61 -5.98945999145508
62 -6.98953056335449
63 -7.98970794677734
64 -8.98998069763184
65 -9.9903736114502
66 -10.9908828735352
67 -11.9915142059326
68 -12.9922695159912
69 -13.9931640625
70 -14.9942054748535
71 -15.9953804016113
72 -16.9967041015625
73 -17.9981918334961
74 -18.999828338623
75 -20.001636505127
76 -19.0036277770996
77 -18.0056495666504
78 -17.0077133178711
79 -16.0097961425781
80 -15.0118980407715
81 -14.0140228271484
82 -13.0161399841309
83 -12.0182647705078
84 -11.0204086303711
85 -10.0225219726562
86 -9.02463531494141
87 -8.02671813964844
88 -7.02877044677734
89 -6.03080749511719
90 -5.03280639648438
91 -4.03476715087891
92 -3.03668975830078
93 -2.03854370117188
94 -1.04035949707031
95 -0.0421142578125
96 0.956192016601562
97 1.95456695556641
98 2.95304870605469
99 3.95159149169922
};
\end{axis}

\end{tikzpicture}

%% file: FtoKConvTrainer_500_mytikz4.tex
\begin{tikzpicture}

\begin{axis}[
xmin=-4.95, xmax=103.95, ticks=none,
ymin=-4.1, ymax=86.1,
width=\figurewidth,
height=\figureheight,
tick align=outside,
x grid style={lightgray!92.026143790849673!black},
y grid style={lightgray!92.026143790849673!black}
]
\addplot [line width=1.64pt, blue]
table {%
0 0
1 0
2 0
3 0
4 0
5 0
6 0
7 0
8 0
9 0
10 0
11 0
12 1
13 2
14 3
15 4
16 5
17 6
18 7
19 8
20 9
21 10
22 11
23 12
24 13
25 14
26 15
27 16
28 17
29 18
30 19
31 20
32 21
33 22
34 23
35 24
36 25
37 26
38 27
39 28
40 29
41 30
42 31
43 32
44 33
45 34
46 35
47 36
48 37
49 38
50 39
51 40
52 38
53 36
54 37
55 38
56 39
57 40
58 41
59 42
60 43
61 44
62 45
63 46
64 47
65 48
66 49
67 50
68 51
69 52
70 53
71 54
72 55
73 56
74 57
75 58
76 59
77 60
78 61
79 62
80 63
81 64
82 65
83 66
84 67
85 68
86 69
87 70
88 71
89 72
90 73
91 74
92 75
93 76
94 77
95 78
96 79
97 80
98 81
99 82
};
\addplot [thick, red]
table {%
0 0
1 0
2 0
3 0
4 0
5 0
6 0
7 0
8 0
9 0
10 0
11 0
12 0.0323295891284943
13 0.0820146352052689
14 0.148726463317871
15 0.232136443257332
16 0.331915855407715
17 0.447736144065857
18 0.579268574714661
19 0.726184487342834
20 0.888155281543732
21 1.06485223770142
22 1.25594663619995
23 1.46110999584198
24 1.68001341819763
25 1.91232848167419
26 2.1577262878418
27 2.41587853431702
28 2.68645620346069
29 2.96913075447083
30 3.26357364654541
31 3.56945610046387
32 3.88644957542419
33 4.21422529220581
34 4.55245494842529
35 4.90080881118774
36 5.25895929336548
37 5.62657737731934
38 6.00333452224731
39 6.38890171051025
40 6.78295135498047
41 7.18515300750732
42 7.59518003463745
43 8.0127010345459
44 8.43739032745361
45 8.86891841888428
46 9.30695533752441
47 9.75117301940918
48 10.2012434005737
49 10.6568393707275
50 11.1176280975342
51 11.5832834243774
52 11.9564876556396
53 12.2818355560303
54 12.6569728851318
55 13.0376329421997
56 13.4234867095947
57 13.814208984375
58 14.2094688415527
59 14.6089382171631
60 15.0122852325439
61 15.4191875457764
62 15.8293104171753
63 16.2423286437988
64 16.6579113006592
65 17.0757331848145
66 17.4954624176025
67 17.9167709350586
68 18.3393325805664
69 18.7628154754639
70 19.1868934631348
71 19.611234664917
72 20.0355110168457
73 20.459400177002
74 20.8825645446777
75 21.3046836853027
76 21.7254180908203
77 22.1444511413574
78 22.5614471435547
79 22.9760780334473
80 23.3880157470703
81 23.796932220459
82 24.2024993896484
83 24.6043891906738
84 25.002269744873
85 25.3958129882812
86 25.7846984863281
87 26.1685829162598
88 26.5471477508545
89 26.9200630187988
90 27.2870006561279
91 27.6476249694824
92 28.0016136169434
93 28.3486366271973
94 28.6883716583252
95 29.0204772949219
96 29.3446369171143
97 29.6605110168457
98 29.9677810668945
99 30.2661113739014
};
\end{axis}

\end{tikzpicture}

%% file: FtoKConvTrainer_10000_mytikz4.tex
\begin{tikzpicture}

\begin{axis}[
xmin=-4.95, xmax=103.95, ticks=none,
ymin=-4.1, ymax=86.1,
width=\figurewidth,
height=\figureheight,
tick align=outside,
x grid style={lightgray!92.026143790849673!black},
y grid style={lightgray!92.026143790849673!black}
]
\addplot [line width=1.64pt, blue]
table {%
0 0
1 0
2 0
3 0
4 0
5 0
6 0
7 0
8 0
9 0
10 0
11 0
12 1
13 2
14 3
15 4
16 5
17 6
18 7
19 8
20 9
21 10
22 11
23 12
24 13
25 14
26 15
27 16
28 17
29 18
30 19
31 20
32 21
33 22
34 23
35 24
36 25
37 26
38 27
39 28
40 29
41 30
42 31
43 32
44 33
45 34
46 35
47 36
48 37
49 38
50 39
51 40
52 38
53 36
54 37
55 38
56 39
57 40
58 41
59 42
60 43
61 44
62 45
63 46
64 47
65 48
66 49
67 50
68 51
69 52
70 53
71 54
72 55
73 56
74 57
75 58
76 59
77 60
78 61
79 62
80 63
81 64
82 65
83 66
84 67
85 68
86 69
87 70
88 71
89 72
90 73
91 74
92 75
93 76
94 77
95 78
96 79
97 80
98 81
99 82
};
\addplot [thick, red]
table {%
0 0
1 0
2 0
3 0
4 0
5 0
6 0
7 0
8 0
9 0
10 0
11 0
12 0.848757266998291
13 1.70052218437195
14 2.55505585670471
15 3.41211915016174
16 4.27147388458252
17 5.13288021087646
18 5.99609899520874
19 6.8608922958374
20 7.72702121734619
21 8.59424591064453
22 9.46233177185059
23 10.3310317993164
24 11.2001104354858
25 12.0693340301514
26 12.9384574890137
27 13.8072443008423
28 14.6754550933838
29 15.5428514480591
30 16.4091968536377
31 17.2742481231689
32 18.1377639770508
33 18.9995193481445
34 19.859260559082
35 20.7167510986328
36 21.5717582702637
37 22.424036026001
38 23.2733497619629
39 24.1194610595703
40 24.9621238708496
41 25.8011074066162
42 26.6361751556396
43 27.4670848846436
44 28.2935943603516
45 29.1154651641846
46 29.9324626922607
47 30.7443428039551
48 31.5508651733398
49 32.351806640625
50 33.1469078063965
51 33.9359436035156
52 32.1723937988281
53 30.393274307251
54 31.1453380584717
55 31.8918056488037
56 32.6324462890625
57 33.367015838623
58 34.095272064209
59 34.8169822692871
60 35.5318984985352
61 36.239803314209
62 36.9404373168945
63 37.6335716247559
64 38.3189659118652
65 38.9963684082031
66 39.6655540466309
67 40.3262748718262
68 40.9783020019531
69 41.6213836669922
70 42.255298614502
71 42.8797912597656
72 43.4946327209473
73 44.0995826721191
74 44.694393157959
75 45.2788391113281
76 45.8526725769043
77 46.415657043457
78 46.9675674438477
79 47.5081405639648
80 48.0371475219727
81 48.5543670654297
82 49.0595321655273
83 49.5524063110352
84 50.0327606201172
85 50.5003623962402
86 50.9549751281738
87 51.3963317871094
88 51.8242149353027
89 52.2383880615234
90 52.6386375427246
91 53.0246429443359
92 53.3962478637695
93 53.7531967163086
94 54.0951995849609
95 54.4220657348633
96 54.7335586547852
97 55.0293884277344
98 55.3093948364258
99 55.5732879638672
};
\end{axis}

\end{tikzpicture}

%% file: FtoKConvTrainer_50000_mytikz4.tex
\begin{tikzpicture}

\begin{axis}[
xmin=-4.95, xmax=103.95, ticks=none,
ymin=-4.13963394165039, ymax=86.9323127746582,
width=\figurewidth,
height=\figureheight,
tick align=outside,
x grid style={lightgray!92.026143790849673!black},
y grid style={lightgray!92.026143790849673!black}
]
\addplot [line width=1.64pt, blue]
table {%
0 0
1 0
2 0
3 0
4 0
5 0
6 0
7 0
8 0
9 0
10 0
11 0
12 1
13 2
14 3
15 4
16 5
17 6
18 7
19 8
20 9
21 10
22 11
23 12
24 13
25 14
26 15
27 16
28 17
29 18
30 19
31 20
32 21
33 22
34 23
35 24
36 25
37 26
38 27
39 28
40 29
41 30
42 31
43 32
44 33
45 34
46 35
47 36
48 37
49 38
50 39
51 40
52 38
53 36
54 37
55 38
56 39
57 40
58 41
59 42
60 43
61 44
62 45
63 46
64 47
65 48
66 49
67 50
68 51
69 52
70 53
71 54
72 55
73 56
74 57
75 58
76 59
77 60
78 61
79 62
80 63
81 64
82 65
83 66
84 67
85 68
86 69
87 70
88 71
89 72
90 73
91 74
92 75
93 76
94 77
95 78
96 79
97 80
98 81
99 82
};
\addplot [thick, red]
table {%
0 0
1 0
2 0
3 0
4 0
5 0
6 0
7 0
8 0
9 0
10 0
11 0
12 0.999992251396179
13 2.00005507469177
14 3.00019359588623
15 4.000412940979
16 5.00071954727173
17 6.00111675262451
18 7.00161123275757
19 8.00220775604248
20 9.0029125213623
21 10.0037279129028
22 11.0046615600586
23 12.0057172775269
24 13.0069007873535
25 14.0082187652588
26 15.0096778869629
27 16.0112781524658
28 17.0130271911621
29 18.0149345397949
30 19.0170059204102
31 20.0192375183105
32 21.0216388702393
33 22.0242252349854
34 23.026985168457
35 24.0299243927002
36 25.0330619812012
37 26.0363941192627
38 27.0399322509766
39 28.0436706542969
40 29.0476226806641
41 30.051794052124
42 31.0561866760254
43 32.0608062744141
44 33.0656547546387
45 34.0707473754883
46 35.0760803222656
47 36.0816688537598
48 37.087516784668
49 38.093620300293
50 39.099983215332
51 40.1066246032715
52 38.1135520935059
53 36.1205673217773
54 37.1276054382324
55 38.1349029541016
56 39.1424674987793
57 40.1503067016602
58 41.1584167480469
59 42.1668090820312
60 43.1754951477051
61 44.1844635009766
62 45.1937370300293
63 46.2033157348633
64 47.2132034301758
65 48.2233924865723
66 49.2339096069336
67 50.2447357177734
68 51.2558937072754
69 52.2673797607422
70 53.2792091369629
71 54.2913780212402
72 55.303897857666
73 56.3167610168457
74 57.3299942016602
75 58.3435859680176
76 59.3575325012207
77 60.3718566894531
78 61.3865776062012
79 62.401668548584
80 63.417163848877
81 64.4330520629883
82 65.4493103027344
83 66.4660034179688
84 67.4831161499023
85 68.5006103515625
86 69.5185699462891
87 70.5369415283203
88 71.5557403564453
89 72.5750045776367
90 73.5946578979492
91 74.6148223876953
92 75.635383605957
93 76.6564331054688
94 77.6779556274414
95 78.6999282836914
96 79.7223968505859
97 80.7453384399414
98 81.7687301635742
99 82.7926788330078
};
\end{axis}

\end{tikzpicture}

%% file: FtoKConvCondTrainer_500_mytikz1.tex
\begin{tikzpicture}

\begin{axis}[
xmin=-4.95, xmax=103.95, ticks=none,
ymin=-21.4500240325928, ymax=10.4500011444092,
width=\figurewidth,
height=\figureheight,
tick align=outside,
x grid style={lightgray!92.026143790849673!black},
y grid style={lightgray!92.026143790849673!black}
]
\addplot [line width=1.64pt, blue]
table {%
0 0
1 0
2 0
3 0
4 0
5 0
6 0
7 0
8 0
9 0
10 0
11 0
12 0
13 0
14 0
15 0
16 0
17 0
18 0
19 0
20 0
21 0
22 0
23 0
24 0
25 0
26 0
27 0
28 0
29 0
30 0
31 0
32 0
33 0
34 0
35 0
36 0
37 0
38 1
39 2
40 3
41 4
42 5
43 6
44 7
45 8
46 9
47 8
48 7
49 6
50 5
51 4
52 3
53 2
54 1
55 0
56 -1
57 -2
58 -3
59 -4
60 -5
61 -6
62 -7
63 -8
64 -9
65 -10
66 -11
67 -12
68 -13
69 -14
70 -15
71 -16
72 -17
73 -18
74 -19
75 -20
76 -19
77 -18
78 -17
79 -16
80 -15
81 -14
82 -13
83 -12
84 -11
85 -10
86 -9
87 -8
88 -7
89 -6
90 -5
91 -4
92 -3
93 -2
94 -1
95 0
96 1
97 2
98 3
99 4
};
\addplot [thick, red]
table {%
0 0
1 0
2 0
3 0
4 0
5 0
6 0
7 0
8 0
9 0
10 0
11 0
12 0
13 0
14 0
15 0
16 0
17 0
18 0
19 0
20 0
21 0
22 0
23 0
24 0
25 0
26 0
27 0
28 0
29 0
30 0
31 0
32 0
33 0
34 0
35 0
36 0
37 0
38 0.99999988079071
39 1.99999976158142
40 2.99999952316284
41 3.99999952316284
42 4.99999952316284
43 5.99999952316284
44 6.99999904632568
45 7.99999856948853
46 8.99999904632568
47 7.99999809265137
48 6.99999713897705
49 5.99999713897705
50 4.99999618530273
51 3.99999523162842
52 2.9999942779541
53 1.99999332427979
54 0.999992370605469
55 -7.62939453125e-06
56 -1.00000953674316
57 -2.00000953674316
58 -3.00000953674316
59 -4.0000114440918
60 -5.0000114440918
61 -6.0000171661377
62 -7.00001907348633
63 -8.0000171661377
64 -9.00001907348633
65 -10.0000190734863
66 -11.0000190734863
67 -12.0000228881836
68 -13.000020980835
69 -14.0000171661377
70 -15.0000114440918
71 -16.0000152587891
72 -17.0000228881836
73 -18.0000305175781
74 -19.0000228881836
75 -20.0000228881836
76 -19.0000190734863
77 -18.0000267028809
78 -17.0000305175781
79 -16.0000305175781
80 -15.0000267028809
81 -14.0000267028809
82 -13.0000343322754
83 -12.0000381469727
84 -11.0000381469727
85 -10.0000228881836
86 -9.00003814697266
87 -8.00003814697266
88 -7.00003051757812
89 -6.00003051757812
90 -5.00003051757812
91 -4.00003814697266
92 -3.00003814697266
93 -2.00004577636719
94 -1.00003814697266
95 -3.0517578125e-05
96 0.999946594238281
97 1.99994659423828
98 2.99995422363281
99 3.99994659423828
};
\end{axis}

\end{tikzpicture}

%% file: FtoKConvCondTrainer_10000_mytikz1.tex
\begin{tikzpicture}

\begin{axis}[
xmin=-4.95, xmax=103.95, ticks=none,
ymin=-21.4500014305115, ymax=10.450030040741,
width=\figurewidth,
height=\figureheight,
tick align=outside,
x grid style={lightgray!92.026143790849673!black},
y grid style={lightgray!92.026143790849673!black}
]
\addplot [line width=1.64pt, blue]
table {%
0 0
1 0
2 0
3 0
4 0
5 0
6 0
7 0
8 0
9 0
10 0
11 0
12 0
13 0
14 0
15 0
16 0
17 0
18 0
19 0
20 0
21 0
22 0
23 0
24 0
25 0
26 0
27 0
28 0
29 0
30 0
31 0
32 0
33 0
34 0
35 0
36 0
37 0
38 1
39 2
40 3
41 4
42 5
43 6
44 7
45 8
46 9
47 8
48 7
49 6
50 5
51 4
52 3
53 2
54 1
55 0
56 -1
57 -2
58 -3
59 -4
60 -5
61 -6
62 -7
63 -8
64 -9
65 -10
66 -11
67 -12
68 -13
69 -14
70 -15
71 -16
72 -17
73 -18
74 -19
75 -20
76 -19
77 -18
78 -17
79 -16
80 -15
81 -14
82 -13
83 -12
84 -11
85 -10
86 -9
87 -8
88 -7
89 -6
90 -5
91 -4
92 -3
93 -2
94 -1
95 0
96 1
97 2
98 3
99 4
};
\addplot [thick, red]
table {%
0 0
1 0
2 0
3 0
4 0
5 0
6 0
7 0
8 0
9 0
10 0
11 0
12 0
13 0
14 0
15 0
16 0
17 0
18 0
19 0
20 0
21 0
22 0
23 0
24 0
25 0
26 0
27 0
28 0
29 0
30 0
31 0
32 0
33 0
34 0
35 0
36 0
37 0
38 0.999988555908203
39 1.99997985363007
40 2.99997425079346
41 3.9999725818634
42 4.99997425079346
43 5.99998092651367
44 6.99999141693115
45 8.00000762939453
46 9.00002861022949
47 8.00008010864258
48 7.00013256072998
49 6.00018453598022
50 5.00023794174194
51 4.00028991699219
52 3.00034046173096
53 2.00039100646973
54 1.00043773651123
55 0.000486373901367188
56 -0.999471664428711
57 -1.99942970275879
58 -2.99939155578613
59 -3.99935722351074
60 -4.99932670593262
61 -5.99930000305176
62 -6.99927711486816
63 -7.99926376342773
64 -8.99924850463867
65 -9.99924850463867
66 -10.9992485046387
67 -11.9992561340332
68 -12.9992752075195
69 -13.9993000030518
70 -14.999324798584
71 -15.9993629455566
72 -16.9994163513184
73 -17.9994621276855
74 -18.9995307922363
75 -19.9996109008789
76 -18.9997177124023
77 -17.999828338623
78 -16.9999504089355
79 -16.000057220459
80 -15.0002059936523
81 -14.0003128051758
82 -13.0004577636719
83 -12.0005874633789
84 -11.0007247924805
85 -10.000862121582
86 -9.00099182128906
87 -8.00115203857422
88 -7.00127410888672
89 -6.00143432617188
90 -5.00156402587891
91 -4.00171661376953
92 -3.00184631347656
93 -2.00200653076172
94 -1.00213623046875
95 -0.0022735595703125
96 0.997596740722656
97 1.99745178222656
98 2.99734497070312
99 3.99716949462891
};
\end{axis}

\end{tikzpicture}

%% file: FtoKConvCondTrainer_50000_mytikz1.tex
\begin{tikzpicture}

\begin{axis}[
xmin=-4.95, xmax=103.95, ticks=none,
ymin=-21.4500014305115, ymax=10.450030040741,
width=\figurewidth,
height=\figureheight,
tick align=outside,
x grid style={lightgray!92.026143790849673!black},
y grid style={lightgray!92.026143790849673!black}
]
\addplot [line width=1.64pt, blue]
table {%
0 0
1 0
2 0
3 0
4 0
5 0
6 0
7 0
8 0
9 0
10 0
11 0
12 0
13 0
14 0
15 0
16 0
17 0
18 0
19 0
20 0
21 0
22 0
23 0
24 0
25 0
26 0
27 0
28 0
29 0
30 0
31 0
32 0
33 0
34 0
35 0
36 0
37 0
38 1
39 2
40 3
41 4
42 5
43 6
44 7
45 8
46 9
47 8
48 7
49 6
50 5
51 4
52 3
53 2
54 1
55 0
56 -1
57 -2
58 -3
59 -4
60 -5
61 -6
62 -7
63 -8
64 -9
65 -10
66 -11
67 -12
68 -13
69 -14
70 -15
71 -16
72 -17
73 -18
74 -19
75 -20
76 -19
77 -18
78 -17
79 -16
80 -15
81 -14
82 -13
83 -12
84 -11
85 -10
86 -9
87 -8
88 -7
89 -6
90 -5
91 -4
92 -3
93 -2
94 -1
95 0
96 1
97 2
98 3
99 4
};
\addplot [thick, red]
table {%
0 0
1 0
2 0
3 0
4 0
5 0
6 0
7 0
8 0
9 0
10 0
11 0
12 0
13 0
14 0
15 0
16 0
17 0
18 0
19 0
20 0
21 0
22 0
23 0
24 0
25 0
26 0
27 0
28 0
29 0
30 0
31 0
32 0
33 0
34 0
35 0
36 0
37 0
38 0.999988555908203
39 1.99997985363007
40 2.99997425079346
41 3.9999725818634
42 4.99997425079346
43 5.99998092651367
44 6.99999141693115
45 8.00000762939453
46 9.00002861022949
47 8.00008010864258
48 7.00013256072998
49 6.00018453598022
50 5.00023794174194
51 4.00028991699219
52 3.00034046173096
53 2.00039100646973
54 1.00043773651123
55 0.000486373901367188
56 -0.999471664428711
57 -1.99942970275879
58 -2.99939155578613
59 -3.99935722351074
60 -4.99932670593262
61 -5.99930000305176
62 -6.99927711486816
63 -7.99926376342773
64 -8.99924850463867
65 -9.99924850463867
66 -10.9992485046387
67 -11.9992561340332
68 -12.9992752075195
69 -13.9993000030518
70 -14.999324798584
71 -15.9993629455566
72 -16.9994163513184
73 -17.9994621276855
74 -18.9995307922363
75 -19.9996109008789
76 -18.9997177124023
77 -17.999828338623
78 -16.9999504089355
79 -16.000057220459
80 -15.0002059936523
81 -14.0003128051758
82 -13.0004577636719
83 -12.0005874633789
84 -11.0007247924805
85 -10.000862121582
86 -9.00099182128906
87 -8.00115203857422
88 -7.00127410888672
89 -6.00143432617188
90 -5.00156402587891
91 -4.00171661376953
92 -3.00184631347656
93 -2.00200653076172
94 -1.00213623046875
95 -0.0022735595703125
96 0.997596740722656
97 1.99745178222656
98 2.99734497070312
99 3.99716949462891
};
\end{axis}

\end{tikzpicture}

%% file: FtoKConvCondTrainer_500_mytikz4.tex
\begin{tikzpicture}

\begin{axis}[
xmin=-4.95, xmax=103.95, ticks=none,
ymin=-4.1, ymax=86.1,
width=\figurewidth,
height=\figureheight,
tick align=outside,
x grid style={lightgray!92.026143790849673!black},
y grid style={lightgray!92.026143790849673!black}
]
\addplot [line width=1.64pt, blue]
table {%
0 0
1 0
2 0
3 0
4 0
5 0
6 0
7 0
8 0
9 0
10 0
11 0
12 1
13 2
14 3
15 4
16 5
17 6
18 7
19 8
20 9
21 10
22 11
23 12
24 13
25 14
26 15
27 16
28 17
29 18
30 19
31 20
32 21
33 22
34 23
35 24
36 25
37 26
38 27
39 28
40 29
41 30
42 31
43 32
44 33
45 34
46 35
47 36
48 37
49 38
50 39
51 40
52 38
53 36
54 37
55 38
56 39
57 40
58 41
59 42
60 43
61 44
62 45
63 46
64 47
65 48
66 49
67 50
68 51
69 52
70 53
71 54
72 55
73 56
74 57
75 58
76 59
77 60
78 61
79 62
80 63
81 64
82 65
83 66
84 67
85 68
86 69
87 70
88 71
89 72
90 73
91 74
92 75
93 76
94 77
95 78
96 79
97 80
98 81
99 82
};
\addplot [thick, red]
table {%
0 0
1 0
2 0
3 0
4 0
5 0
6 0
7 0
8 0
9 0
10 0
11 0
12 0.99999988079071
13 1.99999976158142
14 2.99999952316284
15 3.99999952316284
16 4.99999952316284
17 5.99999952316284
18 6.99999904632568
19 7.99999856948853
20 8.99999904632568
21 9.99999713897705
22 10.9999971389771
23 11.9999971389771
24 12.9999961853027
25 13.9999952316284
26 14.9999942779541
27 15.9999952316284
28 16.9999942779541
29 17.9999923706055
30 18.9999904632568
31 19.9999904632568
32 20.9999923706055
33 21.9999885559082
34 22.9999923706055
35 23.9999885559082
36 24.9999904632568
37 25.9999866485596
38 26.9999866485596
39 27.9999866485596
40 28.9999866485596
41 29.9999847412109
42 30.9999847412109
43 31.9999847412109
44 32.9999847412109
45 33.9999809265137
46 34.9999885559082
47 35.9999771118164
48 36.9999694824219
49 37.9999656677246
50 38.9999618530273
51 39.9999504089355
52 37.9999618530273
53 35.9999580383301
54 36.9999618530273
55 37.9999504089355
56 38.9999580383301
57 39.9999656677246
58 40.9999656677246
59 41.9999694824219
60 42.9999809265137
61 43.9999771118164
62 44.9999847412109
63 45.9999885559082
64 46.9999961853027
65 47.9999885559082
66 48.9999923706055
67 49.9999961853027
68 50.9999847412109
69 51.9999923706055
70 53
71 54.0000038146973
72 55
73 56.0000114440918
74 57
75 58.0000114440918
76 58.9999961853027
77 59.9999961853027
78 60.9999923706055
79 62.0000114440918
80 63.0000038146973
81 64.0000076293945
82 64.9999771118164
83 65.9999771118164
84 66.9999771118164
85 67.9999694824219
86 68.9999694824219
87 69.9999847412109
88 70.9999694824219
89 71.9999847412109
90 72.9999694824219
91 73.9999771118164
92 74.9999771118164
93 75.9999542236328
94 76.9999771118164
95 77.9999847412109
96 79.0000076293945
97 79.9999923706055
98 80.9999923706055
99 81.9999923706055
};
\end{axis}

\end{tikzpicture}

%% file: FtoKConvCondTrainer_10000_mytikz4.tex
\begin{tikzpicture}

\begin{axis}[
xmin=-4.95, xmax=103.95, ticks=none,
ymin=-4.10242385864258, ymax=86.1509010314941,
width=\figurewidth,
height=\figureheight,
tick align=outside,
x grid style={lightgray!92.026143790849673!black},
y grid style={lightgray!92.026143790849673!black}
]
\addplot [line width=1.64pt, blue]
table {%
0 0
1 0
2 0
3 0
4 0
5 0
6 0
7 0
8 0
9 0
10 0
11 0
12 1
13 2
14 3
15 4
16 5
17 6
18 7
19 8
20 9
21 10
22 11
23 12
24 13
25 14
26 15
27 16
28 17
29 18
30 19
31 20
32 21
33 22
34 23
35 24
36 25
37 26
38 27
39 28
40 29
41 30
42 31
43 32
44 33
45 34
46 35
47 36
48 37
49 38
50 39
51 40
52 38
53 36
54 37
55 38
56 39
57 40
58 41
59 42
60 43
61 44
62 45
63 46
64 47
65 48
66 49
67 50
68 51
69 52
70 53
71 54
72 55
73 56
74 57
75 58
76 59
77 60
78 61
79 62
80 63
81 64
82 65
83 66
84 67
85 68
86 69
87 70
88 71
89 72
90 73
91 74
92 75
93 76
94 77
95 78
96 79
97 80
98 81
99 82
};
\addplot [thick, red]
table {%
0 0
1 0
2 0
3 0
4 0
5 0
6 0
7 0
8 0
9 0
10 0
11 0
12 0.999988555908203
13 1.99997985363007
14 2.99997425079346
15 3.9999725818634
16 4.99997425079346
17 5.99998092651367
18 6.99999141693115
19 8.00000762939453
20 9.00002861022949
21 10.000057220459
22 11.0000915527344
23 12.0001316070557
24 13.0001802444458
25 14.0002365112305
26 15.0003004074097
27 16.0003719329834
28 17.0004501342773
29 18.0005378723145
30 19.0006370544434
31 20.0007419586182
32 21.0008602142334
33 22.0009880065918
34 23.0011291503906
35 24.0012836456299
36 25.0014476776123
37 26.0016250610352
38 27.0018119812012
39 28.002010345459
40 29.0022239685059
41 30.0024490356445
42 31.0026912689209
43 32.0029449462891
44 33.0032081604004
45 34.0034980773926
46 35.0037841796875
47 36.0041007995605
48 37.0044212341309
49 38.004768371582
50 39.0051307678223
51 40.0055084228516
52 38.0059356689453
53 36.0063781738281
54 37.0067939758301
55 38.0072174072266
56 39.0076751708984
57 40.0081481933594
58 41.0086288452148
59 42.0091323852539
60 43.0096549987793
61 44.0102005004883
62 45.0107536315918
63 46.0113410949707
64 47.0119400024414
65 48.0125541687012
66 49.0132026672363
67 50.0138702392578
68 51.0145530700684
69 52.0152587890625
70 53.0159912109375
71 54.0167427062988
72 55.0175170898438
73 56.0183067321777
74 57.0191383361816
75 58.0199813842773
76 59.0208702087402
77 60.0217399597168
78 61.0226631164551
79 62.023609161377
80 63.0246086120605
81 64.0255966186523
82 65.0266342163086
83 66.0276870727539
84 67.0287551879883
85 68.0298614501953
86 69.0309982299805
87 70.0321807861328
88 71.0333404541016
89 72.0345458984375
90 73.0357971191406
91 74.0370864868164
92 75.0383834838867
93 76.0397262573242
94 77.0411071777344
95 78.0425415039062
96 79.0439682006836
97 80.0454635620117
98 81.0469436645508
99 82.0484771728516
};
\end{axis}

\end{tikzpicture}

%% file: FtoKConvCondTrainer_50000_mytikz4.tex
\begin{tikzpicture}

\begin{axis}[
xmin=-4.95, xmax=103.95, ticks=none,
ymin=-4.10242385864258, ymax=86.1509010314941,
width=\figurewidth,
height=\figureheight,
tick align=outside,
x grid style={lightgray!92.026143790849673!black},
y grid style={lightgray!92.026143790849673!black}
]
\addplot [line width=1.64pt, blue]
table {%
0 0
1 0
2 0
3 0
4 0
5 0
6 0
7 0
8 0
9 0
10 0
11 0
12 1
13 2
14 3
15 4
16 5
17 6
18 7
19 8
20 9
21 10
22 11
23 12
24 13
25 14
26 15
27 16
28 17
29 18
30 19
31 20
32 21
33 22
34 23
35 24
36 25
37 26
38 27
39 28
40 29
41 30
42 31
43 32
44 33
45 34
46 35
47 36
48 37
49 38
50 39
51 40
52 38
53 36
54 37
55 38
56 39
57 40
58 41
59 42
60 43
61 44
62 45
63 46
64 47
65 48
66 49
67 50
68 51
69 52
70 53
71 54
72 55
73 56
74 57
75 58
76 59
77 60
78 61
79 62
80 63
81 64
82 65
83 66
84 67
85 68
86 69
87 70
88 71
89 72
90 73
91 74
92 75
93 76
94 77
95 78
96 79
97 80
98 81
99 82
};
\addplot [thick, red]
table {%
0 0
1 0
2 0
3 0
4 0
5 0
6 0
7 0
8 0
9 0
10 0
11 0
12 0.999988555908203
13 1.99997985363007
14 2.99997425079346
15 3.9999725818634
16 4.99997425079346
17 5.99998092651367
18 6.99999141693115
19 8.00000762939453
20 9.00002861022949
21 10.000057220459
22 11.0000915527344
23 12.0001316070557
24 13.0001802444458
25 14.0002365112305
26 15.0003004074097
27 16.0003719329834
28 17.0004501342773
29 18.0005378723145
30 19.0006370544434
31 20.0007419586182
32 21.0008602142334
33 22.0009880065918
34 23.0011291503906
35 24.0012836456299
36 25.0014476776123
37 26.0016250610352
38 27.0018119812012
39 28.002010345459
40 29.0022239685059
41 30.0024490356445
42 31.0026912689209
43 32.0029449462891
44 33.0032081604004
45 34.0034980773926
46 35.0037841796875
47 36.0041007995605
48 37.0044212341309
49 38.004768371582
50 39.0051307678223
51 40.0055084228516
52 38.0059356689453
53 36.0063781738281
54 37.0067939758301
55 38.0072174072266
56 39.0076751708984
57 40.0081481933594
58 41.0086288452148
59 42.0091323852539
60 43.0096549987793
61 44.0102005004883
62 45.0107536315918
63 46.0113410949707
64 47.0119400024414
65 48.0125541687012
66 49.0132026672363
67 50.0138702392578
68 51.0145530700684
69 52.0152587890625
70 53.0159912109375
71 54.0167427062988
72 55.0175170898438
73 56.0183067321777
74 57.0191383361816
75 58.0199813842773
76 59.0208702087402
77 60.0217399597168
78 61.0226631164551
79 62.023609161377
80 63.0246086120605
81 64.0255966186523
82 65.0266342163086
83 66.0276870727539
84 67.0287551879883
85 68.0298614501953
86 69.0309982299805
87 70.0321807861328
88 71.0333404541016
89 72.0345458984375
90 73.0357971191406
91 74.0370864868164
92 75.0383834838867
93 76.0397262573242
94 77.0411071777344
95 78.0425415039062
96 79.0439682006836
97 80.0454635620117
98 81.0469436645508
99 82.0484771728516
};
\end{axis}

\end{tikzpicture}

%% file: FAutoEncoderTrainer_500_mytikz1.tex
\begin{tikzpicture}

\begin{axis}[
xmin=-4.95, xmax=103.95, ticks=none,
ymin=-21.45, ymax=10.45,
width=\figurewidth,
height=\figureheight,
tick align=outside,
x grid style={lightgray!92.026143790849673!black},
y grid style={lightgray!92.026143790849673!black}
]
\addplot [line width=1.64pt, blue]
table {%
0 0
1 0
2 0
3 0
4 0
5 0
6 0
7 0
8 0
9 0
10 0
11 0
12 0
13 0
14 0
15 0
16 0
17 0
18 0
19 0
20 0
21 0
22 0
23 0
24 0
25 0
26 0
27 0
28 0
29 0
30 0
31 0
32 0
33 0
34 0
35 0
36 0
37 0
38 0
39 1
40 2
41 3
42 4
43 5
44 6
45 7
46 8
47 9
48 8
49 7
50 6
51 5
52 4
53 3
54 2
55 1
56 0
57 -1
58 -2
59 -3
60 -4
61 -5
62 -6
63 -7
64 -8
65 -9
66 -10
67 -11
68 -12
69 -13
70 -14
71 -15
72 -16
73 -17
74 -18
75 -19
76 -20
77 -19
78 -18
79 -17
80 -16
81 -15
82 -14
83 -13
84 -12
85 -11
86 -10
87 -9
88 -8
89 -7
90 -6
91 -5
92 -4
93 -3
94 -2
95 -1
96 0
97 1
98 2
99 3
};
\addplot [thick, red]
table {%
0 -0.895403742790222
1 -1.22367548942566
2 -0.995910048484802
3 -2.7716965675354
4 0.374166965484619
5 -1.97640955448151
6 -0.592536747455597
7 1.2096688747406
8 0.532388091087341
9 1.7824102640152
10 3.17051577568054
11 2.30676555633545
12 0.175725549459457
13 3.7025203704834
14 1.04728174209595
15 3.41768574714661
16 1.42631328105927
17 2.18472170829773
18 4.58659172058105
19 -0.1012904047966
20 3.83838272094727
21 3.64327907562256
22 1.1259263753891
23 3.37723469734192
24 1.81837344169617
25 2.58925175666809
26 2.58598256111145
27 4.89510297775269
28 3.30877017974854
29 0.586589157581329
30 2.936199426651
31 3.93065786361694
32 1.68755125999451
33 1.40529608726501
34 2.01077628135681
35 3.29160666465759
36 3.52808523178101
37 3.20651054382324
38 2.54087591171265
39 1.94736862182617
40 1.29008364677429
41 1.90241742134094
42 0.0911586433649063
43 -1.42338502407074
44 -0.0676268637180328
45 -1.78553342819214
46 -3.11880683898926
47 1.37473380565643
48 -1.34222805500031
49 -3.25036597251892
50 -1.36063647270203
51 -3.48775219917297
52 -4.03801107406616
53 -2.26598024368286
54 -4.07351922988892
55 -4.88585376739502
56 -5.44001913070679
57 -7.54631662368774
58 -6.43288040161133
59 -5.7834324836731
60 -6.15128421783447
61 -6.79979133605957
62 -7.73288583755493
63 -5.66089010238647
64 -8.18808174133301
65 -7.14514255523682
66 -6.21271944046021
67 -8.36264228820801
68 -7.25728845596313
69 -6.8155689239502
70 -10.3991432189941
71 -9.69802284240723
72 -7.79610109329224
73 -10.3732261657715
74 -7.68062782287598
75 -9.73330593109131
76 -11.0236158370972
77 -7.73386573791504
78 -10.9291000366211
79 -8.83367824554443
80 -7.61941003799438
81 -10.6383047103882
82 -6.91062116622925
83 -7.19870185852051
84 -7.47081661224365
85 -9.76853179931641
86 -8.48686790466309
87 -9.84215545654297
88 -8.47906684875488
89 -9.15750980377197
90 -7.4945387840271
91 -8.46850872039795
92 -7.88849878311157
93 -7.1205153465271
94 -5.67998504638672
95 -7.92092800140381
96 -5.51077508926392
97 -9.13301467895508
98 -7.18940591812134
99 -5.11364364624023
};
\end{axis}

\end{tikzpicture}

%% file: FAutoEncoderTrainer_10000_mytikz1.tex
\begin{tikzpicture}

\begin{axis}[
xmin=-4.95, xmax=103.95, ticks=none,
ymin=-21.45, ymax=10.45,
width=\figurewidth,
height=\figureheight,
tick align=outside,
x grid style={lightgray!92.026143790849673!black},
y grid style={lightgray!92.026143790849673!black}
]
\addplot [line width=1.64pt, blue]
table {%
0 0
1 0
2 0
3 0
4 0
5 0
6 0
7 0
8 0
9 0
10 0
11 0
12 0
13 0
14 0
15 0
16 0
17 0
18 0
19 0
20 0
21 0
22 0
23 0
24 0
25 0
26 0
27 0
28 0
29 0
30 0
31 0
32 0
33 0
34 0
35 0
36 0
37 0
38 0
39 1
40 2
41 3
42 4
43 5
44 6
45 7
46 8
47 9
48 8
49 7
50 6
51 5
52 4
53 3
54 2
55 1
56 0
57 -1
58 -2
59 -3
60 -4
61 -5
62 -6
63 -7
64 -8
65 -9
66 -10
67 -11
68 -12
69 -13
70 -14
71 -15
72 -16
73 -17
74 -18
75 -19
76 -20
77 -19
78 -18
79 -17
80 -16
81 -15
82 -14
83 -13
84 -12
85 -11
86 -10
87 -9
88 -8
89 -7
90 -6
91 -5
92 -4
93 -3
94 -2
95 -1
96 0
97 1
98 2
99 3
};
\addplot [thick, red]
table {%
0 0.0333801880478859
1 -0.106046333909035
2 -0.108085364103317
3 0.0911639705300331
4 0.0486109256744385
5 -0.287403911352158
6 -0.207376673817635
7 -0.0346805974841118
8 -0.169859692454338
9 0.0341572910547256
10 -0.184240251779556
11 -0.306138008832932
12 0.108227878808975
13 0.307169228792191
14 0.0659080669283867
15 0.210263431072235
16 0.344531238079071
17 0.0719461515545845
18 0.364930093288422
19 0.38566979765892
20 0.577516555786133
21 0.674111366271973
22 0.810160577297211
23 0.971998453140259
24 1.0531051158905
25 0.95813262462616
26 1.091423869133
27 1.24409008026123
28 1.24957168102264
29 1.4067051410675
30 1.31738138198853
31 1.20708525180817
32 1.38729870319366
33 1.0615131855011
34 1.31254148483276
35 0.867791354656219
36 1.01795315742493
37 0.712983548641205
38 0.83282482624054
39 0.658086717128754
40 0.503829658031464
41 0.429709613323212
42 0.652105748653412
43 0.0811723172664642
44 -0.61363422870636
45 -0.191954672336578
46 -0.334244161844254
47 -0.716805934906006
48 -0.911810100078583
49 -1.04762375354767
50 -1.48804879188538
51 -1.68867838382721
52 -2.01615357398987
53 -2.36088442802429
54 -2.89044523239136
55 -3.63182234764099
56 -3.30388832092285
57 -3.78013849258423
58 -4.5233850479126
59 -4.95771408081055
60 -5.1730055809021
61 -5.69522094726562
62 -6.25450992584229
63 -6.51033353805542
64 -6.88766860961914
65 -7.28895092010498
66 -7.63965368270874
67 -8.49337482452393
68 -8.82928943634033
69 -9.35289859771729
70 -9.84965038299561
71 -10.2455072402954
72 -10.5276327133179
73 -11.2299509048462
74 -11.3830251693726
75 -11.9648103713989
76 -12.411922454834
77 -12.7016010284424
78 -13.134033203125
79 -12.936222076416
80 -13.191987991333
81 -12.8057746887207
82 -13.0029363632202
83 -12.8065366744995
84 -12.9644317626953
85 -12.5615301132202
86 -12.4001321792603
87 -11.7282018661499
88 -11.8054304122925
89 -10.8510999679565
90 -10.73410987854
91 -10.1154890060425
92 -9.98893928527832
93 -9.29999256134033
94 -8.63604068756104
95 -7.73266792297363
96 -7.51014804840088
97 -7.04870653152466
98 -6.47504901885986
99 -5.91030693054199
};
\end{axis}

\end{tikzpicture}

%% file: FAutoEncoderTrainer_50000_mytikz1.tex
\begin{tikzpicture}

\begin{axis}[
xmin=-4.95, xmax=103.95, ticks=none,
ymin=-21.45, ymax=10.45,
width=\figurewidth,
height=\figureheight,
tick align=outside,
x grid style={lightgray!92.026143790849673!black},
y grid style={lightgray!92.026143790849673!black}
]
\addplot [line width=1.64pt, blue]
table {%
0 0
1 0
2 0
3 0
4 0
5 0
6 0
7 0
8 0
9 0
10 0
11 0
12 0
13 0
14 0
15 0
16 0
17 0
18 0
19 0
20 0
21 0
22 0
23 0
24 0
25 0
26 0
27 0
28 0
29 0
30 0
31 0
32 0
33 0
34 0
35 0
36 0
37 0
38 0
39 1
40 2
41 3
42 4
43 5
44 6
45 7
46 8
47 9
48 8
49 7
50 6
51 5
52 4
53 3
54 2
55 1
56 0
57 -1
58 -2
59 -3
60 -4
61 -5
62 -6
63 -7
64 -8
65 -9
66 -10
67 -11
68 -12
69 -13
70 -14
71 -15
72 -16
73 -17
74 -18
75 -19
76 -20
77 -19
78 -18
79 -17
80 -16
81 -15
82 -14
83 -13
84 -12
85 -11
86 -10
87 -9
88 -8
89 -7
90 -6
91 -5
92 -4
93 -3
94 -2
95 -1
96 0
97 1
98 2
99 3
};
\addplot [thick, red]
table {%
0 0.194749966263771
1 0.200206100940704
2 0.405748099088669
3 0.158102497458458
4 -0.0840198546648026
5 0.597282111644745
6 0.432281136512756
7 0.266960740089417
8 0.386163473129272
9 0.510665416717529
10 0.381205558776855
11 0.740670382976532
12 0.276468396186829
13 0.710728764533997
14 0.56970351934433
15 0.286250680685043
16 0.396598488092422
17 0.511542022228241
18 0.59746128320694
19 0.341365128755569
20 0.165563255548477
21 0.414099484682083
22 0.113098070025444
23 0.570797979831696
24 0.346702128648758
25 0.568476319313049
26 0.482431262731552
27 0.456416606903076
28 0.459087908267975
29 0.60431843996048
30 0.418678343296051
31 0.488544881343842
32 0.522432506084442
33 0.613852858543396
34 0.446736812591553
35 0.868480026721954
36 0.999678552150726
37 1.35945463180542
38 1.66694331169128
39 1.83940148353577
40 2.01377201080322
41 2.63549494743347
42 2.78470993041992
43 3.10657525062561
44 3.83159589767456
45 3.29968881607056
46 3.72159218788147
47 3.70421195030212
48 3.72751617431641
49 3.56267642974854
50 2.78807878494263
51 2.92758655548096
52 1.90526151657104
53 1.10344135761261
54 0.11705008149147
55 -0.730757892131805
56 -1.68853962421417
57 -2.78294706344604
58 -3.52750754356384
59 -4.95605659484863
60 -5.5374903678894
61 -6.84408617019653
62 -7.71405267715454
63 -8.3218412399292
64 -9.40552520751953
65 -10.7349967956543
66 -11.5681896209717
67 -12.3188142776489
68 -13.0950317382812
69 -14.1139039993286
70 -14.9775886535645
71 -15.5091304779053
72 -15.5949058532715
73 -16.0115547180176
74 -16.0255069732666
75 -16.085018157959
76 -15.9620590209961
77 -15.5332012176514
78 -15.2893915176392
79 -15.2312440872192
80 -14.3429908752441
81 -13.2511873245239
82 -12.9595804214478
83 -12.5459547042847
84 -12.2996006011963
85 -10.9950389862061
86 -10.0456256866455
87 -9.95797061920166
88 -9.65723609924316
89 -8.36069202423096
90 -7.40903759002686
91 -6.6023268699646
92 -6.11694192886353
93 -5.37817430496216
94 -4.67748355865479
95 -4.03761577606201
96 -3.74309253692627
97 -3.01009821891785
98 -2.4201831817627
99 -1.1100709438324
};
\end{axis}

\end{tikzpicture}

%% file: FAutoEncoderTrainer_500_mytikz4.tex
\begin{tikzpicture}

\begin{axis}[
xmin=-4.95, xmax=103.95, ticks=none,
ymin=-4.76336696147919, ymax=85.0839698553085,
width=\figurewidth,
height=\figureheight,
tick align=outside,
x grid style={lightgray!92.026143790849673!black},
y grid style={lightgray!92.026143790849673!black}
]
\addplot [line width=1.64pt, blue]
table {%
0 0
1 0
2 0
3 0
4 0
5 0
6 0
7 0
8 0
9 0
10 0
11 0
12 0
13 1
14 2
15 3
16 4
17 5
18 6
19 7
20 8
21 9
22 10
23 11
24 12
25 13
26 14
27 15
28 16
29 17
30 18
31 19
32 20
33 21
34 22
35 23
36 24
37 25
38 26
39 27
40 28
41 29
42 30
43 31
44 32
45 33
46 34
47 35
48 36
49 37
50 38
51 39
52 40
53 38
54 36
55 37
56 38
57 39
58 40
59 41
60 42
61 43
62 44
63 45
64 46
65 47
66 48
67 49
68 50
69 51
70 52
71 53
72 54
73 55
74 56
75 57
76 58
77 59
78 60
79 61
80 62
81 63
82 64
83 65
84 66
85 67
86 68
87 69
88 70
89 71
90 72
91 73
92 74
93 75
94 76
95 77
96 78
97 79
98 80
99 81
};
\addplot [thick, red]
table {%
0 -0.332733005285263
1 -0.679397106170654
2 1.12745547294617
3 0.353993088006973
4 1.00569844245911
5 1.36709582805634
6 1.12408649921417
7 1.59927999973297
8 1.79631793498993
9 0.993868947029114
10 3.24978303909302
11 2.83456897735596
12 4.06464385986328
13 4.84128427505493
14 4.90220022201538
15 5.43416404724121
16 5.53436136245728
17 6.63312721252441
18 6.85744190216064
19 7.80787420272827
20 9.66762351989746
21 10.0295734405518
22 10.5832223892212
23 11.2151155471802
24 11.6974000930786
25 12.3560791015625
26 13.6167211532593
27 14.0031452178955
28 14.868839263916
29 16.4642314910889
30 17.0676822662354
31 17.7539558410645
32 18.8200607299805
33 19.7920875549316
34 20.6288967132568
35 21.117094039917
36 21.987512588501
37 23.4710807800293
38 24.8132419586182
39 25.2071781158447
40 25.9330234527588
41 26.2205791473389
42 27.8027572631836
43 29.0667018890381
44 29.6543312072754
45 30.4693946838379
46 31.250373840332
47 32.2413330078125
48 33.0597801208496
49 33.9929351806641
50 35.0117950439453
51 36.6367530822754
52 37.0387001037598
53 37.7130012512207
54 37.9556312561035
55 40.0899276733398
56 40.7315444946289
57 41.6150283813477
58 42.4803886413574
59 42.9920997619629
60 44.0776519775391
61 45.5484619140625
62 45.5960922241211
63 46.6819534301758
64 47.5450325012207
65 48.6810722351074
66 49.486011505127
67 50.763557434082
68 52.7692756652832
69 52.0955581665039
70 54.4095039367676
71 54.4398422241211
72 56.2058067321777
73 55.4142837524414
74 57.1804656982422
75 58.3132209777832
76 59.6377563476562
77 60.2879257202148
78 61.1888465881348
79 61.8098983764648
80 62.888053894043
81 63.3913993835449
82 64.9843063354492
83 65.0934829711914
84 66.3161544799805
85 68.2463531494141
86 67.7268447875977
87 69.057991027832
88 69.5005264282227
89 71.5194702148438
90 71.8283920288086
91 72.3289947509766
92 73.6308059692383
93 73.7366409301758
94 76.0916595458984
95 75.9114303588867
96 76.2895965576172
97 77.1172027587891
98 78.1685943603516
99 79.8869705200195
};
\end{axis}

\end{tikzpicture}

%% file: FAutoEncoderTrainer_10000_mytikz4.tex
\begin{tikzpicture}

\begin{axis}[
xmin=-4.95, xmax=103.95, ticks=none,
ymin=-4.31871961206198, ymax=85.2356223031879,
width=\figurewidth,
height=\figureheight,
tick align=outside,
x grid style={lightgray!92.026143790849673!black},
y grid style={lightgray!92.026143790849673!black}
]
\addplot [line width=1.64pt, blue]
table {%
0 0
1 0
2 0
3 0
4 0
5 0
6 0
7 0
8 0
9 0
10 0
11 0
12 0
13 1
14 2
15 3
16 4
17 5
18 6
19 7
20 8
21 9
22 10
23 11
24 12
25 13
26 14
27 15
28 16
29 17
30 18
31 19
32 20
33 21
34 22
35 23
36 24
37 25
38 26
39 27
40 28
41 29
42 30
43 31
44 32
45 33
46 34
47 35
48 36
49 37
50 38
51 39
52 40
53 38
54 36
55 37
56 38
57 39
58 40
59 41
60 42
61 43
62 44
63 45
64 46
65 47
66 48
67 49
68 50
69 51
70 52
71 53
72 54
73 55
74 56
75 57
76 58
77 59
78 60
79 61
80 62
81 63
82 64
83 65
84 66
85 67
86 68
87 69
88 70
89 71
90 72
91 73
92 74
93 75
94 76
95 77
96 78
97 79
98 80
99 81
};
\addplot [thick, red]
table {%
0 0.251457840204239
1 -0.0909299850463867
2 -0.248067706823349
3 0.398036450147629
4 0.368852198123932
5 0.459675759077072
6 0.844348073005676
7 1.52651965618134
8 1.67094206809998
9 1.92557752132416
10 2.82246017456055
11 3.26306009292603
12 3.88077855110168
13 4.39021968841553
14 5.03241729736328
15 5.79660558700562
16 6.51027154922485
17 7.00199556350708
18 7.49726390838623
19 8.24932670593262
20 8.85189628601074
21 9.91836452484131
22 10.6119070053101
23 11.2017345428467
24 12.0909986495972
25 12.7594003677368
26 13.8357572555542
27 14.428786277771
28 15.4826354980469
29 16.0224895477295
30 16.9144554138184
31 17.7460327148438
32 18.7697620391846
33 19.6624126434326
34 20.5305690765381
35 21.2642612457275
36 22.2510070800781
37 23.0456256866455
38 23.8736820220947
39 24.8629589080811
40 25.7428913116455
41 26.6879711151123
42 27.494556427002
43 28.6087493896484
44 29.5838375091553
45 30.3365936279297
46 31.221061706543
47 32.1996231079102
48 33.1190567016602
49 34.0492630004883
50 35.0533294677734
51 35.9786262512207
52 36.5745162963867
53 37.530029296875
54 38.4168663024902
55 39.2452697753906
56 40.2124938964844
57 41.1113090515137
58 42.1528205871582
59 42.9317474365234
60 43.9811096191406
61 44.7959327697754
62 45.6234016418457
63 46.6481399536133
64 47.3850898742676
65 48.3278350830078
66 49.3630294799805
67 50.0184364318848
68 51.140625
69 51.8545074462891
70 52.8212966918945
71 53.7955169677734
72 54.7710723876953
73 55.5005187988281
74 56.3895378112793
75 57.4186477661133
76 58.2782211303711
77 59.0409927368164
78 60.0478057861328
79 60.9094505310059
80 61.8235359191895
81 62.8571014404297
82 63.9335098266602
83 64.9245452880859
84 65.8269958496094
85 66.8975982666016
86 67.90234375
87 68.8482437133789
88 69.8836975097656
89 70.8637008666992
90 71.8819580078125
91 72.9538879394531
92 73.9208297729492
93 74.8040313720703
94 76.0075454711914
95 77.0924606323242
96 78.0594482421875
97 79.0617752075195
98 79.9980773925781
99 81.1649703979492
};
\end{axis}

\end{tikzpicture}

%% file: FAutoEncoderTrainer_50000_mytikz4.tex
\begin{tikzpicture}

\begin{axis}[
xmin=-4.95, xmax=103.95, ticks=none,
ymin=-4.78822700679302, ymax=87.199374422431,
width=\figurewidth,
height=\figureheight,
tick align=outside,
x grid style={lightgray!92.026143790849673!black},
y grid style={lightgray!92.026143790849673!black}
]
\addplot [line width=1.64pt, blue]
table {%
0 0
1 0
2 0
3 0
4 0
5 0
6 0
7 0
8 0
9 0
10 0
11 0
12 0
13 1
14 2
15 3
16 4
17 5
18 6
19 7
20 8
21 9
22 10
23 11
24 12
25 13
26 14
27 15
28 16
29 17
30 18
31 19
32 20
33 21
34 22
35 23
36 24
37 25
38 26
39 27
40 28
41 29
42 30
43 31
44 32
45 33
46 34
47 35
48 36
49 37
50 38
51 39
52 40
53 38
54 36
55 37
56 38
57 39
58 40
59 41
60 42
61 43
62 44
63 45
64 46
65 47
66 48
67 49
68 50
69 51
70 52
71 53
72 54
73 55
74 56
75 57
76 58
77 59
78 60
79 61
80 62
81 63
82 64
83 65
84 66
85 67
86 68
87 69
88 70
89 71
90 72
91 73
92 74
93 75
94 76
95 77
96 78
97 79
98 80
99 81
};
\addplot [thick, red]
table {%
0 0.248849347233772
1 -0.127667531371117
2 -0.0352695845067501
3 -0.0296861100941896
4 -0.217177972197533
5 -0.373887687921524
6 -0.582915723323822
7 -0.606972396373749
8 -0.582009553909302
9 -0.299872100353241
10 -0.0944805294275284
11 0.294728219509125
12 0.746397256851196
13 1.49758839607239
14 2.1867368221283
15 2.84450745582581
16 3.69681739807129
17 4.76267433166504
18 5.95887994766235
19 6.97324466705322
20 8.02901363372803
21 9.32173728942871
22 10.2493658065796
23 11.3991756439209
24 12.5309734344482
25 13.8219432830811
26 14.8892641067505
27 15.9973201751709
28 16.9851608276367
29 18.08984375
30 19.0674800872803
31 19.8960132598877
32 20.826229095459
33 21.7114963531494
34 22.6522541046143
35 23.537712097168
36 24.3591461181641
37 25.2938213348389
38 26.0907783508301
39 26.9907684326172
40 28.0359802246094
41 28.9910163879395
42 29.7622261047363
43 30.4501800537109
44 31.3591213226318
45 32.110481262207
46 33.0809135437012
47 33.8928184509277
48 34.8238983154297
49 35.6946449279785
50 36.3256149291992
51 37.3425521850586
52 37.88818359375
53 38.621940612793
54 39.4770545959473
55 40.2383575439453
56 40.9060249328613
57 41.7398109436035
58 42.654224395752
59 43.4967765808105
60 44.2259674072266
61 45.4504356384277
62 45.9030609130859
63 46.9705810546875
64 47.7577247619629
65 48.559009552002
66 49.5593452453613
67 50.3001480102539
68 51.2706642150879
69 52.1636924743652
70 53.0974540710449
71 53.8909873962402
72 54.8242988586426
73 55.8385276794434
74 56.7694664001465
75 57.7930221557617
76 58.6734008789062
77 59.5710983276367
78 60.6713027954102
79 61.6913986206055
80 62.6907272338867
81 63.8137474060059
82 64.8122024536133
83 65.8570175170898
84 66.968864440918
85 67.9185562133789
86 68.7853317260742
87 69.8566207885742
88 70.8310165405273
89 71.9321823120117
90 73.1243133544922
91 73.9716796875
92 75.1252822875977
93 76.2546844482422
94 77.3117828369141
95 78.195426940918
96 79.5832977294922
97 80.5721588134766
98 81.9652862548828
99 83.0181198120117
};
\end{axis}

\end{tikzpicture}

%% file: smooth_approx_mytikz.tex
\begin{tikzpicture}

\begin{axis}[
xmin=-1, xmax=0.5,
ymin=-2, ymax=0.5,
axis on top,
width=\figurewidth,
height=\figureheight,
xtick={-1,-0.8,-0.6,-0.4,-0.2,0,0.2,0.4,0.6,0.8}, ticks=none,
legend entries={{$u$},{$\tilde{u}$}},
legend style={at={(0.95,0.05)},anchor=south east}
]
\addplot [thick, red]
table {%
-1 -2
-0.995 -1
-0.99 -1
-0.985 -1
-0.98 -1
-0.975 -1
-0.97 -1
-0.965 -1
-0.96 -1
-0.955 -1
-0.95 -1
-0.945 -1
-0.94 -1
-0.935 -1
-0.93 -1
-0.925 -1
-0.92 -1
-0.915 -1
-0.91 -1
-0.905 -1
-0.9 -1
-0.895 -1
-0.89 -1
-0.885 -1
-0.88 -1
-0.875 -1
-0.87 -1
-0.865 -1
-0.86 -1
-0.855 -1
-0.85 -1
-0.845 -1
-0.84 -1
-0.835 -1
-0.83 -1
-0.825 -1
-0.82 -1
-0.815 -1
-0.81 -1
-0.805 -1
-0.8 -1
-0.795 -1
-0.79 -1
-0.785 -1
-0.78 -1
-0.775 -1
-0.77 -1
-0.765 -1
-0.76 -1
-0.755 -1
-0.75 -1
-0.745 -1
-0.74 -1
-0.735 -1
-0.73 -1
-0.725 -1
-0.72 -1
-0.715 -1
-0.71 -1
-0.705 -1
-0.7 -1
-0.695 -1
-0.69 -1
-0.685 -1
-0.68 -1
-0.675 -1
-0.67 -1
-0.665 -1
-0.66 -1
-0.655 -1
-0.65 -1
-0.645 -1
-0.64 -1
-0.635 -1
-0.63 -1
-0.625 -1
-0.62 -1
-0.615 -1
-0.61 -1
-0.605 -1
-0.6 -1
-0.595 -1
-0.59 -1
-0.585 -1
-0.58 -1
-0.575 -1
-0.57 -1
-0.565 -1
-0.56 -1
-0.555 -1
-0.55 -1
-0.545 -1
-0.54 -1
-0.535 -1
-0.53 -1
-0.525 -1
-0.52 -1
-0.515 -1
-0.51 -1
-0.505 -1
-0.5 -1
-0.495 -1
-0.49 -1
-0.485 -1
-0.48 -1
-0.475 -1
-0.47 -1
-0.465 -1
-0.46 -1
-0.455 -1
-0.45 -1
-0.445 -1
-0.44 -1
-0.434999999999999 -1
-0.429999999999999 -1
-0.424999999999999 -1
-0.419999999999999 -1
-0.414999999999999 -1
-0.409999999999999 -1
-0.404999999999999 -1
-0.399999999999999 -1
-0.394999999999999 -1
-0.389999999999999 -1
-0.384999999999999 -1
-0.379999999999999 -1
-0.374999999999999 -1
-0.369999999999999 -1
-0.364999999999999 -1
-0.359999999999999 -1
-0.354999999999999 -1
-0.349999999999999 -1
-0.344999999999999 -1
-0.339999999999999 -1
-0.334999999999999 -1
-0.329999999999999 -1
-0.324999999999999 -1
-0.319999999999999 -1
-0.314999999999999 -1
-0.309999999999999 -1
-0.304999999999999 -1
-0.299999999999999 -1
-0.294999999999999 -1
-0.289999999999999 -1
-0.284999999999999 -1
-0.279999999999999 -1
-0.274999999999999 -1
-0.269999999999999 -1
-0.264999999999999 -1
-0.259999999999999 -1
-0.254999999999999 -1
-0.249999999999999 -1
-0.244999999999999 -1
-0.239999999999999 -1
-0.234999999999999 -1
-0.229999999999999 -1
-0.224999999999999 -1
-0.219999999999999 -1
-0.214999999999999 -1
-0.209999999999999 -1
-0.204999999999999 -1
-0.199999999999999 -0.2
-0.194999999999999 -0.2
-0.189999999999999 -0.2
-0.184999999999999 -0.2
-0.179999999999999 -0.2
-0.174999999999999 -0.2
-0.169999999999999 -0.2
-0.164999999999999 -0.2
-0.159999999999999 -0.2
-0.154999999999999 -0.2
-0.149999999999999 -0.2
-0.144999999999999 -0.2
-0.139999999999999 -0.2
-0.134999999999999 -0.2
-0.129999999999999 -0.2
-0.124999999999999 -0.2
-0.119999999999999 -0.2
-0.114999999999999 -0.2
-0.109999999999999 -0.2
-0.104999999999999 -0.2
-0.0999999999999992 -0.2
-0.0949999999999992 -0.2
-0.0899999999999992 -0.2
-0.0849999999999992 -0.2
-0.0799999999999992 -0.2
-0.0749999999999992 -0.2
-0.0699999999999992 -0.2
-0.0649999999999992 -0.2
-0.0599999999999992 -0.2
-0.0549999999999992 -0.2
-0.0499999999999992 -0.2
-0.0449999999999992 -0.2
-0.0399999999999991 -0.2
-0.0349999999999991 -0.2
-0.0299999999999991 -0.2
-0.0249999999999991 -0.2
-0.0199999999999991 -0.2
-0.0149999999999991 -0.2
-0.00999999999999912 -0.2
-0.00499999999999912 -0.2
8.88178419700125e-16 0
0.00500000000000078 0
0.0100000000000009 0
0.015000000000001 0
0.0200000000000009 0
0.0250000000000008 0
0.0300000000000009 0
0.035000000000001 0
0.0400000000000009 0
0.0450000000000008 0
0.0500000000000009 0
0.055000000000001 0
0.0600000000000009 0
0.0650000000000008 0
0.070000000000001 0
0.0750000000000011 0
0.080000000000001 0
0.0850000000000009 0
0.090000000000001 0
0.0950000000000011 0
0.100000000000001 0.1
0.105000000000001 0.1
0.110000000000001 0.1
0.115000000000001 0.1
0.120000000000001 0.1
0.125000000000001 0.1
0.130000000000001 0.1
0.135000000000001 0.1
0.140000000000001 0.1
0.145000000000001 0.1
0.150000000000001 0.15
0.155000000000001 0.15
0.160000000000001 0.15
0.165000000000001 0.15
0.170000000000001 0.15
0.175000000000001 0.175
0.180000000000001 0.175
0.185000000000001 0.175
0.190000000000001 0.175
0.195000000000001 0.175
0.200000000000001 0.175
0.205000000000001 0.175
0.210000000000001 0.175
0.215000000000001 0.175
0.220000000000001 0.175
0.225000000000001 0.175
0.230000000000001 0.175
0.235000000000001 0.175
0.240000000000001 0.175
0.245000000000001 0.175
0.250000000000001 0.175
0.255000000000001 0.175
0.260000000000001 0.175
0.265000000000001 0.175
0.270000000000001 0.175
0.275000000000001 0.175
0.280000000000001 0.175
0.285000000000001 0.175
0.290000000000001 0.175
0.295000000000001 0.175
0.300000000000001 0.175
0.305000000000001 0.175
0.310000000000001 0.175
0.315000000000001 0.175
0.320000000000001 0.175
0.325000000000001 0.175
0.330000000000001 0.175
0.335000000000001 0.175
0.340000000000001 0.175
0.345000000000001 0.175
0.350000000000001 0.175
0.355000000000001 0.175
0.360000000000001 0.175
0.365000000000001 0.175
0.370000000000001 0.175
0.375000000000001 0.175
0.380000000000001 0.175
0.385000000000001 0.175
0.390000000000001 0.175
0.395000000000001 0.175
0.400000000000001 0.175
0.405000000000001 0.175
0.410000000000001 0.175
0.415000000000001 0.175
0.420000000000001 0.175
0.425000000000001 0.175
0.430000000000001 0.175
0.435000000000001 0.175
0.440000000000001 0.175
0.445000000000001 0.175
0.450000000000001 0.175
0.455000000000001 0.175
0.460000000000001 0.175
0.465000000000001 0.175
0.470000000000001 0.175
0.475000000000001 0.175
0.480000000000001 0.175
0.485000000000001 0.175
0.490000000000001 0.175
0.495000000000001 0.175
};
\addplot [thick, green!50.0!black]
table {%
-1 -1.5
-0.995 -1.4378234991142
-0.99 -1.37754066879815
-0.985 -1.32082130082461
-0.98 -1.26894142136999
-0.975 -1.22270013882531
-0.97 -1.18242552380636
-0.965 -1.14804719803169
-0.96 -1.11920292202212
-0.955 -1.09534946489911
-0.95 -1.07585818002124
-0.945 -1.06008665017401
-0.94 -1.04742587317757
-0.935 -1.03732688734413
-0.93 -1.02931223075136
-0.925 -1.02297736991003
-0.92 -1.01798620996209
-0.915 -1.01406362704325
-0.91 -1.01098694263059
-0.905 -1.00857748541371
-0.9 -1.00669285092428
-0.895 -1.00522012569356
-0.89 -1.0040701377159
-0.885 -1.00317268284249
-0.88 -1.00247262315663
-0.875 -1.00192673466333
-0.87 -1.00150118225674
-0.865 -1.00116951026506
-0.86 -1.0009110511944
-0.855 -1.0007096703991
-0.85 -1.00055277863692
-0.845 -1.00043055708132
-0.84 -1.00033535013047
-0.835 -1.0002611903191
-0.83 -1.00020342697806
-0.825 -1.0001584362191
-0.82 -1.00012339457599
-0.815 -1.0000961024155
-0.81 -1.00007484622751
-0.805 -1.00005829126566
-0.8 -1.0000453978687
-0.795 -1.00003535625074
-0.79 -1.00002753569111
-0.785 -1.00002144494842
-0.78 -1.00001670142185
-0.775 -1.00001300712847
-0.77 -1.00001012999098
-0.765 -1.00000788926259
-0.76 -1.0000061441746
-0.755 -1.00000478509449
-0.75 -1.00000372663928
-0.745 -1.00000290231198
-0.74 -1.0000022603243
-0.735 -1.00000176034321
-0.73 -1.0000013709572
-0.725 -1.00000106770287
-0.72 -1.00000083152802
-0.715 -1.00000064759479
-0.71 -1.0000005043474
-0.705 -1.00000039278618
-0.7 -1.00000030590221
-0.695 -1.00000023823688
-0.69 -1.00000018553906
-0.685 -1.00000014449795
-0.68 -1.00000011253509
-0.675 -1.00000008764237
-0.67 -1.00000006825589
-0.665 -1.00000005315766
-0.66 -1.00000004139911
-0.655 -1.00000003224151
-0.65 -1.0000000251095
-0.645 -1.00000001955501
-0.64 -1.00000001522907
-0.635 -1.00000001185987
-0.63 -1.00000000923575
-0.625 -1.0000000071918
-0.62 -1.00000000559961
-0.615 -1.00000000435911
-0.61 -1.00000000339232
-0.605 -1.00000000263845
-0.6 -1.00000000205004
-0.595 -1.00000000159004
-0.59 -1.0000000012294
-0.585 -1.00000000094525
-0.58 -1.00000000071948
-0.575 -1.00000000053753
-0.57 -1.00000000038746
-0.565 -1.00000000025915
-0.56 -1.0000000001436
-0.555 -1.00000000003224
-0.55 -0.999999999916319
-0.545 -0.999999999786132
-0.54 -0.999999999630194
-0.535 -0.999999999434192
-0.53 -0.999999999179637
-0.525 -0.999999998842096
-0.52 -0.999999998388828
-0.515 -0.99999999777559
-0.51 -0.999999996942304
-0.505 -0.99999999580716
-0.5 -0.999999994258582
-0.495 -0.999999992144245
-0.49 -0.999999989256096
-0.485 -0.999999985309871
-0.48 -0.999999979917116
-0.475 -0.999999972546949
-0.47 -0.999999962473788
-0.465 -0.999999948705929
-0.46 -0.9999999298879
-0.455 -0.999999904167019
-0.45 -0.999999869010991
-0.445 -0.999999820958593
-0.44 -0.99999975527891
-0.434999999999999 -0.99999966550556
-0.429999999999999 -0.999999542800062
-0.424999999999999 -0.999999375081662
-0.419999999999999 -0.999999145837958
-0.414999999999999 -0.999998832499221
-0.409999999999999 -0.999998404216377
-0.404999999999999 -0.999997818823904
-0.399999999999999 -0.999997018688666
-0.394999999999999 -0.99999592503607
-0.389999999999999 -0.999994430195057
-0.384999999999999 -0.999992386998689
-0.379999999999999 -0.999989594297261
-0.374999999999999 -0.999985777158536
-0.369999999999999 -0.999980559807426
-0.364999999999999 -0.999973428644058
-0.359999999999999 -0.999963681705051
-0.354999999999999 -0.999950359603191
-0.349999999999999 -0.999932151166629
-0.344999999999999 -0.999907264525457
-0.339999999999999 -0.999873251024723
-0.334999999999999 -0.999826764760087
-0.329999999999999 -0.999763234309192
-0.324999999999999 -0.999676414801729
-0.319999999999999 -0.999557777090463
-0.314999999999999 -0.999395675493818
-0.309999999999999 -0.999174215170597
-0.304999999999999 -0.998871713202706
-0.299999999999999 -0.998458612269339
-0.294999999999999 -0.997894660820093
-0.289999999999999 -0.997125117934864
-0.284999999999999 -0.996075675435422
-0.279999999999999 -0.994645719260572
-0.274999999999999 -0.992699490088656
-0.269999999999999 -0.990054679317451
-0.264999999999999 -0.986468069386618
-0.259999999999999 -0.981618104071979
-0.254999999999999 -0.975084935572176
-0.249999999999999 -0.966329817667504
-0.244999999999999 -0.954678059753619
-0.239999999999999 -0.939313455983003
-0.234999999999999 -0.919297101815952
-0.229999999999999 -0.893628607841733
-0.224999999999999 -0.861369435256534
-0.219999999999999 -0.821839888939747
-0.214999999999999 -0.774875514056678
-0.209999999999999 -0.721083891732835
-0.204999999999999 -0.661996292148641
-0.199999999999999 -0.599999999999991
-0.194999999999999 -0.538003707851342
-0.189999999999999 -0.478916108267148
-0.184999999999999 -0.425124485943308
-0.179999999999999 -0.378160111060241
-0.174999999999999 -0.338630564743456
-0.169999999999999 -0.306371392158259
-0.164999999999999 -0.280702898184042
-0.159999999999999 -0.260686544016992
-0.154999999999999 -0.245321940246377
-0.149999999999999 -0.233670182332493
-0.144999999999999 -0.224915064427821
-0.139999999999999 -0.21838189592802
-0.134999999999999 -0.21353193061338
-0.129999999999999 -0.209945320682547
-0.124999999999999 -0.207300509911339
-0.119999999999999 -0.205354280739409
-0.114999999999999 -0.203924324564512
-0.109999999999999 -0.202874882064908
-0.104999999999999 -0.202105339179112
-0.0999999999999992 -0.201541387727884
-0.0949999999999992 -0.2011282867876
-0.0899999999999992 -0.200825784795566
-0.0849999999999992 -0.200604324388077
-0.0799999999999992 -0.200442222497308
-0.0749999999999992 -0.200323583759447
-0.0699999999999992 -0.200236760668812
-0.0649999999999992 -0.200173217711421
-0.0599999999999992 -0.200126687794837
-0.0549999999999992 -0.200092521933975
-0.0499999999999992 -0.200067103505522
-0.0449999999999992 -0.200047038971125
-0.0399999999999991 -0.200027238721221
-0.0349999999999991 -0.199994884112138
-0.0299999999999991 -0.19990888446521
-0.0249999999999991 -0.199628875908825
-0.0199999999999991 -0.198671835517916
-0.0149999999999991 -0.195412139019349
-0.00999999999999912 -0.184833933800748
-0.00499999999999912 -0.155464047198933
8.88178419700125e-16 -0.100002981311416
0.00500000000000078 -0.0445422089412715
0.0100000000000009 -0.0151732317880227
0.015000000000001 -0.00459664148298083
0.0200000000000009 -0.00133942434715276
0.0250000000000008 -0.00038597185132982
0.0300000000000009 -0.000111012927742175
0.035000000000001 -3.20217387981915e-05
0.0400000000000009 -9.32429551269865e-06
0.0450000000000008 -2.78046787412907e-06
0.0500000000000009 -8.76316617032495e-07
0.055000000000001 -3.09358099485396e-07
0.0600000000000009 -1.31088309609132e-07
0.0650000000000008 -6.63139802618709e-08
0.070000000000001 -1.19610857841689e-08
0.0750000000000011 3.43768071433898e-07
0.080000000000001 4.51928664672687e-06
0.0850000000000009 5.52630488969652e-05
0.090000000000001 0.000669274306263095
0.0950000000000011 0.00758581012586227
0.100000000000001 0.0499999942419281
0.105000000000001 0.0924141777864108
0.110000000000001 0.0993307118267504
0.115000000000001 0.0999447198824322
0.120000000000001 0.0999954585641928
0.125000000000001 0.0999996261303829
0.130000000000001 0.0999999686302287
0.135000000000001 0.100000012138384
0.140000000000001 0.100002269214895
0.145000000000001 0.100334642183662
0.150000000000001 0.124999999745754
0.155000000000001 0.149665357268668
0.160000000000001 0.149997729971206
0.165000000000001 0.149999984657391
0.170000000000001 0.150001134771211
0.175000000000001 0.162499999946316
0.180000000000001 0.174998865014499
0.185000000000001 0.1749999999201
0.190000000000001 0.174999999979241
0.195000000000001 0.174999999984814
0.200000000000001 0.17499999998889
0.205000000000001 0.174999999991872
0.210000000000001 0.174999999994053
0.215000000000001 0.174999999995649
0.220000000000001 0.174999999996817
0.225000000000001 0.174999999997671
0.230000000000001 0.174999999998296
0.235000000000001 0.174999999998754
0.240000000000001 0.174999999999088
0.245000000000001 0.174999999999333
0.250000000000001 0.174999999999512
0.255000000000001 0.174999999999643
0.260000000000001 0.174999999999739
0.265000000000001 0.174999999999809
0.270000000000001 0.17499999999986
0.275000000000001 0.174999999999898
0.280000000000001 0.174999999999925
0.285000000000001 0.174999999999945
0.290000000000001 0.17499999999996
0.295000000000001 0.174999999999971
0.300000000000001 0.174999999999979
0.305000000000001 0.174999999999984
0.310000000000001 0.174999999999988
0.315000000000001 0.174999999999992
0.320000000000001 0.174999999999994
0.325000000000001 0.174999999999996
0.330000000000001 0.174999999999997
0.335000000000001 0.174999999999998
0.340000000000001 0.174999999999998
0.345000000000001 0.174999999999999
0.350000000000001 0.174999999999999
0.355000000000001 0.174999999999999
0.360000000000001 0.174999999999999
0.365000000000001 0.175
0.370000000000001 0.175
0.375000000000001 0.175
0.380000000000001 0.175
0.385000000000001 0.175
0.390000000000001 0.175
0.395000000000001 0.175
0.400000000000001 0.175
0.405000000000001 0.175
0.410000000000001 0.175
0.415000000000001 0.175
0.420000000000001 0.175
0.425000000000001 0.175000000000001
0.430000000000001 0.175000000000001
0.435000000000001 0.175000000000001
0.440000000000001 0.175000000000002
0.445000000000001 0.175000000000002
0.450000000000001 0.175000000000003
0.455000000000001 0.175000000000004
0.460000000000001 0.175000000000005
0.465000000000001 0.175000000000007
0.470000000000001 0.175000000000009
0.475000000000001 0.175000000000013
0.480000000000001 0.175000000000017
0.485000000000001 0.175000000000023
0.490000000000001 0.175000000000031
0.495000000000001 0.175000000000042
};
\path [draw=black, fill opacity=0] (axis cs:-1,1)
--(axis cs:0.6,1);

\path [draw=black, fill opacity=0] (axis cs:1,-2)
--(axis cs:1,0.5);



\end{axis}

\end{tikzpicture}

%% file: MCFloorLearn_mytikz_zoom.tex
\begin{tikzpicture}

\begin{axis}[
ticks=none,
xmin=0.7, xmax=0.75,
ymin=0.5, ymax=0.75,
width=\figurewidth,
height=\figureheight,
tick align=outside,
x grid style={lightgray!92.026143790849673!black},
y grid style={lightgray!92.026143790849673!black}
]
\addplot [semithick, green!50.0!black, dashed]
table {%
0.700018738680185 0.526
0.700109326324722 0.526
0.700129286269478 0.526
0.700200435987133 0.526
0.700306617954912 0.526
0.70067236728541 0.526
0.700889092193077 0.526
0.700948269085545 0.526
0.701208678532549 0.526
0.701289400102298 0.526
0.701303691678094 0.526
0.701526379613664 0.526
0.701779244892899 0.526
0.701846951140821 0.526
0.701875335593248 0.526
0.702004494640593 0.526
0.702011323320721 0.526
0.702123192038138 0.526
0.702189486295755 0.526
0.702447345258896 0.526
0.702693228287683 0.526
0.702740980351143 0.526
0.702770014859533 0.526
0.702784104569846 0.526
0.702786737307584 0.526
0.70298643292746 0.526
0.70309787931732 0.526
0.703219468883896 0.526
0.703392784810488 0.526
0.703426211773576 0.526
0.70344633387069 0.526
0.703498803909933 0.526
0.703720677748354 0.526
0.703747250619513 0.526
0.703807994494126 0.526
0.703969325008777 0.526
0.704396153990287 0.526
0.704408039041356 0.526
0.704445721758798 0.526
0.704453760382534 0.526
0.704506293487625 0.526
0.704523937583007 0.526
0.704592725112947 0.526
0.704614126652856 0.526
0.704753328034311 0.526
0.704851624618037 0.526
0.705146513564844 0.526
0.705173253497984 0.526
0.705248901089504 0.526
0.705381641332462 0.526
0.705397680741053 0.526
0.705582184642234 0.526
0.7056021254698 0.526
0.705783739684041 0.526
0.705821250909859 0.526
0.705876563823493 0.526
0.705963115635759 0.526
0.706172387173868 0.526
0.706214463697457 0.526
0.706290607530498 0.526
0.70630392813097 0.526
0.706450377156173 0.526
0.706471093376166 0.526
0.706508532480416 0.526
0.706519602148029 0.526
0.706622307811158 0.526
0.706681497133644 0.526
0.70681474993713 0.526
0.706845424312733 0.526
0.706857320653305 0.526
0.706943912011699 0.526
0.706967706633542 0.526
0.707051464831676 0.526
0.70708144829753 0.526
0.707143161026768 0.526
0.707190980616663 0.526
0.707250365880587 0.526
0.70726928878465 0.526
0.707781495488526 0.526
0.707783267390447 0.526
0.70786927976994 0.526
0.707885106618437 0.526
0.707931891233296 0.526
0.707993486605692 0.526
0.708021482729617 0.526
0.708145270118849 0.526
0.708363822941309 0.526
0.708511487866493 0.526
0.708541197763522 0.526
0.708675856047736 0.526
0.708886069905301 0.526
0.70891620663735 0.526
0.708955689850854 0.526
0.70903493779024 0.526
0.709117065448363 0.526
0.709309958514014 0.526
0.709315323082315 0.526
0.709412663763556 0.526
0.7094165499947 0.526
0.709474028491865 0.526
0.709516976436305 0.526
0.709563092642709 0.526
0.70973234297484 0.526
0.709764982173929 0.526
0.709911972557546 0.526
0.710310672517661 0.526
0.710390725951943 0.526
0.710604534464516 0.526
0.710747423870969 0.526
0.711220226778739 0.526
0.7113856565883 0.526
0.711617795912271 0.526
0.711852370367876 0.526
0.711936717182035 0.526
0.711937398878682 0.526
0.711955519359097 0.526
0.712131475011131 0.526
0.71214014488471 0.526
0.712281111305069 0.526
0.712422806135486 0.526
0.712575630955103 0.526
0.712668477498936 0.526
0.712687186168796 0.526
0.71272151095007 0.526
0.712765018133516 0.526
0.712793663254057 0.526
0.713082979397406 0.526
0.713102760719355 0.526
0.713331712984266 0.526
0.713546160254569 0.526
0.713581816456326 0.526
0.713605725273439 0.526
0.713808900777667 0.526
0.713865954274299 0.526
0.713879072546772 0.526
0.714070154667016 0.526
0.714075443368869 0.526
0.714177172435037 0.526
0.714209313092931 0.526
0.714310481389956 0.526
0.714331817059058 0.526
0.714350537325698 0.526
0.714433085105238 0.526
0.714453958708883 0.526
0.714523483101722 0.526
0.714620112283384 0.526
0.714832436050973 0.526
0.714991545992254 0.526
0.715072583292359 0.526
0.715101745920726 0.526
0.715213283933002 0.526
0.715482550128612 0.526
0.715533602770618 0.526
0.715625230781921 0.526
0.715774214458614 0.526
0.716043940645147 0.526
0.716076707553013 0.526
0.716172346942098 0.526
0.716205940600238 0.526
0.716218024333918 0.526
0.716261258933314 0.526
0.716466272790014 0.526
0.716630058310516 0.526
0.716703369584352 0.526
0.716710019719329 0.526
0.716723089379299 0.526
0.71675631560525 0.526
0.7170958431987 0.526
0.717112523783937 0.526
0.717332374089221 0.526
0.717339758386796 0.526
0.71751752331499 0.526
0.717528635068759 0.526
0.717555637527251 0.526
0.717681187499755 0.526
0.717728232913352 0.526
0.717940117026583 0.526
0.718069501584567 0.526
0.718212958572165 0.526
0.718225259103357 0.526
0.718300937809132 0.526
0.718393482160605 0.526
0.718846421727082 0.526
0.718877672045415 0.526
0.718889243508392 0.526
0.718899710975691 0.526
0.718949465533547 0.526
0.7193909965342 0.526
0.719524761833461 0.526
0.719575084657523 0.526
0.719701110456544 0.526
0.71977720683359 0.526
0.719838634999087 0.526
0.719941309034424 0.526
0.719993049873068 0.526
0.72004390278201 0.526
0.720065714221672 0.526
0.720264398554235 0.526
0.720337216718017 0.526
0.720338608805637 0.526
0.720403435628979 0.526
0.720607894002401 0.526
0.720950799331504 0.526
0.720990849097543 0.526
0.721030045782292 0.526
0.72116521204678 0.526
0.721233579451054 0.526
0.721336355494479 0.526
0.721404271079401 0.526
0.721540523546887 0.526
0.721605267982398 0.526
0.721852046257298 0.526
0.721917270252372 0.526
0.721970253603409 0.526
0.722087733723751 0.526
0.722218462965659 0.526
0.722229039660847 0.526
0.722326490443359 0.526
0.722336517128736 0.526
0.722405405558923 0.526
0.722453918489504 0.526
0.72251168577818 0.526
0.722571618440819 0.526
0.722958192546589 0.526
0.723002992092244 0.526
0.723479786687148 0.526
0.723537845201997 0.526
0.723617879418569 0.526
0.723647600356024 0.526
0.723707966650325 0.526
0.723773963446852 0.526
0.72403238508435 0.526
0.724114744437118 0.526
0.724349672939074 0.526
0.72439093344524 0.526
0.724412130207064 0.526
0.724444276535249 0.526
0.724474989913772 0.526
0.724577096281228 0.526
0.724596882791598 0.526
0.724766539829746 0.526
0.724778908818875 0.526
0.724789821164829 0.526
0.724802374361942 0.526
0.724806882534162 0.526
0.724867384276993 0.526
0.724872671138503 0.526
0.724901725553146 0.526
0.724971082275676 0.526
0.725005020433277 0.526
0.725015021978523 0.526
0.725036524419049 0.526
0.725072694411856 0.526
0.725250734589415 0.526
0.725556343542283 0.526
0.725628700823217 0.526
0.725778796387302 0.526
0.725818077993944 0.526
0.725841829941794 0.526
0.725871311825434 0.526
0.726085355257525 0.726
0.726183173957392 0.726
0.726255233283425 0.726
0.726370642747356 0.726
0.726450333789027 0.726
0.726782720542587 0.726
0.726843392081424 0.726
0.726857179152117 0.726
0.726912179807117 0.726
0.726963987415945 0.726
0.726988507998218 0.726
0.727011421629178 0.726
0.727115761471054 0.726
0.727131981461165 0.726
0.727133495027311 0.726
0.72788898137549 0.726
0.727951485265899 0.726
0.728082857848806 0.726
0.728098508995082 0.726
0.728352717181714 0.726
0.728653138881452 0.726
0.728730779273313 0.726
0.728769259024859 0.726
0.728785780271076 0.726
0.728916879793844 0.726
0.729003732218181 0.726
0.729093885612174 0.726
0.729208559570778 0.726
0.729255699891131 0.726
0.729386704782948 0.726
0.729402733255089 0.726
0.729425595452539 0.726
0.729447650277794 0.726
0.729562134274214 0.726
0.729591923720689 0.726
0.729721023506655 0.726
0.729742756850387 0.726
0.729772769313805 0.726
0.729832943387625 0.726
0.73018336183243 0.726
0.730339360822525 0.726
0.73043672875332 0.726
0.730440067627171 0.726
0.730594118599619 0.726
0.730618505754852 0.726
0.730637180177385 0.726
0.730638030592825 0.726
0.730796952930311 0.726
0.730871623242333 0.726
0.730978890057537 0.726
0.730991870775414 0.726
0.731005710069951 0.726
0.731059430154277 0.726
0.731131690180492 0.726
0.731213807972336 0.726
0.731273486373407 0.726
0.73138350254206 0.726
0.731462296640302 0.726
0.731769352413206 0.726
0.731967981801915 0.726
0.732014709983451 0.726
0.732068363465578 0.726
0.73231588389607 0.726
0.732383972758096 0.726
0.732398330673553 0.726
0.732414350864087 0.726
0.73287110053677 0.726
0.732885149664417 0.726
0.732932738763519 0.726
0.732952032647376 0.726
0.732953718059689 0.726
0.733064354764697 0.726
0.733193887005599 0.726
0.733309721427943 0.726
0.733396122293719 0.726
0.733410198641993 0.726
0.733439813577895 0.726
0.733712001955802 0.726
0.733747433610544 0.726
0.733827562686916 0.726
0.733905974838675 0.726
0.734008904658822 0.726
0.734060661796446 0.726
0.734192564393273 0.726
0.734200846275334 0.726
0.734384860699864 0.726
0.734538704657621 0.726
0.734814409698712 0.726
0.734835633976415 0.726
0.734949813467375 0.726
0.734984388413421 0.726
0.735104823102998 0.726
0.73518057488664 0.726
0.735255689507655 0.726
0.735387576547474 0.726
0.735401883519802 0.726
0.73558445330268 0.726
0.735709221664539 0.726
0.735882785904959 0.726
0.73603760380574 0.726
0.736100944209039 0.726
0.736260399487217 0.726
0.736520616088856 0.726
0.7367364426431 0.726
0.736780776740629 0.726
0.736875287951465 0.726
0.736891211796616 0.726
0.73693679415749 0.726
0.736973926449652 0.726
0.737237764120328 0.726
0.737288025811261 0.726
0.737361290232984 0.726
0.737454770228814 0.726
0.737622693061396 0.726
0.737713421994435 0.726
0.737749989872451 0.726
0.73777602990538 0.726
0.737873877272287 0.726
0.737968260629075 0.726
0.737993299676236 0.726
0.7379949019441 0.726
0.737997054269477 0.726
0.738218461681974 0.726
0.738333013716261 0.726
0.738364772427168 0.726
0.738418027264246 0.726
0.738538399933313 0.726
0.738661202437857 0.726
0.738753675836635 0.726
0.738963268412662 0.726
0.739055605451304 0.726
0.739124556801841 0.726
0.739129893938375 0.726
0.739182091196261 0.726
0.739190993573864 0.726
0.739420996589787 0.726
0.739601752737426 0.726
0.73962447016179 0.726
0.739787032032757 0.726
0.739968547305914 0.726
0.740057311674367 0.726
0.740147906591382 0.726
0.740155354525022 0.726
0.740201582269635 0.726
0.740330669727541 0.726
0.740338797600883 0.726
0.740360720649962 0.726
0.740557050704798 0.726
0.740568940825055 0.726
0.740654679847635 0.726
0.740689507634259 0.726
0.740691915223765 0.726
0.740750365461614 0.726
0.740809855072583 0.726
0.74095360260542 0.726
0.740999982411933 0.726
0.741004240484666 0.726
0.741090952511759 0.726
0.74120914067018 0.726
0.741261763860328 0.726
0.741415506031747 0.726
0.74155982561946 0.726
0.741629309679663 0.726
0.741637191823399 0.726
0.741668233269725 0.726
0.741802004285288 0.726
0.741813730841581 0.726
0.741858049559261 0.726
0.741865322511058 0.726
0.741899855516009 0.726
0.741927425767606 0.726
0.741992992069035 0.726
0.742078142944192 0.726
0.742265633133795 0.726
0.742274398785216 0.726
0.742280471951369 0.726
0.742323645706642 0.726
0.742440179244369 0.726
0.74245968135987 0.726
0.742529736291492 0.726
0.742678306760143 0.726
0.742683446369472 0.726
0.742689593143074 0.726
0.742720863662519 0.726
0.742761694167697 0.726
0.742847951815035 0.726
0.743010229568797 0.726
0.743051783147418 0.726
0.743160272048357 0.726
0.743198767527792 0.726
0.743200412009747 0.726
0.743260101329166 0.726
0.743421686612366 0.726
0.743560870354669 0.726
0.743858085250271 0.726
0.743863073255116 0.726
0.744023282837972 0.726
0.744064802935758 0.726
0.744331429871595 0.726
0.744350513201487 0.726
0.744422335527672 0.726
0.744671192811157 0.726
0.744747342895195 0.726
0.744847314626551 0.726
0.744849183836518 0.726
0.744918093193621 0.726
0.744972972905526 0.726
0.74513909221194 0.726
0.745182008617134 0.726
0.745190938384524 0.726
0.745212452762787 0.726
0.745333103874835 0.726
0.745450730594285 0.726
0.745509432984084 0.726
0.745583747895796 0.726
0.745736668631077 0.726
0.745783187347113 0.726
0.745795603160978 0.726
0.745989701229188 0.726
0.746347900800356 0.726
0.746536803746527 0.726
0.746557101835762 0.726
0.746768727177215 0.726
0.7467798425165 0.726
0.746854571196776 0.726
0.747063297907869 0.726
0.747153559714695 0.726
0.747268635305241 0.726
0.747270862359391 0.726
0.747436632420643 0.726
0.747523475468249 0.726
0.747597278000132 0.726
0.747623107441743 0.726
0.747733169513638 0.726
0.747745323621729 0.726
0.748180547395889 0.726
0.748190448396205 0.726
0.74827659002498 0.726
0.748288804385427 0.726
0.74854711049147 0.726
0.748657863168618 0.726
0.748698735749917 0.726
0.748919048637677 0.726
0.749177131334125 0.726
0.749383419232174 0.726
0.749502848680483 0.726
0.749637592350195 0.726
0.749698882976582 0.726
0.749713381357251 0.726
0.749815040555141 0.726
0.74992660226137 0.726
};
\addplot [semithick, red]
table {%
0.700018738680185 0.526
0.700109326324722 0.526
0.700129286269478 0.526
0.700200435987133 0.526
0.700306617954912 0.526
0.70067236728541 0.526
0.700889092193077 0.526
0.700948269085545 0.526
0.701208678532549 0.526
0.701289400102298 0.526
0.701303691678094 0.526
0.701526379613664 0.526
0.701779244892899 0.526
0.701846951140821 0.526
0.701875335593248 0.526
0.702004494640593 0.526
0.702011323320721 0.526
0.702123192038138 0.526
0.702189486295755 0.526
0.702447345258896 0.526
0.702693228287683 0.526
0.702740980351143 0.526
0.702770014859533 0.526
0.702784104569846 0.526
0.702786737307584 0.526
0.70298643292746 0.526
0.70309787931732 0.526
0.703219468883896 0.526
0.703392784810488 0.526
0.703426211773576 0.526
0.70344633387069 0.526
0.703498803909933 0.526
0.703720677748354 0.526
0.703747250619513 0.526
0.703807994494126 0.526
0.703969325008777 0.526
0.704396153990287 0.526
0.704408039041356 0.526
0.704445721758798 0.526
0.704453760382534 0.526
0.704506293487625 0.526
0.704523937583007 0.526
0.704592725112947 0.526
0.704614126652856 0.526
0.704753328034311 0.526
0.704851624618037 0.526
0.705146513564844 0.526
0.705173253497984 0.526
0.705248901089504 0.526
0.705381641332462 0.526
0.705397680741053 0.526
0.705582184642234 0.526
0.7056021254698 0.526
0.705783739684041 0.526
0.705821250909859 0.526
0.705876563823493 0.526
0.705963115635759 0.526
0.706172387173868 0.526
0.706214463697457 0.526
0.706290607530498 0.526
0.70630392813097 0.526
0.706450377156173 0.526
0.706471093376166 0.526
0.706508532480416 0.526
0.706519602148029 0.526
0.706622307811158 0.526
0.706681497133644 0.526
0.70681474993713 0.526
0.706845424312733 0.526
0.706857320653305 0.526
0.706943912011699 0.526
0.706967706633542 0.526
0.707051464831676 0.526
0.70708144829753 0.526
0.707143161026768 0.526
0.707190980616663 0.526
0.707250365880587 0.526
0.70726928878465 0.526
0.707781495488526 0.526
0.707783267390447 0.526
0.70786927976994 0.526
0.707885106618437 0.526
0.707931891233296 0.526
0.707993486605692 0.526
0.708021482729617 0.526
0.708145270118849 0.526
0.708363822941309 0.526
0.708511487866493 0.526
0.708541197763522 0.526
0.708675856047736 0.526
0.708886069905301 0.526
0.70891620663735 0.526
0.708955689850854 0.526
0.70903493779024 0.526
0.709117065448363 0.526
0.709309958514014 0.526
0.709315323082315 0.526
0.709412663763556 0.526
0.7094165499947 0.526
0.709474028491865 0.526
0.709516976436305 0.526
0.709563092642709 0.526
0.70973234297484 0.526
0.709764982173929 0.526
0.709911972557546 0.526
0.710310672517661 0.526
0.710390725951943 0.526
0.710604534464516 0.526
0.710747423870969 0.526
0.711220226778739 0.526
0.7113856565883 0.526
0.711617795912271 0.526
0.711852370367876 0.526
0.711936717182035 0.526
0.711937398878682 0.526
0.711955519359097 0.526
0.712131475011131 0.526
0.71214014488471 0.526
0.712281111305069 0.526
0.712422806135486 0.526
0.712575630955103 0.526
0.712668477498936 0.526
0.712687186168796 0.526
0.71272151095007 0.526
0.712765018133516 0.526
0.712793663254057 0.526
0.713082979397406 0.526
0.713102760719355 0.526
0.713331712984266 0.526
0.713546160254569 0.526
0.713581816456326 0.526
0.713605725273439 0.526
0.713808900777667 0.526
0.713865954274299 0.526
0.713879072546772 0.526
0.714070154667016 0.526
0.714075443368869 0.526
0.714177172435037 0.526
0.714209313092931 0.526
0.714310481389956 0.526
0.714331817059058 0.526
0.714350537325698 0.526
0.714433085105238 0.526
0.714453958708883 0.526
0.714523483101722 0.526
0.714620112283384 0.526
0.714832436050973 0.526
0.714991545992254 0.526
0.715072583292359 0.526
0.715101745920726 0.526
0.715213283933002 0.526
0.715482550128612 0.526
0.715533602770618 0.526
0.715625230781921 0.526
0.715774214458614 0.526
0.716043940645147 0.526
0.716076707553013 0.526
0.716172346942098 0.526
0.716205940600238 0.526
0.716218024333918 0.526
0.716261258933314 0.526
0.716466272790014 0.526
0.716630058310516 0.526
0.716703369584352 0.526
0.716710019719329 0.526
0.716723089379299 0.526
0.71675631560525 0.526
0.7170958431987 0.526
0.717112523783937 0.526
0.717332374089221 0.526
0.717339758386796 0.526
0.71751752331499 0.526
0.717528635068759 0.526
0.717555637527251 0.526
0.717681187499755 0.526
0.717728232913352 0.526
0.717940117026583 0.526
0.718069501584567 0.526
0.718212958572165 0.526
0.718225259103357 0.526
0.718300937809132 0.526
0.718393482160605 0.526
0.718846421727082 0.526
0.718877672045415 0.526
0.718889243508392 0.526
0.718899710975691 0.526
0.718949465533547 0.526
0.7193909965342 0.526
0.719524761833461 0.526
0.719575084657523 0.526
0.719701110456544 0.526
0.71977720683359 0.526
0.719838634999087 0.526
0.719941309034424 0.526
0.719993049873068 0.526
0.72004390278201 0.526
0.720065714221672 0.526
0.720264398554235 0.526
0.720337216718017 0.526
0.720338608805637 0.526
0.720403435628979 0.526
0.720607894002401 0.526
0.720950799331504 0.526
0.720990849097543 0.526
0.721030045782292 0.526
0.72116521204678 0.526
0.721233579451054 0.526
0.721336355494479 0.526
0.721404271079401 0.526
0.721540523546887 0.526
0.721605267982398 0.526
0.721852046257298 0.526
0.721917270252372 0.526
0.721970253603409 0.526
0.722087733723751 0.526
0.722218462965659 0.526
0.722229039660847 0.526
0.722326490443359 0.526
0.722336517128736 0.526
0.722405405558923 0.526
0.722453918489504 0.526
0.72251168577818 0.526
0.722571618440819 0.526
0.722958192546589 0.526
0.723002992092244 0.526
0.723479786687148 0.526
0.723537845201997 0.526
0.723617879418569 0.526
0.723647600356024 0.526
0.723707966650325 0.526
0.723773963446852 0.526
0.72403238508435 0.526
0.724114744437118 0.526
0.724349672939074 0.526
0.72439093344524 0.526
0.724412130207064 0.526
0.724444276535249 0.526
0.724474989913772 0.526
0.724577096281228 0.526
0.724596882791598 0.526
0.724766539829746 0.526
0.724778908818875 0.526
0.724789821164829 0.526
0.724802374361942 0.526
0.724806882534162 0.526
0.724867384276993 0.526
0.724872671138503 0.526
0.724901725553146 0.526
0.724971082275676 0.526
0.725005020433277 0.526
0.725015021978523 0.526
0.725036524419049 0.526
0.725072694411856 0.526
0.725250734589415 0.526
0.725556343542283 0.526
0.725628700823217 0.526
0.725778796387302 0.526
0.725818077993944 0.526
0.725841829941794 0.526
0.725871311825434 0.526
0.726085355257525 0.526
0.726183173957392 0.526
0.726255233283425 0.526
0.726370642747356 0.726
0.726450333789027 0.726
0.726782720542587 0.726
0.726843392081424 0.526
0.726857179152117 0.526
0.726912179807117 0.526
0.726963987415945 0.726
0.726988507998218 0.726
0.727011421629178 0.526
0.727115761471054 0.726
0.727131981461165 0.526
0.727133495027311 0.726
0.72788898137549 0.526
0.727951485265899 0.726
0.728082857848806 0.726
0.728098508995082 0.726
0.728352717181714 0.526
0.728653138881452 0.526
0.728730779273313 0.726
0.728769259024859 0.726
0.728785780271076 0.726
0.728916879793844 0.726
0.729003732218181 0.726
0.729093885612174 0.726
0.729208559570778 0.726
0.729255699891131 0.726
0.729386704782948 0.726
0.729402733255089 0.726
0.729425595452539 0.726
0.729447650277794 0.726
0.729562134274214 0.726
0.729591923720689 0.726
0.729721023506655 0.726
0.729742756850387 0.726
0.729772769313805 0.726
0.729832943387625 0.726
0.73018336183243 0.726
0.730339360822525 0.726
0.73043672875332 0.726
0.730440067627171 0.726
0.730594118599619 0.726
0.730618505754852 0.726
0.730637180177385 0.726
0.730638030592825 0.726
0.730796952930311 0.726
0.730871623242333 0.726
0.730978890057537 0.726
0.730991870775414 0.726
0.731005710069951 0.726
0.731059430154277 0.726
0.731131690180492 0.726
0.731213807972336 0.726
0.731273486373407 0.726
0.73138350254206 0.726
0.731462296640302 0.726
0.731769352413206 0.726
0.731967981801915 0.726
0.732014709983451 0.726
0.732068363465578 0.726
0.73231588389607 0.726
0.732383972758096 0.726
0.732398330673553 0.726
0.732414350864087 0.726
0.73287110053677 0.726
0.732885149664417 0.726
0.732932738763519 0.726
0.732952032647376 0.726
0.732953718059689 0.726
0.733064354764697 0.726
0.733193887005599 0.726
0.733309721427943 0.726
0.733396122293719 0.726
0.733410198641993 0.726
0.733439813577895 0.726
0.733712001955802 0.726
0.733747433610544 0.726
0.733827562686916 0.726
0.733905974838675 0.726
0.734008904658822 0.726
0.734060661796446 0.726
0.734192564393273 0.726
0.734200846275334 0.726
0.734384860699864 0.726
0.734538704657621 0.726
0.734814409698712 0.726
0.734835633976415 0.726
0.734949813467375 0.726
0.734984388413421 0.726
0.735104823102998 0.726
0.73518057488664 0.726
0.735255689507655 0.726
0.735387576547474 0.726
0.735401883519802 0.726
0.73558445330268 0.726
0.735709221664539 0.726
0.735882785904959 0.726
0.73603760380574 0.726
0.736100944209039 0.726
0.736260399487217 0.726
0.736520616088856 0.726
0.7367364426431 0.726
0.736780776740629 0.726
0.736875287951465 0.726
0.736891211796616 0.726
0.73693679415749 0.726
0.736973926449652 0.726
0.737237764120328 0.726
0.737288025811261 0.726
0.737361290232984 0.726
0.737454770228814 0.726
0.737622693061396 0.726
0.737713421994435 0.726
0.737749989872451 0.726
0.73777602990538 0.726
0.737873877272287 0.726
0.737968260629075 0.726
0.737993299676236 0.726
0.7379949019441 0.726
0.737997054269477 0.726
0.738218461681974 0.726
0.738333013716261 0.726
0.738364772427168 0.726
0.738418027264246 0.726
0.738538399933313 0.726
0.738661202437857 0.726
0.738753675836635 0.726
0.738963268412662 0.726
0.739055605451304 0.726
0.739124556801841 0.726
0.739129893938375 0.726
0.739182091196261 0.726
0.739190993573864 0.726
0.739420996589787 0.726
0.739601752737426 0.726
0.73962447016179 0.726
0.739787032032757 0.726
0.739968547305914 0.726
0.740057311674367 0.726
0.740147906591382 0.726
0.740155354525022 0.726
0.740201582269635 0.726
0.740330669727541 0.726
0.740338797600883 0.726
0.740360720649962 0.726
0.740557050704798 0.726
0.740568940825055 0.726
0.740654679847635 0.726
0.740689507634259 0.726
0.740691915223765 0.726
0.740750365461614 0.726
0.740809855072583 0.726
0.74095360260542 0.726
0.740999982411933 0.726
0.741004240484666 0.726
0.741090952511759 0.726
0.74120914067018 0.726
0.741261763860328 0.726
0.741415506031747 0.726
0.74155982561946 0.726
0.741629309679663 0.726
0.741637191823399 0.726
0.741668233269725 0.726
0.741802004285288 0.726
0.741813730841581 0.726
0.741858049559261 0.726
0.741865322511058 0.726
0.741899855516009 0.726
0.741927425767606 0.726
0.741992992069035 0.726
0.742078142944192 0.726
0.742265633133795 0.726
0.742274398785216 0.726
0.742280471951369 0.726
0.742323645706642 0.726
0.742440179244369 0.726
0.74245968135987 0.726
0.742529736291492 0.726
0.742678306760143 0.726
0.742683446369472 0.726
0.742689593143074 0.726
0.742720863662519 0.726
0.742761694167697 0.726
0.742847951815035 0.726
0.743010229568797 0.726
0.743051783147418 0.726
0.743160272048357 0.726
0.743198767527792 0.726
0.743200412009747 0.726
0.743260101329166 0.726
0.743421686612366 0.726
0.743560870354669 0.726
0.743858085250271 0.726
0.743863073255116 0.726
0.744023282837972 0.726
0.744064802935758 0.726
0.744331429871595 0.726
0.744350513201487 0.726
0.744422335527672 0.726
0.744671192811157 0.726
0.744747342895195 0.726
0.744847314626551 0.726
0.744849183836518 0.726
0.744918093193621 0.726
0.744972972905526 0.726
0.74513909221194 0.726
0.745182008617134 0.726
0.745190938384524 0.726
0.745212452762787 0.726
0.745333103874835 0.726
0.745450730594285 0.726
0.745509432984084 0.726
0.745583747895796 0.726
0.745736668631077 0.726
0.745783187347113 0.726
0.745795603160978 0.726
0.745989701229188 0.726
0.746347900800356 0.726
0.746536803746527 0.726
0.746557101835762 0.726
0.746768727177215 0.726
0.7467798425165 0.726
0.746854571196776 0.726
0.747063297907869 0.726
0.747153559714695 0.726
0.747268635305241 0.726
0.747270862359391 0.726
0.747436632420643 0.726
0.747523475468249 0.726
0.747597278000132 0.726
0.747623107441743 0.726
0.747733169513638 0.726
0.747745323621729 0.726
0.748180547395889 0.726
0.748190448396205 0.726
0.74827659002498 0.726
0.748288804385427 0.726
0.74854711049147 0.726
0.748657863168618 0.726
0.748698735749917 0.726
0.748919048637677 0.726
0.749177131334125 0.726
0.749383419232174 0.726
0.749502848680483 0.726
0.749637592350195 0.726
0.749698882976582 0.726
0.749713381357251 0.726
0.749815040555141 0.726
0.74992660226137 0.726
};
\end{axis}

\end{tikzpicture}

%% file: extremely_flat_mytikz.tex
\begin{tikzpicture}

\begin{axis}[
xmin=-1, xmax=1,
ymin=-1.2, ymax=1.2,
axis on top,
width=\figurewidth,
height=\figureheight
]
\addplot [thick, blue]
table {%
-1 0.95
-0.99 0.935253835744077
-0.98 0.892688852740087
-0.97 0.827163233733475
-0.96 0.74422936948324
-0.95 0.649649406498519
-0.94 0.549000861795099
-0.93 0.44738036172389
-0.92 0.349203224207781
-0.91 0.258089533162399
-0.9 0.176823186097656
-0.89 0.10736859534109
-0.88 0.0509297014892091
-0.87 0.0080371533833834
-0.86 -0.0213485643232906
-0.85 -0.0377280069604415
-0.84 -0.0419552804207074
-0.83 -0.0351376443048595
-0.82 -0.0185457570860327
-0.81 0.00646347219173199
-0.8 0.0385107096577263
-0.79 0.0762449131971942
-0.78 0.118381799397658
-0.77 0.163731328137314
-0.76 0.211215688798387
-0.75 0.259879342997506
-0.74 0.30889264431467
-0.73 0.357550450123588
-0.72 0.405266992124419
-0.71 0.451568102150467
-0.7 0.496081714244301
-0.69 0.538527394068063
-0.68 0.578705489708143
-0.67 0.616486358071728
-0.66 0.651800000262425
-0.65 0.684626337820352
-0.64 0.714986278653763
-0.63 0.742933655309568
-0.62 0.768548066983865
-0.61 0.791928618269573
-0.6 0.813188520000918
-0.59 0.83245049873056
-0.58 0.849842949596603
-0.57 0.86549676105161
-0.56 0.879542737807519
-0.55 0.892109549290528
-0.54 0.90332213399291
-0.53 0.913300494629497
-0.52 0.922158824388229
-0.51 0.930004910373454
-0.5 0.936939766255644
-0.49 0.94305745193079
-0.48 0.948445043498272
-0.47 0.953182721985827
-0.46 0.957343953926219
-0.45 0.960995741096909
-0.44 0.964198920469891
-0.429999999999999 0.967008498698826
-0.419999999999999 0.969474008320915
-0.409999999999999 0.971639875304519
-0.399999999999999 0.973545789666957
-0.389999999999999 0.975227072657914
-0.379999999999999 0.976715035489911
-0.369999999999999 0.978037325833834
-0.359999999999999 0.979218259318132
-0.349999999999999 0.980279134105309
-0.339999999999999 0.98123852729645
-0.329999999999999 0.982112572458117
-0.319999999999999 0.982915217997705
-0.309999999999999 0.983658466451898
-0.299999999999999 0.984352595014576
-0.289999999999999 0.985006357829024
-0.279999999999999 0.985627170716555
-0.269999999999999 0.986221279119335
-0.259999999999999 0.986793910107677
-0.249999999999999 0.987349409348102
-0.239999999999999 0.987891363953634
-0.229999999999999 0.988422712146721
-0.219999999999999 0.988945840661477
-0.209999999999999 0.989462670798578
-0.199999999999999 0.989974734025526
-0.189999999999999 0.990483237988802
-0.179999999999999 0.990989123774291
-0.169999999999999 0.991493115219166
-0.159999999999999 0.991995761043177
-0.149999999999999 0.992497470530474
-0.139999999999999 0.992998543455291
-0.129999999999999 0.993499194906265
-0.119999999999999 0.993999575625136
-0.109999999999999 0.994499788436279
-0.0999999999999992 0.994999901303959
-0.0899999999999992 0.99549995751459
-0.0799999999999992 0.995999983441552
-0.0699999999999992 0.99649999431037
-0.0599999999999992 0.996999998342285
-0.0499999999999992 0.997499999614469
-0.0399999999999991 0.997999999935319
-0.0299999999999991 0.998499999993525
-0.0199999999999991 0.998999999999747
-0.00999999999999912 0.999499999999999
8.88178419700125e-16 1
0.0100000000000009 1.001
0.0200000000000009 1.00199999999949
0.0300000000000009 1.00299999998705
0.0400000000000009 1.00399999987064
0.0500000000000009 1.00499999922894
0.0600000000000009 1.00599999668457
0.070000000000001 1.00699998862074
0.080000000000001 1.0079999668831
0.090000000000001 1.00899991502918
0.100000000000001 1.00999980260792
0.110000000000001 1.01099957687256
0.120000000000001 1.01199915125027
0.130000000000001 1.01299838981253
0.140000000000001 1.01399708691058
0.150000000000001 1.01499494106095
0.160000000000001 1.01599152208635
0.170000000000001 1.01698623043833
0.180000000000001 1.01797824754858
0.190000000000001 1.0189664759776
0.200000000000001 1.01994946805105
0.210000000000001 1.02092534159716
0.220000000000001 1.02189168132295
0.230000000000001 1.02284542429344
0.240000000000001 1.02378272790727
0.250000000000001 1.0246988186962
0.260000000000001 1.02558782021535
0.270000000000001 1.02644255823867
0.280000000000001 1.02725434143311
0.290000000000001 1.02801271565805
0.300000000000001 1.02870519002915
0.310000000000001 1.0293169329038
0.320000000000001 1.02983043599541
0.330000000000001 1.03022514491623
0.340000000000001 1.0304770545929
0.350000000000001 1.03055826821062
0.360000000000001 1.03043651863626
0.370000000000001 1.03007465166767
0.380000000000001 1.02943007097982
0.390000000000001 1.02845414531583
0.400000000000001 1.02709157933391
0.410000000000001 1.02527975060904
0.420000000000001 1.02294801664183
0.430000000000001 1.02001699739765
0.440000000000001 1.01639784093978
0.450000000000001 1.01199148219382
0.460000000000001 1.00668790785244
0.470000000000001 1.00036544397165
0.480000000000001 0.992890086996542
0.490000000000001 0.984114903861578
0.500000000000001 0.973879532511285
0.510000000000001 0.962009820746905
0.520000000000001 0.948317648776455
0.530000000000001 0.932600989258991
0.540000000000001 0.914644267985817
0.550000000000001 0.894219098581052
0.560000000000001 0.871085475615034
0.570000000000001 0.844993522103216
0.580000000000001 0.8156858991932
0.590000000000001 0.782900997461113
0.600000000000001 0.746377040001829
0.610000000000001 0.705857236539139
0.620000000000001 0.661096133967721
0.630000000000001 0.611867310619127
0.640000000000001 0.557972557307515
0.650000000000001 0.499252675640692
0.660000000000001 0.435600000524839
0.670000000000001 0.366972716143443
0.680000000000001 0.293410979416273
0.690000000000002 0.215054788136112
0.700000000000002 0.132163428488586
0.710000000000002 0.0451362043009173
0.720000000000002 -0.045466015751178
0.730000000000002 -0.138899099752842
0.740000000000002 -0.234214711370676
0.750000000000002 -0.330241314005005
0.760000000000002 -0.425568622403243
0.770000000000002 -0.518537343725388
0.780000000000002 -0.607236401204699
0.790000000000002 -0.689510173605625
0.800000000000002 -0.76297858068456
0.810000000000002 -0.825073055616546
0.820000000000002 -0.873091514172073
0.830000000000002 -0.904275288609723
0.840000000000002 -0.915910560841415
0.850000000000002 -0.905456013920879
0.860000000000002 -0.870697128646573
0.870000000000002 -0.80992569323322
0.880000000000002 -0.722140597021564
0.890000000000002 -0.607262809317796
0.900000000000002 -0.466353627804661
0.910000000000002 -0.301820933675171
0.920000000000002 -0.117593551584403
0.930000000000002 0.0807607234478169
0.940000000000002 0.286001723590234
0.950000000000002 0.489298812997074
0.960000000000002 0.680458738966512
0.970000000000002 0.848326467466977
0.980000000000002 0.981377705480194
0.990000000000002 1.06850767148816
};
\addplot [thick, green!50.0!black]
table {%
-1 0.8
-0.99 0.771507671488153
-0.98 0.687377705480174
-0.97 0.557326467466949
-0.96 0.39245873896648
-0.95 0.204298812997039
-0.94 0.00400172359019715
-0.93 -0.19823927655222
-0.92 -0.393593551584438
-0.91 -0.574820933675202
-0.9 -0.736353627804687
-0.89 -0.874262809317819
-0.88 -0.986140597021582
-0.87 -1.07092569323323
-0.86 -1.12869712864658
-0.85 -1.16045601392088
-0.84 -1.16791056084141
-0.83 -1.15327528860972
-0.82 -1.11909151417207
-0.81 -1.06807305561654
-0.8 -1.00297858068455
-0.79 -0.926510173605612
-0.78 -0.841236401204684
-0.77 -0.749537343725372
-0.76 -0.653568622403226
-0.75 -0.555241314004987
-0.74 -0.456214711370659
-0.73 -0.357899099752825
-0.72 -0.261466015751162
-0.71 -0.167863795699067
-0.7 -0.0778365715113989
-0.69 0.00805478813612603
-0.68 0.0894109794162863
-0.67 0.165972716143455
-0.66 0.237600000524851
-0.65 0.304252675640703
-0.64 0.365972557307525
-0.63 0.422867310619136
-0.62 0.475096133967729
-0.61 0.522857236539147
-0.6 0.566377040001836
-0.59 0.60590099746112
-0.58 0.641685899193206
-0.57 0.673993522103221
-0.56 0.703085475615039
-0.55 0.729219098581056
-0.54 0.752644267985821
-0.53 0.773600989258994
-0.52 0.792317648776458
-0.51 0.809009820746907
-0.5 0.823879532511287
-0.49 0.83711490386158
-0.48 0.848890086996543
-0.47 0.859365443971654
-0.46 0.868687907852437
-0.45 0.876991482193817
-0.44 0.884397840939782
-0.429999999999999 0.891016997397652
-0.419999999999999 0.89694801664183
-0.409999999999999 0.902279750609039
-0.399999999999999 0.907091579333914
-0.389999999999999 0.911454145315828
-0.379999999999999 0.915430070979821
-0.369999999999999 0.919074651667668
-0.359999999999999 0.922436518636264
-0.349999999999999 0.925558268210618
-0.339999999999999 0.9284770545929
-0.329999999999999 0.931225144916235
-0.319999999999999 0.933830435995409
-0.309999999999999 0.936316932903797
-0.299999999999999 0.938705190029152
-0.289999999999999 0.941012715658048
-0.279999999999999 0.94325434143311
-0.269999999999999 0.94544255823867
-0.259999999999999 0.947587820215354
-0.249999999999999 0.949698818696204
-0.239999999999999 0.951782727907268
-0.229999999999999 0.953845424293443
-0.219999999999999 0.955891681322953
-0.209999999999999 0.957925341597156
-0.199999999999999 0.959949468051052
-0.189999999999999 0.961966475977605
-0.179999999999999 0.963978247548582
-0.169999999999999 0.965986230438332
-0.159999999999999 0.967991522086354
-0.149999999999999 0.969994941060947
-0.139999999999999 0.971997086910582
-0.129999999999999 0.97399838981253
-0.119999999999999 0.975999151250272
-0.109999999999999 0.977999576872559
-0.0999999999999992 0.979999802607919
-0.0899999999999992 0.98199991502918
-0.0799999999999992 0.983999966883103
-0.0699999999999992 0.985999988620739
-0.0599999999999992 0.987999996684571
-0.0499999999999992 0.989999999228937
-0.0399999999999991 0.991999999870637
-0.0299999999999991 0.993999999987049
-0.0199999999999991 0.995999999999495
-0.00999999999999912 0.997999999999998
8.88178419700125e-16 1
0.0100000000000009 1.001
0.0200000000000009 1.00199999999975
0.0300000000000009 1.00299999999352
0.0400000000000009 1.00399999993532
0.0500000000000009 1.00499999961447
0.0600000000000009 1.00599999834229
0.070000000000001 1.00699999431037
0.080000000000001 1.00799998344155
0.090000000000001 1.00899995751459
0.100000000000001 1.00999990130396
0.110000000000001 1.01099978843628
0.120000000000001 1.01199957562514
0.130000000000001 1.01299919490626
0.140000000000001 1.01399854345529
0.150000000000001 1.01499747053047
0.160000000000001 1.01599576104318
0.170000000000001 1.01699311521917
0.180000000000001 1.01798912377429
0.190000000000001 1.0189832379888
0.200000000000001 1.01997473402553
0.210000000000001 1.02096267079858
0.220000000000001 1.02194584066148
0.230000000000001 1.02292271214672
0.240000000000001 1.02389136395363
0.250000000000001 1.0248494093481
0.260000000000001 1.02579391010768
0.270000000000001 1.02672127911933
0.280000000000001 1.02762717071656
0.290000000000001 1.02850635782902
0.300000000000001 1.02935259501458
0.310000000000001 1.0301584664519
0.320000000000001 1.0309152179977
0.330000000000001 1.03161257245812
0.340000000000001 1.03223852729645
0.350000000000001 1.03277913410531
0.360000000000001 1.03321825931813
0.370000000000001 1.03353732583383
0.380000000000001 1.03371503548991
0.390000000000001 1.03372707265791
0.400000000000001 1.03354578966696
0.410000000000001 1.03313987530452
0.420000000000001 1.03247400832091
0.430000000000001 1.03150849869883
0.440000000000001 1.03019892046989
0.450000000000001 1.02849574109691
0.460000000000001 1.02634395392622
0.470000000000001 1.02368272198583
0.480000000000001 1.02044504349827
0.490000000000001 1.01655745193079
0.500000000000001 1.01193976625564
0.510000000000001 1.00650491037345
0.520000000000001 1.00015882438823
0.530000000000001 0.992800494629496
0.540000000000001 0.984322133992908
0.550000000000001 0.974609549290526
0.560000000000001 0.963542737807517
0.570000000000001 0.950996761051608
0.580000000000001 0.9368429495966
0.590000000000001 0.920950498730557
0.600000000000001 0.903188520000914
0.610000000000001 0.88342861826957
0.620000000000001 0.861548066983861
0.630000000000001 0.837433655309564
0.640000000000001 0.810986278653758
0.650000000000001 0.782126337820346
0.660000000000001 0.750800000262419
0.670000000000001 0.716986358071721
0.680000000000001 0.680705489708136
0.690000000000002 0.642027394068056
0.700000000000002 0.601081714244293
0.710000000000002 0.558068102150459
0.720000000000002 0.513266992124411
0.730000000000002 0.467050450123579
0.740000000000002 0.419892644314662
0.750000000000002 0.372379342997498
0.760000000000002 0.325215688798379
0.770000000000002 0.279231328137306
0.780000000000002 0.23538179939765
0.790000000000002 0.194744913197187
0.800000000000002 0.15851070965772
0.810000000000002 0.127963472191727
0.820000000000002 0.104454242913964
0.830000000000002 0.0893623556951386
0.840000000000002 0.0840447195792926
0.850000000000002 0.0897719930395605
0.860000000000002 0.107651435676714
0.870000000000002 0.13853715338339
0.880000000000002 0.182929701489218
0.890000000000002 0.240868595341102
0.900000000000002 0.31182318609767
0.910000000000002 0.394589533162415
0.920000000000002 0.487203224207799
0.930000000000002 0.586880361723909
0.940000000000002 0.690000861795117
0.950000000000002 0.792149406498537
0.960000000000002 0.888229369483256
0.970000000000002 0.972663233733488
0.980000000000002 1.0396888527401
0.990000000000002 1.08375383574408
};
\end{axis}

\end{tikzpicture}

%% file: low_snr_mytikz.tex
\begin{tikzpicture}

\begin{axis}[
xmin=-1, xmax=1,
ymin=-1.2, ymax=1.2,
axis on top,
width=\figurewidth,
height=\figureheight
]
\addplot [thick, blue]
table {%
-1 -0.1
-0.99 -0.099
-0.98 -0.098
-0.97 -0.097
-0.96 -0.096
-0.95 -0.095
-0.94 -0.094
-0.93 -0.093
-0.92 -0.092
-0.91 -0.091
-0.9 -0.09
-0.89 -0.089
-0.88 -0.088
-0.87 -0.087
-0.86 -0.086
-0.85 -0.085
-0.84 -0.084
-0.83 -0.083
-0.82 -0.082
-0.81 -0.081
-0.8 -0.08
-0.79 -0.079
-0.78 -0.078
-0.77 -0.077
-0.76 -0.076
-0.75 -0.075
-0.74 -0.074
-0.73 -0.073
-0.72 -0.072
-0.71 -0.071
-0.7 -0.07
-0.69 -0.069
-0.68 -0.068
-0.67 -0.067
-0.66 -0.066
-0.65 -0.065
-0.64 -0.064
-0.63 -0.063
-0.62 -0.062
-0.61 -0.061
-0.6 -0.06
-0.59 -0.059
-0.58 -0.058
-0.57 -0.057
-0.56 -0.056
-0.55 -0.055
-0.54 -0.054
-0.53 -0.053
-0.52 -0.052
-0.51 -0.051
-0.5 -0.05
-0.49 -0.049
-0.48 -0.048
-0.47 -0.047
-0.46 -0.046
-0.45 -0.045
-0.44 -0.044
-0.429999999999999 -0.043
-0.419999999999999 -0.042
-0.409999999999999 -0.041
-0.399999999999999 -0.04
-0.389999999999999 -0.039
-0.379999999999999 -0.038
-0.369999999999999 -0.0369999999999999
-0.359999999999999 -0.0359999999999999
-0.349999999999999 -0.0349999999999999
-0.339999999999999 -0.0339999999999999
-0.329999999999999 -0.0329999999999999
-0.319999999999999 -0.0319999999999999
-0.309999999999999 -0.0309999999999999
-0.299999999999999 -0.0299999999999999
-0.289999999999999 -0.0289999999999999
-0.279999999999999 -0.0279999999999999
-0.269999999999999 -0.0269999999999999
-0.259999999999999 -0.0259999999999999
-0.249999999999999 -0.0249999999999999
-0.239999999999999 -0.0239999999999999
-0.229999999999999 -0.0229999999999999
-0.219999999999999 -0.0219999999999999
-0.209999999999999 -0.0209999999999999
-0.199999999999999 -0.0199999999999999
-0.189999999999999 -0.0189999999999999
-0.179999999999999 -0.0179999999999999
-0.169999999999999 -0.0169999999999999
-0.159999999999999 -0.0159999999999999
-0.149999999999999 -0.0149999999999999
-0.139999999999999 -0.0139999999999999
-0.129999999999999 -0.0129999999999999
-0.119999999999999 -0.0119999999999999
-0.109999999999999 -0.0109999999999999
-0.0999999999999992 -0.00999999999999992
-0.0899999999999992 -0.00899999999999992
-0.0799999999999992 -0.00799999999999992
-0.0699999999999992 -0.00699999999999992
-0.0599999999999992 -0.00599999999999992
-0.0499999999999992 -0.00499999999999992
-0.0399999999999991 -0.00399999999999992
-0.0299999999999991 -0.00299999999999991
-0.0199999999999991 -0.00199999999999991
-0.00999999999999912 -0.000999999999999912
8.88178419700125e-16 8.88178419700125e-17
0.0100000000000009 0.00100000000000009
0.0200000000000009 0.00200000000000009
0.0300000000000009 0.00300000000000009
0.0400000000000009 0.00400000000000009
0.0500000000000009 0.00500000000000009
0.0600000000000009 0.00600000000000009
0.070000000000001 0.0070000000000001
0.080000000000001 0.0080000000000001
0.090000000000001 0.0090000000000001
0.100000000000001 0.0100000000000001
0.110000000000001 0.0110000000000001
0.120000000000001 0.0120000000000001
0.130000000000001 0.0130000000000001
0.140000000000001 0.0140000000000001
0.150000000000001 0.0150000000000001
0.160000000000001 0.0160000000000001
0.170000000000001 0.0170000000000001
0.180000000000001 0.0180000000000001
0.190000000000001 0.0190000000000001
0.200000000000001 0.0200000000000001
0.210000000000001 0.0210000000000001
0.220000000000001 0.0220000000000001
0.230000000000001 0.0230000000000001
0.240000000000001 0.0240000000000001
0.250000000000001 0.0250000000000001
0.260000000000001 0.0260000000000001
0.270000000000001 0.0270000000000001
0.280000000000001 0.0280000000000001
0.290000000000001 0.0290000000000001
0.300000000000001 0.0300000000000001
0.310000000000001 0.0310000000000001
0.320000000000001 0.0320000000000001
0.330000000000001 0.0330000000000001
0.340000000000001 0.0340000000000001
0.350000000000001 0.0350000000000001
0.360000000000001 0.0360000000000001
0.370000000000001 0.0370000000000001
0.380000000000001 0.0380000000000001
0.390000000000001 0.0390000000000001
0.400000000000001 0.0400000000000001
0.410000000000001 0.0410000000000001
0.420000000000001 0.0420000000000001
0.430000000000001 0.0430000000000001
0.440000000000001 0.0440000000000001
0.450000000000001 0.0450000000000001
0.460000000000001 0.0460000000000001
0.470000000000001 0.0470000000000001
0.480000000000001 0.0480000000000001
0.490000000000001 0.0490000000000001
0.500000000000001 0.0500000000000001
0.510000000000001 0.0510000000000001
0.520000000000001 0.0520000000000001
0.530000000000001 0.0530000000000001
0.540000000000001 0.0540000000000001
0.550000000000001 0.0550000000000001
0.560000000000001 0.0560000000000001
0.570000000000001 0.0570000000000001
0.580000000000001 0.0580000000000001
0.590000000000001 0.0590000000000001
0.600000000000001 0.0600000000000001
0.610000000000001 0.0610000000000001
0.620000000000001 0.0620000000000001
0.630000000000001 0.0630000000000002
0.640000000000001 0.0640000000000002
0.650000000000001 0.0650000000000002
0.660000000000001 0.0660000000000002
0.670000000000001 0.0670000000000002
0.680000000000001 0.0680000000000002
0.690000000000002 0.0690000000000002
0.700000000000002 0.0700000000000002
0.710000000000002 0.0710000000000002
0.720000000000002 0.0720000000000002
0.730000000000002 0.0730000000000002
0.740000000000002 0.0740000000000002
0.750000000000002 0.0750000000000002
0.760000000000002 0.0760000000000002
0.770000000000002 0.0770000000000002
0.780000000000002 0.0780000000000002
0.790000000000002 0.0790000000000002
0.800000000000002 0.0800000000000002
0.810000000000002 0.0810000000000002
0.820000000000002 0.0820000000000002
0.830000000000002 0.0830000000000002
0.840000000000002 0.0840000000000002
0.850000000000002 0.0850000000000002
0.860000000000002 0.0860000000000002
0.870000000000002 0.0870000000000002
0.880000000000002 0.0880000000000002
0.890000000000002 0.0890000000000002
0.900000000000002 0.0900000000000002
0.910000000000002 0.0910000000000002
0.920000000000002 0.0920000000000002
0.930000000000002 0.0930000000000002
0.940000000000002 0.0940000000000002
0.950000000000002 0.0950000000000002
0.960000000000002 0.0960000000000002
0.970000000000002 0.0970000000000002
0.980000000000002 0.0980000000000002
0.990000000000002 0.0990000000000002
};
\addplot [blue, dashed]
table {%
-1 -0.163495432635733
-0.99 -0.142792689630223
-0.98 -0.136907540171062
-0.97 -0.140037256224275
-0.96 -0.0829063682249599
-0.95 -0.0530299177737183
-0.94 -0.0835327527036313
-0.93 -0.0689502951195425
-0.92 0.070408238503859
-0.91 -0.132882045330693
-0.9 -0.117819008180006
-0.89 -0.145482850460837
-0.88 0.00960328315350853
-0.87 -0.00863612735869317
-0.86 -0.094747839001222
-0.85 0.0379029849575601
-0.84 -0.102968341274934
-0.83 -0.0929009674410988
-0.82 -0.0467785062178711
-0.81 0.105062897303556
-0.8 0.0423223185403792
-0.79 -0.163762358023154
-0.78 -0.126791645713009
-0.77 -0.144289841679731
-0.76 -0.179179318504002
-0.75 -0.0269948927938132
-0.74 0.0471827121166135
-0.73 -0.198532965346411
-0.72 -0.109928472772186
-0.71 -0.136531207086172
-0.7 0.0152079336467311
-0.69 -0.168342888010604
-0.68 -0.135045425183117
-0.67 -0.044222705643666
-0.66 -0.18679434833695
-0.65 -0.0618068796137266
-0.64 0.0202090350437455
-0.63 -0.193446473721533
-0.62 0.0200480413501443
-0.61 -0.0172359956221149
-0.6 -0.135123062513553
-0.59 0.0746468253806445
-0.58 -0.0483492335557488
-0.57 -0.108904691171366
-0.56 0.17615143709802
-0.55 -0.193314624009629
-0.54 -0.229606423655103
-0.53 -0.123119287485153
-0.52 -0.125427469921884
-0.51 -0.0538513461205607
-0.5 -0.093083000606208
-0.49 0.0465668094741104
-0.48 -0.0165408245305214
-0.47 0.10584881165461
-0.46 -0.0502319817574888
-0.45 -0.266656091196224
-0.44 -0.0826631869205788
-0.429999999999999 -0.0417974011607298
-0.419999999999999 0.0771211946759634
-0.409999999999999 -0.0886247694458764
-0.399999999999999 -0.102998798408016
-0.389999999999999 -0.0248157597865335
-0.379999999999999 0.0742181504934019
-0.369999999999999 -0.130933202654974
-0.359999999999999 0.0438904132756874
-0.349999999999999 -0.0053806110572928
-0.339999999999999 0.168770121641573
-0.329999999999999 -0.101158249312757
-0.319999999999999 0.0441813006686269
-0.309999999999999 0.0528421994548797
-0.299999999999999 -0.0931525748938775
-0.289999999999999 0.00386296652568814
-0.279999999999999 -0.156764669360307
-0.269999999999999 -0.00613806625347245
-0.259999999999999 -0.0625414666323602
-0.249999999999999 0.157641780917915
-0.239999999999999 -0.120124468699539
-0.229999999999999 0.0256976426771125
-0.219999999999999 0.0101236536830992
-0.209999999999999 -0.0382742489361326
-0.199999999999999 -0.110576491722763
-0.189999999999999 -0.068698116519263
-0.179999999999999 -0.00692006591835396
-0.169999999999999 -0.163140134268064
-0.159999999999999 -0.0955067347311876
-0.149999999999999 -0.0538535484923061
-0.139999999999999 -0.0900844726981131
-0.129999999999999 0.0952182150597656
-0.119999999999999 0.0189286682995834
-0.109999999999999 -0.016383079058393
-0.0999999999999992 -0.102801367318168
-0.0899999999999992 0.145623586101941
-0.0799999999999992 -0.0831234394997171
-0.0699999999999992 0.0274142154801234
-0.0599999999999992 0.0890279838694751
-0.0499999999999992 -0.247090694124073
-0.0399999999999991 0.00352507068082929
-0.0299999999999991 0.0152806164494949
-0.0199999999999991 0.0523468847034905
-0.00999999999999912 0.00609140605835145
8.88178419700125e-16 -0.205841107028429
0.0100000000000009 -0.266326294704724
0.0200000000000009 0.0605806678309718
0.0300000000000009 -0.131000899404844
0.0400000000000009 0.0350455820709916
0.0500000000000009 -0.0557568097607282
0.0600000000000009 0.00110301824377428
0.070000000000001 -0.0322200120836763
0.080000000000001 0.0585929674563864
0.090000000000001 0.0188096922894047
0.100000000000001 -0.0273309766690187
0.110000000000001 -0.043952630650478
0.120000000000001 0.00227994325049956
0.130000000000001 -0.0436991321513672
0.140000000000001 0.0891563688490532
0.150000000000001 0.0901545689293056
0.160000000000001 -0.109633554417696
0.170000000000001 0.0251018228044718
0.180000000000001 0.0988699885404953
0.190000000000001 -0.00133103211954742
0.200000000000001 0.100004314389974
0.210000000000001 0.130797439246742
0.220000000000001 -0.0253515556505924
0.230000000000001 0.0763374258161948
0.240000000000001 -0.129783147791549
0.250000000000001 0.210986333534163
0.260000000000001 -0.0114596018168296
0.270000000000001 0.188828776371924
0.280000000000001 0.0591822478315656
0.290000000000001 0.133886581539953
0.300000000000001 0.0368983488732085
0.310000000000001 -0.0396520388454038
0.320000000000001 0.067496443558916
0.330000000000001 -0.0906618147687446
0.340000000000001 -0.00443708856620702
0.350000000000001 -0.132811551209031
0.360000000000001 -0.0486573965805339
0.370000000000001 0.13726030290936
0.380000000000001 -0.0959521548404375
0.390000000000001 0.0645380137792748
0.400000000000001 -0.0744811486664343
0.410000000000001 -0.0807225873710849
0.420000000000001 -0.0231935499568311
0.430000000000001 -0.0311749291095296
0.440000000000001 0.20910528620114
0.450000000000001 0.00203957559874608
0.460000000000001 0.174921383921978
0.470000000000001 -0.140164914118026
0.480000000000001 -0.158664839277324
0.490000000000001 0.0138636536009116
0.500000000000001 -0.0182443750306512
0.510000000000001 0.0178220326380944
0.520000000000001 0.0875457991847118
0.530000000000001 0.0310939457470898
0.540000000000001 -0.112621480479953
0.550000000000001 0.140802825941472
0.560000000000001 0.0574768751640088
0.570000000000001 -0.0184533318515403
0.580000000000001 0.170380830878023
0.590000000000001 0.20974323044815
0.600000000000001 -0.0026620181759961
0.610000000000001 -0.00147267674675077
0.620000000000001 0.147557491167601
0.630000000000001 0.0466039736500003
0.640000000000001 0.127359959010524
0.650000000000001 0.0995526775561583
0.660000000000001 0.0438572944834162
0.670000000000001 0.239692422818539
0.680000000000001 0.099591005710249
0.690000000000002 0.125780285006779
0.700000000000002 0.0819975001383833
0.710000000000002 0.0259934424816959
0.720000000000002 0.104735780000978
0.730000000000002 0.00736262113736728
0.740000000000002 0.0744781853426778
0.750000000000002 0.0620308462115411
0.760000000000002 -0.112091988174113
0.770000000000002 0.196772404299555
0.780000000000002 -0.130935162622139
0.790000000000002 0.129754125351619
0.800000000000002 0.17253779736417
0.810000000000002 0.112924722423326
0.820000000000002 -0.044855875326029
0.830000000000002 0.0557133971741998
0.840000000000002 0.0325991350917565
0.850000000000002 0.0634353607634586
0.860000000000002 0.196940822613358
0.870000000000002 0.0575951363704857
0.880000000000002 0.0274030405849824
0.890000000000002 0.209747405320261
0.900000000000002 0.186051050441003
0.910000000000002 0.112701606977366
0.920000000000002 0.162268267962131
0.930000000000002 -0.0256528737659208
0.940000000000002 0.132317888360843
0.950000000000002 0.182723402500852
0.960000000000002 0.11645790029742
0.970000000000002 -0.00741221661226754
0.980000000000002 0.040367154223034
0.990000000000002 0.0954072105298793
};
\addplot [blue, dashed]
table {%
-1 -0.0363700107564485
-0.99 -0.205184659615599
-0.98 -0.140062929868547
-0.97 -0.084000941959182
-0.96 -0.0413077169006438
-0.95 -0.119095819273795
-0.94 -0.200415659518542
-0.93 0.0721264271002921
-0.92 -0.0903582006203962
-0.91 0.00426741013330828
-0.9 -0.060362579908096
-0.89 -0.0541323393568631
-0.88 0.0836050273390418
-0.87 -0.0969209761660659
-0.86 0.0555012642418383
-0.85 -0.128387495747374
-0.84 -0.168063671874329
-0.83 -0.00293378853487229
-0.82 -0.0293479725750853
-0.81 0.0580347761480099
-0.8 -0.289009952722148
-0.79 -0.189729202979657
-0.78 -0.137095030538747
-0.77 0.173297684329953
-0.76 -0.12824288671402
-0.75 -0.0534110982749012
-0.74 -0.0741990187220209
-0.73 -0.117654609271403
-0.72 -0.100467788879407
-0.71 -0.286259812316557
-0.7 -0.0659179897312941
-0.69 0.0373802764593409
-0.68 -0.004434916615351
-0.67 -0.194162368902211
-0.66 0.0609436221837242
-0.65 -0.019298469265813
-0.64 -0.121450416713743
-0.63 -0.0828535703107176
-0.62 -0.037553390119876
-0.61 -0.119117587239263
-0.6 -0.0516308143476113
-0.59 -0.170469269525674
-0.58 0.029740017905424
-0.57 -0.0212013156683388
-0.56 -0.14247544912611
-0.55 0.0950918455406366
-0.54 -0.072386115095321
-0.53 -0.162317221845435
-0.52 0.169958891091993
-0.51 -0.232521709119096
-0.5 -0.193207175540536
-0.49 -0.154901280296901
-0.48 -0.100778111708651
-0.47 -0.241625048070071
-0.46 -0.0379378783271254
-0.45 -0.258313422470315
-0.44 0.10201371025102
-0.429999999999999 -0.0196674059563879
-0.419999999999999 -0.122895200724582
-0.409999999999999 -0.197048721733881
-0.399999999999999 0.0116439840344012
-0.389999999999999 0.0208166999502061
-0.379999999999999 0.155121992991402
-0.369999999999999 0.177062551608289
-0.359999999999999 -0.175334942982778
-0.349999999999999 -0.0545762142649729
-0.339999999999999 -0.081952126385806
-0.329999999999999 0.136985119596565
-0.319999999999999 -0.0822094518411788
-0.309999999999999 0.160055014526674
-0.299999999999999 0.192932219791864
-0.289999999999999 -0.116115410778294
-0.279999999999999 0.0590808786744879
-0.269999999999999 0.0376901490681573
-0.259999999999999 -0.0590991807206113
-0.249999999999999 -0.00202527702982284
-0.239999999999999 -0.0259479465557007
-0.229999999999999 0.295742583415837
-0.219999999999999 -0.171673891727002
-0.209999999999999 -0.0100765579338373
-0.199999999999999 -0.0259327205225345
-0.189999999999999 0.0682354244339063
-0.179999999999999 -0.0470827777244121
-0.169999999999999 0.03772281561873
-0.159999999999999 0.113372253117875
-0.149999999999999 -0.131185564721818
-0.139999999999999 0.0779857699340737
-0.129999999999999 0.137568873839156
-0.119999999999999 0.038438484413502
-0.109999999999999 0.119442452250736
-0.0999999999999992 -0.100307869279693
-0.0899999999999992 -0.112795426962862
-0.0799999999999992 -0.0169241912043181
-0.0699999999999992 0.0335949376994426
-0.0599999999999992 -0.0505429808213013
-0.0499999999999992 -0.0825021879775805
-0.0399999999999991 -0.0364537165479471
-0.0299999999999991 -0.032655955425365
-0.0199999999999991 0.0274123306109481
-0.00999999999999912 0.0836685198987926
8.88178419700125e-16 0.0738594698676307
0.0100000000000009 0.00879461852650064
0.0200000000000009 -0.0966832576291125
0.0300000000000009 0.0836339591819633
0.0400000000000009 0.0366319592233039
0.0500000000000009 -0.0944882340997391
0.0600000000000009 -0.026088907963368
0.070000000000001 0.132683309653363
0.080000000000001 0.0430084533629886
0.090000000000001 -0.0986689232345232
0.100000000000001 -0.0683827949656191
0.110000000000001 0.114430243013592
0.120000000000001 0.142129858746989
0.130000000000001 -0.0887476248575927
0.140000000000001 0.0545484799278364
0.150000000000001 -0.0639029565745491
0.160000000000001 -0.114206157678983
0.170000000000001 -0.0472571369728868
0.180000000000001 -0.111390596129425
0.190000000000001 -0.0475766336463341
0.200000000000001 0.0793370800697815
0.210000000000001 0.0331420685995239
0.220000000000001 0.0915338678428652
0.230000000000001 -0.10410093103614
0.240000000000001 0.0510557239347024
0.250000000000001 0.103320878008582
0.260000000000001 -0.0538785802677993
0.270000000000001 0.103831534482689
0.280000000000001 -0.00687655855504615
0.290000000000001 0.11114544207147
0.300000000000001 0.0470477383012613
0.310000000000001 0.0246103328432928
0.320000000000001 -0.0937174948620337
0.330000000000001 0.0864559369494426
0.340000000000001 -0.108440880008466
0.350000000000001 -0.0367314618866142
0.360000000000001 0.182729318350422
0.370000000000001 0.12742290046353
0.380000000000001 -0.0533643500156557
0.390000000000001 0.0561704935438583
0.400000000000001 0.113402997957119
0.410000000000001 0.139416907154302
0.420000000000001 0.0694133766377314
0.430000000000001 0.287873941066257
0.440000000000001 0.207856515538682
0.450000000000001 0.101207073538004
0.460000000000001 0.138245918037354
0.470000000000001 -0.0440748526168164
0.480000000000001 -0.0171674078763253
0.490000000000001 0.163598453463203
0.500000000000001 0.229778649700087
0.510000000000001 -0.0142834578849237
0.520000000000001 0.0670729238512905
0.530000000000001 0.144596881904805
0.540000000000001 0.0933091167053314
0.550000000000001 -0.0360385190487521
0.560000000000001 -0.170899581543975
0.570000000000001 -0.0143102706238373
0.580000000000001 0.12168506346225
0.590000000000001 -0.15045097495503
0.600000000000001 -0.0741159048750378
0.610000000000001 0.0941189295441718
0.620000000000001 0.0481519556329629
0.630000000000001 0.211762052344818
0.640000000000001 -0.039712696327731
0.650000000000001 0.195232552133148
0.660000000000001 0.162785055232145
0.670000000000001 0.1598203812581
0.680000000000001 0.028698672910308
0.690000000000002 0.141560872894551
0.700000000000002 0.105272607520528
0.710000000000002 0.0324640980116381
0.720000000000002 -0.0156003671129061
0.730000000000002 0.0862046623143601
0.740000000000002 -0.0179410982531398
0.750000000000002 0.158178433410938
0.760000000000002 0.231051862917081
0.770000000000002 0.132842983207274
0.780000000000002 0.171133546890049
0.790000000000002 0.285855247050569
0.800000000000002 0.197857891020822
0.810000000000002 -0.0298819708837953
0.820000000000002 0.0140111604181952
0.830000000000002 0.0415931197736705
0.840000000000002 0.175756424203939
0.850000000000002 0.142388760558612
0.860000000000002 -0.0134008451574316
0.870000000000002 0.107847045090093
0.880000000000002 0.155267006657458
0.890000000000002 0.104427107377327
0.900000000000002 0.000292564620871752
0.910000000000002 0.156618390069011
0.920000000000002 0.0760606636054064
0.930000000000002 -0.0458814363978817
0.940000000000002 0.103003961436041
0.950000000000002 0.0308496557270257
0.960000000000002 0.16880793288307
0.970000000000002 0.26175665622438
0.980000000000002 0.0636993598785983
0.990000000000002 0.0429339428060448
};
\addplot [blue, dashed]
table {%
-1 -0.011241177771931
-0.99 0.00598579123735732
-0.98 -0.258876654929348
-0.97 -0.0277078840875698
-0.96 -0.166564099080406
-0.95 -0.0790158414316234
-0.94 0.0222063683136044
-0.93 -0.109632314854603
-0.92 -0.224838607446851
-0.91 -0.161535414100473
-0.9 -0.179439796816357
-0.89 -0.0528938374734844
-0.88 -0.16958662741098
-0.87 -0.0514150057820273
-0.86 -0.199757820172204
-0.85 -0.221055361844075
-0.84 0.0735339212833268
-0.83 -0.110438407296428
-0.82 -0.028638014546155
-0.81 -0.142958419279231
-0.8 -0.176923347421353
-0.79 -0.233436644828897
-0.78 0.0215345865609088
-0.77 -0.00806750944125305
-0.76 0.154632323265375
-0.75 -0.0410683573959178
-0.74 -0.323186952348815
-0.73 -0.237110556168531
-0.72 -0.128793464193829
-0.71 -0.0842034446782172
-0.7 -0.0201301203717397
-0.69 -0.13874039115162
-0.68 -0.0577068049068547
-0.67 -0.0856816521171821
-0.66 -0.0477946586655714
-0.65 -0.0486291880333506
-0.64 0.0771555235653899
-0.63 0.0175951215544575
-0.62 -0.11664095735485
-0.61 -0.0633796282428273
-0.6 0.119967749710417
-0.59 -0.0886519975058726
-0.58 0.0756392206132433
-0.57 -0.117786321028472
-0.56 0.0144059045590817
-0.55 -0.127089364661984
-0.54 0.00673209240566529
-0.53 -0.0453624826055754
-0.52 -0.0366639953974564
-0.51 0.0666694516347963
-0.5 -0.244893428063609
-0.49 0.0249765913043014
-0.48 -0.0114937186981104
-0.47 -0.00445974181657265
-0.46 -0.113807778291875
-0.45 -0.0679297207274852
-0.44 -0.027305312534258
-0.429999999999999 0.05038614148896
-0.419999999999999 -0.113957159561283
-0.409999999999999 -0.00712061904151261
-0.399999999999999 -0.0932973703898131
-0.389999999999999 -0.192729642886369
-0.379999999999999 -0.0170772079121599
-0.369999999999999 0.0276534598095219
-0.359999999999999 0.0133683998403425
-0.349999999999999 -0.172856299887844
-0.339999999999999 -0.139155245805473
-0.329999999999999 -0.0197169712128218
-0.319999999999999 0.117040836109482
-0.309999999999999 -0.0534446031684648
-0.299999999999999 0.0544222818783966
-0.289999999999999 -0.0309600741841207
-0.279999999999999 -0.122114041003805
-0.269999999999999 -0.122255601636593
-0.259999999999999 -0.0462542705492096
-0.249999999999999 0.0274849772301238
-0.239999999999999 -0.0123750061087351
-0.229999999999999 -0.135876114468085
-0.219999999999999 -0.166277064788843
-0.209999999999999 0.0399529598490248
-0.199999999999999 -0.054256652545034
-0.189999999999999 -0.0435910920713415
-0.179999999999999 -0.0156838002526359
-0.169999999999999 0.0436802208008653
-0.159999999999999 -0.0843778776322773
-0.149999999999999 -0.15650854181052
-0.139999999999999 -0.200933334672166
-0.129999999999999 -0.114334751349202
-0.119999999999999 -0.111119238641109
-0.109999999999999 -0.023208983982831
-0.0999999999999992 0.0630915828497583
-0.0899999999999992 0.0798035892227532
-0.0799999999999992 -0.105906794958416
-0.0699999999999992 -0.016944756087461
-0.0599999999999992 0.0351208228300379
-0.0499999999999992 0.0367123627495727
-0.0399999999999991 0.0647374009440624
-0.0299999999999991 0.0683438779600084
-0.0199999999999991 -0.104125501134542
-0.00999999999999912 0.0436637104085615
8.88178419700125e-16 0.136243878583853
0.0100000000000009 0.0494537314504301
0.0200000000000009 -0.0077166690309432
0.0300000000000009 0.057041375905744
0.0400000000000009 -0.0445627751340446
0.0500000000000009 -0.0635863510874625
0.0600000000000009 0.0494333869238823
0.070000000000001 -0.0800525699178367
0.080000000000001 0.17487685584994
0.090000000000001 0.213958106427352
0.100000000000001 -0.0629057235831379
0.110000000000001 0.030363737447076
0.120000000000001 -0.0262467022953556
0.130000000000001 0.0875019622094136
0.140000000000001 -0.00942842416313286
0.150000000000001 0.115629041136646
0.160000000000001 -0.0535192460849669
0.170000000000001 0.0586418673546983
0.180000000000001 -0.0868918286037782
0.190000000000001 0.050799465771607
0.200000000000001 0.100987625306305
0.210000000000001 -0.0551364155681027
0.220000000000001 0.0785350695773207
0.230000000000001 0.0172138740730912
0.240000000000001 0.353495196106111
0.250000000000001 -0.0316245893194798
0.260000000000001 -0.0349252639005187
0.270000000000001 0.0742969839350058
0.280000000000001 0.136301528451993
0.290000000000001 0.0871896903783521
0.300000000000001 0.257209994730086
0.310000000000001 -0.118707917790742
0.320000000000001 -0.121394896868323
0.330000000000001 -0.0263902791493756
0.340000000000001 -0.132175349697493
0.350000000000001 0.0902584009289885
0.360000000000001 -0.0463961854338535
0.370000000000001 -0.0429987258811139
0.380000000000001 0.00971242521419039
0.390000000000001 -0.0861475156212604
0.400000000000001 0.0195546513776342
0.410000000000001 -0.0783927185990415
0.420000000000001 -0.0273452902309468
0.430000000000001 0.152572489296782
0.440000000000001 0.0833001308235505
0.450000000000001 0.178904443638686
0.460000000000001 -0.0195456701196187
0.470000000000001 -0.0468675183222126
0.480000000000001 0.128230849644051
0.490000000000001 0.191348477727195
0.500000000000001 0.142092431578761
0.510000000000001 0.252748968935161
0.520000000000001 0.0909482694438127
0.530000000000001 0.223918986002735
0.540000000000001 0.00489648471839758
0.550000000000001 0.157805792239493
0.560000000000001 0.255800358903569
0.570000000000001 -0.172335230861368
0.580000000000001 -0.105645064884855
0.590000000000001 0.0155948629778049
0.600000000000001 0.209740489471466
0.610000000000001 0.0910961247014132
0.620000000000001 0.00706341775652076
0.630000000000001 -0.135846688078533
0.640000000000001 0.0863266851112628
0.650000000000001 0.183339730767442
0.660000000000001 0.00382515822315058
0.670000000000001 0.212122587192913
0.680000000000001 0.0774186548538041
0.690000000000002 0.392598156631625
0.700000000000002 0.171477039473873
0.710000000000002 0.281788896188642
0.720000000000002 -0.0659733067555433
0.730000000000002 0.116145337913527
0.740000000000002 0.0446999866479039
0.750000000000002 0.187375428091924
0.760000000000002 0.139700469040533
0.770000000000002 0.0169905287849083
0.780000000000002 -0.0568671544483078
0.790000000000002 0.0340294124377157
0.800000000000002 -0.0352844896335111
0.810000000000002 0.134503912548244
0.820000000000002 0.163999180833777
0.830000000000002 -0.015985800751226
0.840000000000002 0.213917852199465
0.850000000000002 0.0438136399300391
0.860000000000002 0.0593987824139308
0.870000000000002 0.136989533349466
0.880000000000002 0.167982764073326
0.890000000000002 0.0305631488382855
0.900000000000002 -0.0995424660772363
0.910000000000002 0.155478650191456
0.920000000000002 0.0941254784667709
0.930000000000002 0.153001663473855
0.940000000000002 0.138937730952321
0.950000000000002 0.128391873740472
0.960000000000002 0.192733109361239
0.970000000000002 -0.0168249696018968
0.980000000000002 0.0956399023601168
0.990000000000002 -0.0550333957429259
};
\end{axis}

\end{tikzpicture}

%% file: completely_flat_mytikz.tex
\begin{tikzpicture}

\begin{axis}[
xmin=-1, xmax=1,
ymin=-1.2, ymax=1.2,
axis on top,
width=\figurewidth,
height=\figureheight
]
\addplot [thick, blue]
table {%
-1 -1
-0.99 -1
-0.98 -1
-0.97 -1
-0.96 -1
-0.95 -1
-0.94 -1
-0.93 -1
-0.92 -1
-0.91 -1
-0.9 -1
-0.89 -1
-0.88 -1
-0.87 -1
-0.86 -1
-0.85 -1
-0.84 -1
-0.83 -1
-0.82 -1
-0.81 -1
-0.8 -1
-0.79 -1
-0.78 -1
-0.77 -1
-0.76 -1
-0.75 -1
-0.74 -1
-0.73 -1
-0.72 -1
-0.71 -1
-0.7 -1
-0.69 -1
-0.68 -1
-0.67 -1
-0.66 -1
-0.65 -1
-0.64 -1
-0.63 -1
-0.62 -1
-0.61 -1
-0.6 -1
-0.59 -1
-0.58 -1
-0.57 -1
-0.56 -1
-0.55 -1
-0.54 -1
-0.53 -1
-0.52 -1
-0.51 -1
-0.5 -1
-0.49 -1
-0.48 -1
-0.47 -1
-0.46 -1
-0.45 -1
-0.44 -1
-0.429999999999999 -1
-0.419999999999999 -1
-0.409999999999999 -1
-0.399999999999999 -1
-0.389999999999999 -1
-0.379999999999999 -1
-0.369999999999999 -1
-0.359999999999999 -1
-0.349999999999999 -1
-0.339999999999999 -1
-0.329999999999999 -1
-0.319999999999999 -1
-0.309999999999999 -1
-0.299999999999999 -1
-0.289999999999999 -1
-0.279999999999999 -1
-0.269999999999999 -1
-0.259999999999999 -1
-0.249999999999999 -1
-0.239999999999999 -1
-0.229999999999999 -1
-0.219999999999999 -1
-0.209999999999999 -1
-0.199999999999999 -1
-0.189999999999999 -1
-0.179999999999999 -1
-0.169999999999999 -1
-0.159999999999999 -1
-0.149999999999999 -1
-0.139999999999999 -1
-0.129999999999999 -1
-0.119999999999999 -1
-0.109999999999999 -1
-0.0999999999999992 -1
-0.0899999999999992 -1
-0.0799999999999992 -1
-0.0699999999999992 -1
-0.0599999999999992 -1
-0.0499999999999992 -1
-0.0399999999999991 -1
-0.0299999999999991 -1
-0.0199999999999991 -1
-0.00999999999999912 -1
8.88178419700125e-16 1
0.0100000000000009 1
0.0200000000000009 1
0.0300000000000009 1
0.0400000000000009 1
0.0500000000000009 1
0.0600000000000009 1
0.070000000000001 1
0.080000000000001 1
0.090000000000001 1
0.100000000000001 1
0.110000000000001 1
0.120000000000001 1
0.130000000000001 1
0.140000000000001 1
0.150000000000001 1
0.160000000000001 1
0.170000000000001 1
0.180000000000001 1
0.190000000000001 1
0.200000000000001 1
0.210000000000001 1
0.220000000000001 1
0.230000000000001 1
0.240000000000001 1
0.250000000000001 1
0.260000000000001 1
0.270000000000001 1
0.280000000000001 1
0.290000000000001 1
0.300000000000001 1
0.310000000000001 1
0.320000000000001 1
0.330000000000001 1
0.340000000000001 1
0.350000000000001 1
0.360000000000001 1
0.370000000000001 1
0.380000000000001 1
0.390000000000001 1
0.400000000000001 1
0.410000000000001 1
0.420000000000001 1
0.430000000000001 1
0.440000000000001 1
0.450000000000001 1
0.460000000000001 1
0.470000000000001 1
0.480000000000001 1
0.490000000000001 1
0.500000000000001 1
0.510000000000001 1
0.520000000000001 1
0.530000000000001 1
0.540000000000001 1
0.550000000000001 1
0.560000000000001 1
0.570000000000001 1
0.580000000000001 1
0.590000000000001 1
0.600000000000001 1
0.610000000000001 1
0.620000000000001 1
0.630000000000001 1
0.640000000000001 1
0.650000000000001 1
0.660000000000001 1
0.670000000000001 1
0.680000000000001 1
0.690000000000002 1
0.700000000000002 1
0.710000000000002 1
0.720000000000002 1
0.730000000000002 1
0.740000000000002 1
0.750000000000002 1
0.760000000000002 1
0.770000000000002 1
0.780000000000002 1
0.790000000000002 1
0.800000000000002 1
0.810000000000002 1
0.820000000000002 1
0.830000000000002 1
0.840000000000002 1
0.850000000000002 1
0.860000000000002 1
0.870000000000002 1
0.880000000000002 1
0.890000000000002 1
0.900000000000002 1
0.910000000000002 1
0.920000000000002 1
0.930000000000002 1
0.940000000000002 1
0.950000000000002 1
0.960000000000002 1
0.970000000000002 1
0.980000000000002 1
0.990000000000002 1
};
\end{axis}

\end{tikzpicture}

%% file: Decompositionvs.EndtoEnd_mytikz10.tex
\begin{tikzpicture}

\begin{axis}[
title={$k=10$},
xmin=0, xmax=2500,
ymin=-1.0, ymax=0.1,
ytick={-1,0}, xtick={0,2500}, 
axis on top,
width=\figurewidth,
height=\figureheight,
]
\addplot [blue, very thick]
table {%
0 -0.106761619448662
100 -0.913434445858002
200 -0.984988033771515
300 -0.99321460723877
400 -0.99615877866745
500 -0.997779488563538
600 -0.998586416244507
700 -0.999050498008728
800 -0.999324083328247
900 -0.999517917633057
1000 -0.999645411968231
1100 -0.999722182750702
1200 -0.999783158302307
1300 -0.999814093112946
1400 -0.99985134601593
1500 -0.999870479106903
1600 -0.999886393547058
1700 -0.999897241592407
1800 -0.999906361103058
1900 -0.999918937683105
2000 -0.999925136566162
2100 -0.999928593635559
2200 -0.99993485212326
2300 -0.999938607215881
2400 -0.999943733215332
2500 -0.999947011470795
};
\addplot [red, very thick]
table {%
0 -0.104796208441257
100 -0.645514309406281
200 -0.959836661815643
300 -0.984630346298218
400 -0.992495894432068
500 -0.995682597160339
600 -0.997475206851959
700 -0.998362302780151
800 -0.99895453453064
900 -0.999278247356415
1000 -0.999502778053284
1100 -0.999625980854034
1200 -0.999720096588135
1300 -0.999781966209412
1400 -0.999820113182068
1500 -0.999846458435059
1600 -0.999868214130402
1700 -0.999885678291321
1800 -0.999895036220551
1900 -0.99990701675415
2000 -0.999914526939392
2100 -0.99992311000824
2200 -0.999928653240204
2300 -0.999933481216431
2400 -0.999939560890198
2500 -0.999943017959595
};
\end{axis}

\end{tikzpicture}

%% file: Decompositionvs.EndtoEnd_mytikz1000.tex
\begin{tikzpicture}

\begin{axis}[
title={$k=1000$},
xmin=0, xmax=2500,
ymin=-1, ymax=0.1,
ytick={-1,0}, xtick={0,2500}, 
axis on top,
width=\figurewidth,
height=\figureheight,
]
\addplot [blue, very thick]
table {%
0 -0.00201964494772255
100 -0.0125462003052235
200 -0.123941607773304
300 -0.184201344847679
400 -0.2118029743433
500 -0.239620238542557
600 -0.324453383684158
700 -0.328857421875
800 -0.402112722396851
900 -0.473526060581207
1000 -0.490992069244385
1100 -0.56432569026947
1200 -0.589135944843292
1300 -0.632184982299805
1400 -0.712549567222595
1500 -0.729654550552368
1600 -0.824455797672272
1700 -0.834363877773285
1800 -0.902213156223297
1900 -0.932877123355865
2000 -0.930855989456177
2100 -0.991925895214081
2200 -0.988624334335327
2300 -0.996478319168091
2400 -0.999313950538635
2500 -0.999529898166656
};
\addplot [red, very thick]
table {%
0 0.000359951023710892
100 -0.00099258404225111
200 -0.0706849843263626
300 -0.149435818195343
400 -0.157273516058922
500 -0.248439028859138
600 -0.246359378099442
700 -0.327818423509598
800 -0.325675785541534
900 -0.314688980579376
1000 -0.303794682025909
1100 -0.321512490510941
1200 -0.282225281000137
1300 -0.275350213050842
1400 -0.32072988152504
1500 -0.343700438737869
1600 -0.319009959697723
1700 -0.317295461893082
1800 -0.378635078668594
1900 -0.290154606103897
2000 -0.268019884824753
2100 -0.311224222183228
2200 -0.307428419589996
2300 -0.211532756686211
2400 -0.298941105604172
2500 -0.285933017730713
};
\end{axis}

\end{tikzpicture}

%% file: Decompositionvs.EndtoEnd_mytikz2000.tex
\begin{tikzpicture}

\begin{axis}[
title={$k=2000$},
xmin=0, xmax=2500,
ymin=-1, ymax=0.1,
ytick={-1,0}, 
xtick={0,2500}, 
axis on top,
width=\figurewidth,
height=\figureheight,
]
\addplot [blue, very thick]
table {%
0 0.000994539819657803
100 -0.000947590102441609
200 -0.0408724844455719
300 -0.108514100313187
400 -0.174034208059311
500 -0.217071697115898
600 -0.235920995473862
700 -0.319059312343597
800 -0.342999190092087
900 -0.386191338300705
1000 -0.422139108181
1100 -0.484789431095123
1200 -0.583254158496857
1300 -0.621284902095795
1400 -0.675511002540588
1500 -0.743055284023285
1600 -0.791310489177704
1700 -0.826255083084106
1800 -0.881034851074219
1900 -0.91385680437088
2000 -0.926129996776581
2100 -0.972565472126007
2200 -0.991702258586884
2300 -0.990334987640381
2400 -0.990772604942322
2500 -0.999374568462372
};
\addplot [red, very thick]
table {%
0 -0.000945117208175361
100 0.000775042979512364
200 -0.0162352826446295
300 -0.0169071611016989
400 -0.0876036807894707
500 -0.111845090985298
600 -0.0741686895489693
700 -0.172008991241455
800 -0.155471637845039
900 -0.144302293658257
1000 -0.137203395366669
1100 -0.111406967043877
1200 -0.183514103293419
1300 -0.14453536272049
1400 -0.173164963722229
1500 -0.148761346936226
1600 -0.201202467083931
1700 -0.202359065413475
1800 -0.220212265849113
1900 -0.237122744321823
2000 -0.153671786189079
2100 -0.148650795221329
2200 -0.174396455287933
2300 -0.14641173183918
2400 -0.154208093881607
2500 -0.235572457313538
};
\end{axis}

\end{tikzpicture}